\documentclass[11pt,a4paper]{article}

\usepackage[T1]{fontenc}

\usepackage{amsmath, amssymb, amsthm, mathrsfs}
\usepackage[mathcal]{euscript}
\usepackage{bm} 

\usepackage{graphicx}
\usepackage{caption, subcaption}
\graphicspath{{Figures/}}

\usepackage{tikz}
\usetikzlibrary{shapes, arrows.meta, positioning, calc}
\tikzstyle{terminator} = [rectangle, draw, rounded corners, minimum height=2em, text centered]
\tikzstyle{process}    = [rectangle, draw, minimum height=2em, text centered]
\tikzstyle{decision}   = [diamond, draw, minimum height=2em, text centered]
\tikzstyle{data}       = [trapezium, draw, trapezium left angle=60, trapezium right angle=120, minimum height=2em, text centered]
\tikzstyle{connector}  = [draw, -latex]

\usepackage{algorithm}
\usepackage{algpseudocode}
\makeatletter
\newcommand{\RemoveAlgoNumber}{\renewcommand{\fnum@algorithm}{\AlCapSty{\AlCapFnt\algorithmcfname}}}
\makeatother

\usepackage[round]{natbib}
\usepackage{xr} 

\usepackage{enumitem}
\setlist[itemize]{nosep, leftmargin=*, labelsep=0.5em}

\usepackage{booktabs, multicol, multirow, rotating, float}

\usepackage{setspace}

\usepackage{xcolor}

\newtheorem{remark}{Remark}[section]

\newtheorem{proposition}{Proposition}[section]
\newtheorem{theorem}{Theorem}[section]

\newtheorem{corollary}{Corollary}[section]
\newtheorem{property}{Property}[section]

\def\be{\begin{equation}}
\def\ee{\end{equation}}
\def\bea{\begin{eqnarray}}
\def\eea{\end{eqnarray}}

\title{
Solving the Constrained Random Disambiguation Path Problem\\ via Lagrangian Relaxation and Graph Reduction
\vspace{-0,25cm}}
\author{
Li Zhou$^{*}$ ~\&~ Elvan Ceyhan\thanks{Department of Mathematics and Statistics, Auburn University, Auburn, AL 36849, USA \\
\texttt{ emails: lzz0062@auburn.edu (corresponding author), ceyhan@auburn.edu}}
}
\date{ }

\begin{document}
\color{black}
\maketitle
\vspace{-1.25cm} 

\tableofcontents
	

\begin{abstract}
\noindent
We study a resource-constrained variant of the Random Disambiguation Path (RDP) problem, 
a generalization of the Stochastic Obstacle Scene (SOS) problem, 
in which a navigating agent must reach a target in a spatial environment populated with uncertain obstacles. 
Each ambiguous obstacle may be disambiguated at a (possibly heterogeneous) heterogeneous resource cost, 
subject to a global disambiguation budget. 
We formulate this constrained planning problem as a Weight-Constrained Shortest Path Problem (WCSPP) 
with risk-adjusted edge costs that incorporate probabilistic blockage and traversal penalties.
To solve it, we propose a novel algorithmic framework—COLOGR—combining Lagrangian relaxation 
with a two-phase vertex elimination (TPVE) procedure. 
The method prunes infeasible and suboptimal paths while provably preserving the optimal solution, 
and leverages dual bounds to guide efficient search. 
We establish correctness, feasibility guarantees, and surrogate optimality under mild assumptions. 
Our analysis also demonstrates that COLOGR frequently achieves zero duality gap 
and offers improved computational complexity over prior constrained path-planning methods.
Extensive simulation experiments validate the algorithm’s robustness across 
varying obstacle densities, sensor accuracies, and risk models, 
consistently outperforming greedy baselines and approaching offline-optimal benchmarks.
The proposed framework is broadly applicable to stochastic network design, 
mobility planning, and constrained decision-making under uncertainty.
\end{abstract}


\section{Introduction}
\label{sec:intro}

Graph-based navigation in stochastic and partially observable environments is a fundamental problem 
in autonomous systems, with applications ranging from mobile robotics and autonomous driving to 
coordinated multi-agent systems and GNSS-denied navigation. 
These environments are often characterized by incomplete, noisy, or dynamically evolving information, 
necessitating robust path planning under uncertainty.  
In response, a diverse body of research has emerged, including approaches based on greedy heuristics 
\citep{aksakalli2011, aksakalliari2014}, reinforcement learning in partially observable settings 
\citep{zweig2020graphRL}, and probabilistic planning via factor graphs for GNSS/INS fusion  
\citep{xin2023factor}, among others.

Several classical formulations have laid the foundation for modern stochastic path planning.  
The \emph{Stochastic Obstacle Scene (SOS)} problem \citep{papadimitriou1991shortest} models 
environments populated with probabilistic obstacles
whose true blockage is unknown prior to traversal.  
The \emph{Random Disambiguation Path (RDP)} problem \citep{fishkind2005} extends this setting by 
introducing a disambiguation mechanism: upon encountering a potentially obstructing region (e.g., an obstacle), 
a navigating agent (NAVA) may resolve its status—blocking or non-blocking—at the cost of a 
disambiguation action.  
The goal is to minimize expected traversal cost, balancing movement and information 
acquisition.

In the RDP framework, the environment consists of disk-shaped obstacles, each either truly blocking or not.  
Their status is inferred probabilistically from a noisy sensor model.  
The NAVA may choose to disambiguate an obstacle upon contact, incurring a nonzero cost.  
Existing policies typically assign approximate expected costs to edges and use shortest path 
algorithms such as Dijkstra’s algorithm \citep{dijkstra1959} to compute traversal plans.  
Examples include Reset Disambiguation (RD) \citep{aksakalli2011}, Distance-to-Termination (DT), and 
other penalty-based policies \citep{aksakalliari2014, alkaya2021heuristics}, which differ in how 
they account for obstacle interactions.  
These approaches, while effective in certain settings, do not explicitly model realistic limitations 
on disambiguation resources.

In practice, such actions often incur heterogeneous costs and are subject to cumulative 
resource constraints.  
For instance, in automated logistics, a fleet vehicle must avoid traffic blockages under limited fuel 
or energy reserves; in disaster response, a medical team may need to verify road access to clinics 
with limited communication capabilities; and in minefield navigation, an autonomous agent may be 
constrained by a finite supply of detection scans.  
In such cases, methods that restrict only the \emph{number} of disambiguations 
\citep{fishkind2007, alkaya2015metaheuristic, alkaya2021ctpn} fall short, as they fail to account 
for continuous resource consumption and variable disambiguation costs.

To address these shortcomings, we propose a resource-aware extension of the RDP framework.  
We formulate the \emph{RDP with Constrained Disambiguation} (RCDP) problem as a \emph{Weight 
Constrained Shortest Path Problem} (WCSPP), where the objective is to minimize traversal cost 
subject to a global constraint on disambiguation expenditure.  
This formulation captures both uncertainty in obstacle fields and cumulative limitations 
on information acquisition resources.

We develop a novel RCDP traversal policy that leverages probabilistic sensor information through 
risk-sensitive cost approximations, enabling more accurate estimation under disambiguation.  
Our method embeds the constrained optimization problem into a Lagrangian relaxation framework, 
augmented with multiple graph reduction steps that eliminate infeasible paths, reduce problem 
dimensionality, and ensure optimality by closing the duality gap.  
Unlike previous methods that use static or greedy cost assignment, our algorithm dynamically adapts 
to resource restrictions and maintains feasibility with respect to the disambiguation budget.

In addition to addressing core challenges in stochastic path planning, our framework connects 
naturally to related work in network interdiction \citep{israeli2002shortest, smith2020survey}, 
particularly in settings where uncertainty and partial observability influence adversarial 
interactions \citep{azizi2024shortest, sadeghi2024modified}.  
However, the RCDP problem focuses on the navigator's decision-making process rather than the 
interdictor’s, making it a distinct yet complementary contribution to the broader literature on 
constrained decision-making in uncertain networks.

The remainder of this paper is structured as follows.  
Section~\ref{sec:RDP-problem} formally defines the RCDP problem, introducing modeling assumptions, cost structures, and risk measures.  
Section~\ref{sec:COLOGR} presents the COLOGR optimization algorithm, which combines Lagrangian relaxation with a two-phase graph reduction scheme.  
Section~\ref{sec:RCDP-policy} details the RCDP traversal policy, including its implementation and theoretical guarantees.  
Section~\ref{sec:MC-simulations} reports empirical results from extensive Monte Carlo simulations under varying obstacle and sensor conditions.  
Section~\ref{sec:disc-conc} concludes with a summary and directions for future research.
Appendix Section contains algorithmic pseudocode and proofs for the main theoretical results,
and also provides additional technical content omitted for brevity, 
including secondary proofs, $\alpha$-sensitivity analysis, a Bayesian LU extension, and extended simulation results.

\section{Problem Formulation}
\label{sec:RDP-problem}

We introduce a resource-constrained extension of the \emph{Random Disambiguation Path} (RDP) problem, 
a probabilistic path planning framework originally proposed by \citet{fishkind2005}, which builds on 
the foundational \emph{Stochastic Obstacle Scene} (SOS) problem of \citet{papadimitriou1991shortest}.  
In this setting, a navigating agent (NAVA) must traverse from a source location $s$ to a target location 
$t$ within a bounded planar domain $\Omega \subset \mathbb{R}^2$ that contains a set of spatially 
distributed disk-shaped obstacles whose true statuses—blocking or non-blocking—are initially unknown.

Let $X = X_T \cup X_F$ denote the set of obstacle centers, partitioned into true obstacles $X_T$ that 
block traversal and false obstacles $X_F$ that permit passage.  
Each obstacle $x \in X$ is modeled as a closed disk of fixed radius $\text{radius}(x) > 0$ centered at $x$.  
The agent does not know an obstacle's status a priori 
but instead receives a probabilistic estimate from a noisy sensor, 
represented by a function $\pi: X \rightarrow [0,1]$, 
where $\pi_x := \pi(x)$ denotes the probability that obstacle $x$ is truly blocking.  
Obstacle statuses are assumed to be independent and remain static throughout the traversal.

Upon reaching the boundary of an ambiguous obstacle $x$, the NAVA may pay a disambiguation cost 
$\delta_x > 0$ to reveal its true status.  
If $x$ is non-blocking, the agent proceeds; otherwise, it must reroute.  
Each disambiguation action consumes a portion of a cumulative resource budget, 
such as time, energy, or computational capacity.  
The goal is to minimize the expected total cost, comprising both traversal and disambiguation expenditures, 
while respecting this global constraint.

To make the problem computationally tractable, we discretize $\Omega$ using an 8-adjacency integer 
lattice, transforming the continuous domain into a finite undirected graph $G = (V, E)$.  
Each vertex $v \in V$ corresponds to a lattice point with integer coordinates $(i,j)$, and each edge 
$e \in E$ connects adjacent vertices horizontally, vertically, or diagonally.  
The edge length $\ell_e$ equals the Euclidean distance between its endpoints.  
This discretization balances representational fidelity with computational efficiency.
As discussed in Remark~\ref{rem:robustness and complexity}, finer grids offer greater path redundancy 
but also increase problem size and complexity due to more frequent obstacle-edge intersections.

An edge $e \in E$ is traversable if it does not intersect any true or unresolved obstacle.  
If it intersects an ambiguous obstacle, the agent may disambiguate it at one of the edge's endpoints before traversal.  
Let $p$ denote a valid $s$–$t$ path in $G$, and let $\mathcal{C}_p$ denote the total (random) cost incurred along $p$,
which includes both traversal distance $\ell_p$ and disambiguation cost $\delta_p$:
\[
\mathcal{C}_p = \ell_p + \delta_p,
\]
where $\ell_p$ is the sum of edge lengths along $p$, and $\delta_p = \sum_{x: p \cap D(x) \neq \emptyset} \delta_x$ 
is the total cost of disambiguating all ambiguous obstacles intersected by $p$.

Let $\mathcal{P}(s,t)$ denote the set of all feasible paths from $s$ to $t$ in $G$.
The objective is to find a path $p \in \mathcal{P}(s,t)$ that minimizes the expected total cost:
\[
\min_{p \in \mathcal{P}(s,t)} \mathbf{E}[\mathcal{C}_p].
\]
This optimization accounts for both obstacle uncertainty and the cost trade-offs induced by disambiguation.

\begin{remark}
\label{rem:robustness and complexity}
The granularity of the spatial discretization strongly influences both the computational complexity 
and robustness of the path planning model.  
Specifically, a finer grid increases the number of vertices and edges, expanding the feasible 
solution space and allowing more alternative paths between $s$ and $t$.  
This enhances resilience to disruptions, as measured by the local vertex connectivity $\kappa(s,t)$—the 
minimum number of vertices whose removal disconnects $s$ from $t$ in $G$.  
A higher $\kappa(s,t)$ implies greater robustness against obstacle insertions or adversarial interference.  
However, increased connectivity also raises combinatorial complexity, 
as the number of feasible paths and potential disambiguation points grows rapidly. $\square$
\end{remark}

In the Discrete RDP (D-RDP) model, the vertex closest to the source is selected as $s$, and the closest to the target is designated $t$.  
The cost $\mathcal{C}_p$ is inherently stochastic due to random obstacle locations and uncertain statuses.  
The RDP formulation minimizes its expected value, and in the constrained variant introduced in Section~\ref{sec:RCDPasWCSPP}, 
this optimization is carried out under a global disambiguation budget, yielding a tractable yet flexible framework 
for risk-aware path planning in uncertain environments.

\subsection{Formulating RCDP as a Weight-Constrained Shortest Path Problem}
\label{sec:RCDPasWCSPP}

We now formalize the \emph{Random Constrained Disambiguation Path} (RCDP) problem as a constrained optimization framework.  
In this extension, the navigating agent (NAVA) operates under a global disambiguation budget 
$\delta_{\max}$, representing the maximum allowable cost for resolving obstacle statuses along the path.  
The objective is to find a feasible path $p$ from the source $s$ to the destination $t$ that minimizes the expected total cost, 
while ensuring the total disambiguation expenditure does not exceed $\delta_{\max}$.

This yields the following constrained optimization problem:
\begin{equation}
\label{eqn:WCSPP}
\min_{p \in \mathcal{P}(s,t)} \mathbf{E}[\mathcal{C}_p] \quad \text{subject to} \quad 
\delta_p \leq \delta_{\max},
\end{equation}
which is a specific instance of the classical \emph{Weight Constrained Shortest Path Problem} (WCSPP), 
with the “weight” corresponding to cumulative disambiguation cost.

We consider two cost structures for disambiguation.  
In the \emph{uniform model}, all obstacles incur a fixed cost $\delta_x = \delta$ for every $x \in X$.  
In contrast, the \emph{heterogeneous model} allows $\delta_x$ to vary based on contextual factors 
such as sensor reliability, terrain conditions, or distance to the goal.

\subsection{Risk Modeling and Surrogate Cost Approximation}
\label{sec:risk-modeling}

Computing the exact expected cost $\mathbf{E}[\mathcal{C}_p]$ for a path $p$ is generally intractable 
due to the exponential number of possible obstacle realizations.  
To enable efficient planning under uncertainty, we adopt a surrogate cost approximation based on 
deterministic edge-level penalties.

Specifically, we replace $\mathbf{E}[\mathcal{C}_p]$ with an additive deterministic surrogate:
\begin{equation}
\label{eqn:surrogate-cost}
\widetilde{\mathcal{C}}_p := \sum_{e \in E(p)} \left( \ell_e + r_e \right),
\end{equation}
where $\ell_e$ is the Euclidean length of edge $e$, and $r_e$ is a penalty term that accounts for the 
expected disambiguation cost and traversal risk induced by ambiguous obstacles intersecting $e$.  

This surrogate formulation supports efficient constrained optimization and naturally incorporates 
sensor-driven uncertainty into edge costs, allowing for scalable planning in environments with 
limited information and resource constraints.

\subsubsection*{Surrogate Constrained Optimization}

Combining the deterministic risk-adjusted cost formulation with the disambiguation constraint leads to 
the following approximation of the RCDP problem:
\begin{equation}
\label{eqn:WCSPP_restate}
\min_{p \in \mathcal{P}(s,t)} \widetilde{\mathcal{C}}_p \quad \text{subject to} \quad 
\delta_p := \sum_{e \in E(p)} \delta_e \leq \delta_{\max},
\end{equation}
where $\delta_e$ is the surrogate disambiguation cost assigned to edge $e$, and $\delta_p$ 
is the total accumulated cost over path $p$.  
This formulation maintains the structure of a \emph{Weight Constrained Shortest Path Problem} (WCSPP) 
and allows for specialized solution techniques.  
It strikes a balance between computational efficiency and sensitivity to uncertainty, enabling 
scalable, resource-aware navigation under stochastic conditions.

The algorithmic solution to \eqref{eqn:WCSPP_restate}, including graph pruning and 
Lagrangian dual methods, is developed in Section~\ref{sec:COLOGR}.

\subsubsection*{Edge-Level Risk Aggregation}

For each edge $e \in E$, let $X_e := \{x \in X : D(x) \cap e \neq \emptyset\}$ denote the set of 
ambiguous obstacles intersecting the edge.  
Each obstacle $x \in X_e$ contributes a risk penalty $r_x$ and a disambiguation cost $\delta_x$, 
which are symmetrically divided across intersecting edges since disambiguation may be initiated 
from either endpoint.  
Accordingly, we define the edge-level quantities as:
\[
r_e := \frac{1}{2} \sum_{x \in X_e} r_x, \quad \delta_e := \frac{1}{2} \sum_{x \in X_e} \delta_x.
\]

The value $r_x$ reflects the perceived risk of traversing near obstacle $x$ without resolving its 
status.  
These edge-level terms enter directly into both the surrogate cost objective and the disambiguation 
budget constraint.  
Depending on the sensor model and obstacle features, various forms of $r_x$ can be adopted from the 
literature to capture uncertainty and criticality.

\subsubsection*{Risk Function Models}

We now describe several obstacle-level risk functions $r_x$ used in the surrogate edge cost formulation.  
Each function maps the disambiguation cost $\delta_x$ and the sensor-assigned probability $\pi_x$ of an obstacle being true 
to a penalty that reflects the traversal risk associated with ambiguous obstacles.  
These models, adapted from prior literature, capture different attitudes toward uncertainty and spatial exposure, 
and they influence how conservative or exploratory the resulting path will be.

\paragraph{Reset Disambiguation (RD).}
The RD model \citep{aksakalli2011} defines the risk as:
\[
r_x^{\mathrm{RD}} = \frac{\delta_x}{1 - \pi_x}.
\]
This function penalizes obstacles more heavily as $\pi_x \to 1$, encouraging early disambiguation of high-risk regions.  
It represents a pessimistic cost expectation under a greedy traversal policy and provides sharp risk separation for ambiguous areas.

\paragraph{Distance-to-Termination (DT).}
The DT model \citep{aksakalliari2014} incorporates spatial proximity to the goal:
\[
r_x^{\mathrm{DT}} = \delta_x + \left( \frac{d(x, t)}{1 - \pi_x} \right)^{-\log(1 - \pi_x)},
\]
where $d(x, t)$ is the Euclidean distance between obstacle $x$ and the target $t$.  
The formulation emphasizes obstacles near the goal whose status is uncertain, increasing the risk 
penalty nonlinearly as $\pi_x$ rises and $d(x,t)$ shrinks.

\paragraph{Linear Undesirability (LU).}
The LU model \citep{fishkind2007} uses a logarithmic transformation to model increasing risk aversion:
\[
r_x^{\mathrm{LU}} = -\alpha \log(1 - \pi_x),
\]
where $\alpha > 0$ is a tunable parameter.  
Larger values of $\alpha$ induce more conservative routing.  
We consider three instantiations:
\begin{itemize}
    \item[] $r_{15}^L$: moderate aversion with $\alpha = 15$,
    \item[] $r_{30}^L$: high-risk aversion with $\alpha = 30$,
    \item[] $r_{\delta}^L$: adaptive penalty with $\alpha = \delta_x$.
\end{itemize}
This family provides smooth control over risk sensitivity and avoids unbounded penalties near $\pi_x = 1$, 
making it useful for calibrated trade-offs between caution and efficiency.

These risk models define the edge-wise surrogate cost structure that supports optimization under the disambiguation constraint.  
Their integration into the RCDP solution process is discussed in Section~\ref{sec:COLOGR}.

\subsection{Graph Initialization with Risk and Resource Embedding}
\label{sec:graph-init}

Before solving the constrained problem in Equation~\eqref{eqn:WCSPP_restate}, we preprocess the 
traversal graph $G = (V, E)$ to assign deterministic edge-level costs and disambiguation weights.  
This transformation aggregates obstacle-level information into edge-level quantities, thereby enabling 
efficient constrained shortest path computation.

For each edge $e \in E$, let $X_e := \{x \in X : D(x) \cap e \neq \emptyset\}$ be the set of ambiguous obstacles 
intersecting $e$. Using a selected risk function (e.g., RD, DT, or LU from Section~\ref{sec:RCDPasWCSPP}), 
we compute:
\[
\widetilde{\mathcal{C}}_e := \ell_e + r_e, \quad \text{where} \quad r_e = \frac{1}{2} \sum_{x \in X_e} r_x,
\]
and assign the edge-level disambiguation weight as:
\[
\delta_e := \frac{1}{2} \sum_{x \in X_e} \delta_x,
\]
with $\ell_e$ denoting the Euclidean length of edge $e$. The factor $1/2$ reflects symmetric allocation 
since disambiguation can occur from either endpoint.

The initialization process is summarized in Algorithm~\ref{graph_initialization}, which takes as input the 
graph structure, obstacle configuration, sensor probabilities $\pi_x$, disambiguation costs $\delta_x$, 
and a specified risk function. The output is an adjusted graph with edge-wise costs and weights that embed 
both geometric and probabilistic features.

\begin{algorithm}[H]
\caption{Graph Initialization (GI)}
\label{graph_initialization}
\begin{algorithmic}[1]
\footnotesize
\Statex \textbf{Input:} Graph $G = (V,E)$; obstacles $X$; sensor probabilities $\pi_x$; disambiguation costs $\delta_x$; risk function $r_x$
\Statex \textbf{Output:} Adjusted graph $G_{\text{adj}}$ with edge costs $\widetilde{\mathcal{C}}_e$ and weights $\delta_e$
\For{each edge $e \in E$}
    \State Identify intersecting obstacles: $X_e \gets \{ x \in X : D(x) \cap e \neq \emptyset \}$
    \State Compute length: $\ell_e \gets \|v_1(e) - v_2(e)\|$
    \State Compute risk: $r_e \gets \frac{1}{2} \sum_{x \in X_e} r_x$
    \State Assign cost: $\widetilde{\mathcal{C}}_e \gets \ell_e + r_e$
    \State Assign weight: $\delta_e \gets \frac{1}{2} \sum_{x \in X_e} \delta_x$
\EndFor
\end{algorithmic}
\end{algorithm}

This preprocessing step equips the graph with risk-aware traversal costs and normalized disambiguation 
weights, facilitating efficient Lagrangian-based optimization under resource constraints. 
It bridges the probabilistic obstacle model with deterministic path planning algorithms.

\section{COLOGR: A Lagrangian \& Graph-Reduction Framework}
\label{sec:COLOGR}

We address the constrained RCDP problem in Equation~\eqref{eqn:WCSPP_restate} 
using a unified framework that combines 
\emph{Cost- and Obstacle-based Graph Reduction (COGR)} 
with \emph{Lagrangian Optimization with Graph Reduction (LOGR)}, 
jointly referred to as the \textbf{COLOGR} method.

The COGR phase initiates the process by pruning infeasible or dominated vertices 
using relaxed versions of the problem that consider cost-only and weight-only criteria.  
This reduces the size of the graph and tightens optimization bounds before invoking 
the full solver.

On the reduced graph, the LOGR phase applies a Lagrangian relaxation of the disambiguation constraint 
by incorporating it into the objective function:
\begin{equation} 
\label{eqn:phi-for-lagrange}
\Phi(\lambda) = \min_{p \in \mathcal{P}(s,t)} \left\{ \widetilde{\mathcal{C}}_p + \lambda (\delta_p - \delta_{\max}) \right\},
\end{equation} 
where $\lambda \geq 0$ is iteratively adjusted to minimize the duality gap.  
At each iteration, further vertex elimination is performed based on updated feasibility 
and dominance evaluations relative to the current best bound.

Together, these two phases implement a \emph{Two-Phase Vertex Elimination (TPVE)} strategy 
that enhances computational efficiency and preserves the optimal solution.  
The resulting \textbf{COLOGR} algorithm integrates structural simplification with 
dual-guided refinement, producing high-quality paths under resource constraints.  

\subsection{Path Planning Using Lagrangian Relaxation}
\label{sec:COLOGR-Lagrangian}

For a fixed $\lambda \geq 0$, let $p^\lambda$ denote the $s$–$t$ path that minimizes the penalized cost 
$\widetilde{\mathcal{C}}_p + \lambda \delta_p$, with corresponding traversal cost $\widetilde{\mathcal{C}}^\lambda$ 
and disambiguation cost $\delta^\lambda$.  
The function $\Phi(\lambda) := \widetilde{\mathcal{C}}^\lambda + \lambda(\delta^\lambda - \delta_{\max})$ 
serves as a valid lower bound on the optimal cost, and is maximized to identify a near-optimal 
Lagrangian multiplier and associated feasible path.

The algorithm begins by evaluating two baseline cases.  
When $\lambda = 0$, the unconstrained minimum-cost path $p^0$ is computed.  
If it satisfies $\delta^0 \leq \delta_{\max}$, it is optimal and the algorithm terminates.  
When $\lambda = \infty$, the disambiguation term dominates, yielding the minimum-weight path $p^\infty$.  
If $\delta^\infty > \delta_{\max}$, no feasible solution exists—this step constitutes the feasibility test.  
While this situation is avoided in our setting (see Section~\ref{sec:MC-simulations}), we retain the test for completeness.  
Rather than computing $p^\infty$ via random tie-breaking, we select the shortest obstacle-free path as 
the minimum-weight solution, improving efficiency.

If neither baseline path resolves the problem, we initialize the upper and lower bounds 
$\widetilde{\mathcal{C}}_U = \widetilde{\mathcal{C}}^\infty$ and 
$\widetilde{\mathcal{C}}_L = \widetilde{\mathcal{C}}^0$, and proceed with iterative updates over $\lambda$.

At each iteration, the algorithm maintains two multipliers, $\lambda^-$ and $\lambda^+$, 
corresponding to paths $p^-$ and $p^+$ such that $\delta^- < \delta_{\max} < \delta^+$.  
A new multiplier $\lambda_{i+1}$ is computed using the intersection formula \citep{Juttner2001}:
\begin{equation}
\label{eqn:lambda_new}
\lambda_{i+1} = \frac{\widetilde{\mathcal{C}}^- - \widetilde{\mathcal{C}}^+}{\delta^+ - \delta^-}.
\end{equation}
This ensures $\lambda^+ \leq \lambda_{i+1} \leq \lambda^-$ due to the monotonicity of both cost and disambiguation terms.

The updated path $p^{\lambda_{i+1}}$ is then computed.  
Depending on whether its disambiguation cost falls below or above the budget, the pair $(\lambda^-, p^-)$ or 
$(\lambda^+, p^+)$ is updated accordingly.  
This process continues until the optimal multiplier $\lambda^*$ is found, as characterized in the following result:

\begin{proposition}[Optimality Condition for Lagrange Multiplier]
\label{prop_optimal_multiplier}
Let $p^-_i$ and $p^+_i$ be the paths with disambiguation costs $\delta^-_i < \delta_{\max} < \delta^+_i$ 
and corresponding costs $\widetilde{\mathcal{C}}^-_i, \widetilde{\mathcal{C}}^+_i$.  
Let $\lambda_{i+1}$ be computed via Equation~\eqref{eqn:lambda_new}, and let 
$p_{i+1} := p^{\lambda_{i+1}}$ denote the path minimizing 
$\Phi(\lambda)$ at $\lambda = \lambda_{i+1}$.  
Then $\lambda_{i+1}$ is optimal if either:
\begin{enumerate}
    \item[(i)] $\delta_{p_{i+1}} = \delta_{\max}$;
    \item[(ii)] All three paths $p^-_i, p^+_i$, and $p_{i+1}$ attain the same penalized cost at $\lambda_{i+1}$:
    \[
    \widetilde{\mathcal{C}}_{p_{i+1}} + \lambda_{i+1} \delta_{p_{i+1}} = \widetilde{\mathcal{C}}^-_i + \lambda_{i+1} \delta^-_i = \widetilde{\mathcal{C}}^+_i + \lambda_{i+1} \delta^+_i.
    \]
\end{enumerate}
\end{proposition}

\begin{proof}
The dual function $\Phi(\lambda)$ is concave and piecewise linear, with each path $p$ contributing a 
linear segment of slope $\delta_p - \delta_{\max}$.  
If $p_{i+1}$ satisfies $\delta_{p_{i+1}} = \delta_{\max}$, then $\Phi'(\lambda_{i+1}) = 0$, 
and $\lambda_{i+1}$ is a maximizer.  
Alternatively, if $p^-_i, p^+_i$, and $p_{i+1}$ yield equal penalized cost at $\lambda_{i+1}$, 
then $\Phi(\lambda_{i+1})$ lies on the convex hull of the dual function and cannot be improved, 
ensuring optimality.
\end{proof}

Despite the method’s strength, Lagrangian-based optimization for constrained shortest path problems is 
not guaranteed to yield an optimal solution in all cases.  
Two known failure modes include: (i) indistinguishability between multiple paths with equal modified costs, 
and (ii) exclusion of the true optimal path from those explored by the iterative scheme \citep{GUO200373, Juttner2001}.  
In our setting, the COGR pruning steps (Section~\ref{sec:COLOGR-GraphReduce}) reduce such risks 
by eliminating suboptimal structures early, accelerating convergence and improving solution quality.

\subsection{Two-Phase Vertex Elimination}
\label{sec:COLOGR-GraphReduce}

To accelerate the Lagrangian optimization process and minimize duality gap risk,
we incorporate a graph reduction strategy that eliminates nonessential vertices before and during optimization.  
Inspired by vertex-pruning heuristics in \cite{Ranga2009}, our approach introduces a two-stage method, 
termed \emph{Two-Phase Vertex Elimination} (TPVE), that substantially shrinks the graph while preserving the optimal path.

The first phase, \textbf{Cost- and Obstacle-based Graph Reduction (COGR)},
operates on structural bounds and applies elimination rules prior to Lagrangian optimization.  
Since removing a vertex implicitly removes all incident edges, 
we focus on vertex elimination for greater efficiency.  
Let $\mathcal{S}_{\min}$ denote any $(s,t)$-vertex cut in $G$.  
Then the optimal surrogate cost $\widetilde{\mathcal{C}}^*$ is bounded by:

\begin{property}[Cut-Based Cost Bounds]
\label{prop:cut_bounds}
Let $\mathcal{S}_{\min}$ be a minimum $(s,t)$-vertex cut. Then
\[
\min_{\substack{p \in \mathcal{P}(s,t) \\ p \cap \mathcal{S}_{\min} \neq \emptyset}} \widetilde{\mathcal{C}}^0_p
\;\leq\;
\widetilde{\mathcal{C}}^*
\;\leq\;
\min_{\substack{p \in \mathcal{P}(s,t) \\ \delta_p \leq \delta_{\max},\; p \cap \mathcal{S}_{\min} \neq \emptyset}} \widetilde{\mathcal{C}}^\infty_p.
\]
\end{property}

The result enables targeted pruning based on vertex-local evaluations of cost and feasibility.  
Specifically, for each vertex $v$, we compute the shortest path through $v$ using both minimum weight ($\lambda = \infty$) and minimum cost ($\lambda = 0$) criteria, denoted $p_v^\infty$ and $p_v^0$, respectively.  
If $p_v^\infty$ violates the disambiguation constraint, then $v$ cannot lie on any feasible path and is removed.  
If $p_v^0$ exceeds the current best upper bound $\widetilde{\mathcal{C}}_U$, it is also pruned.  
These checks are applied dynamically during graph scanning, with bounds updated as feasible candidates are encountered.  
This phase either identifies the optimal solution $p^*$ or tightens the search interval $\left[\widetilde{\mathcal{C}}_L, \widetilde{\mathcal{C}}_U\right]$ for subsequent optimization.

To ensure the validity of this reduction process, 
it is essential to verify that the original $p^*$ is preserved within the reduced graph
(see Corollary \ref{cor:phase1-preservation}).  

The second phase, \textbf{Lagrangian Optimization with Graph Reduction (LOGR)}, 
integrates vertex pruning directly into the Lagrange multiplier update process.  
At each iteration with multiplier $\lambda_i$, 
we compute for every vertex $v$ the shortest modified-cost path $p_v^{\lambda_i}$ 
and use it to update the upper bound (if feasible) or dual lower bound (if not).  
A vertex $v$ is eliminated if its associated dual value at $\lambda_i$ exceeds the current best upper bound.  
This pruning continues throughout the optimization cycle, 
and termination occurs when either the optimal multiplier $\lambda^*$ is identified 
(per Proposition~\ref{prop_optimal_multiplier}) or the duality gap is closed
 ($\widetilde{\mathcal{C}}_L = \widetilde{\mathcal{C}}_U$).

As with COGR, results similar to Proposition \ref{prop:path-preservation}  
can be provided to ensure that the optimal path $p^*$ is not removed from the graph  
throughout the second phase of vertex elimination.

Our TPVE method differs from earlier SNE strategies \citep{Ranga2009} in three critical aspects:  
(i) it employs a pre-optimization reduction phase (COGR);  
(ii) it uses a convergence-aware stopping condition that aligns with Proposition~\ref{prop_optimal_multiplier}; and  
(iii) it retains vertices whose lower bounds match the current upper bound,
preserving potential ties in penalized cost and preventing premature elimination of optimal candidates.  
These enhancements collectively yield significant performance gains and robust convergence guarantees, 
as justified in Section~\ref{sec:theo-guarantee-complexity}
and demonstrated in Section~\ref{sec:MC-simulations}.

\subsection{COLOGR: Combining TPVE \& Lagrangian Search}
\label{sec:COLOGR-comb}

\begin{algorithm}[H]
\caption{Unified Lagrangian Optimization with Graph Reduction (COLOGR)}
\label{alg:unified}
\begin{algorithmic}[1]
\footnotesize
\Statex \textbf{Input:} Adjusted graph $G_{\text{adj}}$ (from Algorithm~\ref{graph_initialization}), source $s$, target $t$, budget $\delta_{\max}$
\Statex \textbf{Output:} Optimal or best feasible path $p^*$
\State Initialize $V_{\text{del}} \gets \emptyset$, $E_{\text{del}} \gets \emptyset$
\State Compute $p^\infty$ (min-weight path) and $p^0$ (min-cost path)
\If{$\delta_{p^0} \leq \delta_{\max}$} \Return $p^* = p^0$ \EndIf
\State Set bounds: $p_U \gets p^\infty$, $p_L \gets p^0$
\Repeat \Comment{Phase 1: Vertex Elimination (COGR)}
    \For{$v \in V(G_{\text{adj}})$}
        \State Compute $p^\infty_v$, $p^0_v$
        \If{$\delta_{p^\infty_v} > \delta_{\max}$ or $\widetilde{\mathcal{C}}_{p^0_v} > \widetilde{\mathcal{C}}_{p_U}$}
            \State Eliminate $v$: $V_{\text{del}} \gets V_{\text{del}} \cup \{v\}$;
            \State Remove incident edges: $E_{\text{del}} \gets E_{\text{del}} \cup \{e \in E: v \in e\}$
            \State Update $G_{\text{adj}} \gets G(V \setminus V_{\text{del}}, E \setminus E_{\text{del}})$
         \ElsIf{$\widetilde{\mathcal{C}}_{p^\infty_v} < \widetilde{\mathcal{C}}_{p_U}$ or $\widetilde{\mathcal{C}}_{p^0_v} < \widetilde{\mathcal{C}}_{p_U}\,\&\,\delta_{p^0_v} \leq \delta_{\max}$}
            \State Update $p_U \gets p^\infty_v$ or $p_U \gets p^0_v$
        \EndIf
    \EndFor
\Until{no further eliminations}
\While{duality gap remains open}
    \State $\lambda \gets \frac{\widetilde{\mathcal{C}}_{p_U} - \widetilde{\mathcal{C}}_{p_L}}{\delta_{p_L} - \delta_{p_U}}$
    \State Find $p_{\lambda}$ minimizing $\widetilde{\mathcal{C}}_p + \lambda \delta_p$
    \If{$\delta_{p_{\lambda}} = \delta_{\max}$} \Return $p^* = p_{\lambda}$ \EndIf
    \For{$v \in V(G_{\text{adj}})$} \Comment{Phase 2: Vertex Elimination (LOGR)}
        \State Compute $p_{\lambda,v}$ through $v$
        \If{$\widetilde{\mathcal{C}}_{p_{\lambda,v}} < \widetilde{\mathcal{C}}_{p_U}$ and $\delta_{p_{\lambda,v}} \leq \delta_{\max}$}
            \State Update $p_U \gets p_{\lambda,v}$
        \ElsIf{$\Phi(\lambda, p_{\lambda,v}) > \widetilde{\mathcal{C}}_{p_U}$}
            \State Eliminate $v$: $V_{\text{del}} \gets V_{\text{del}} \cup \{v\}$;
            \State Remove incident edges: $E_{\text{del}} \gets E_{\text{del}} \cup \{e \in E: v \in e\}$
            \State Update $G_{\text{adj}} \gets G(V \setminus V_{\text{del}}, E \setminus E_{\text{del}})$
        \EndIf
    \EndFor
    \State Update $p_L$ and bounds
\EndWhile
\State \Return $p^* = p_U$
\end{algorithmic}
\end{algorithm}

\noindent
\textbf{Remark.}  
Throughout this manuscript, we use the terms \textit{COLOGR} and 
\textit{Two-Phase Vertex Elimination (TPVE)} interchangeably.  
They refer to the same unified graph reduction and 
optimization framework introduced in Sections~\ref{sec:COLOGR-comb}--\ref{sec:COLOGR-GraphReduce}.  
Specifically, COLOGR is the algorithmic realization of TPVE, 
combining pre-optimization reduction (COGR) with in-process vertex pruning 
guided by Lagrangian dual bounds (LOGR).  
All theoretical results and simulation benchmarks reported 
under the TPVE name are based on the implementation given in Algorithm~\ref{alg:unified}. $\square$

\subsection{Guarantees and Complexity of COLOGR}
\label{sec:theo-guarantee-complexity}

We now provide guarantees on solution correctness and computational complexity. 
These results justify the effectiveness of TPVE when integrated with the Lagrangian scheme,
establishing that no part of the true optimal path is pruned despite aggressive graph reduction. 
The following proposition ensures that the constrained optimal solution is preserved through both phases of TPVE.

\begin{proposition}[Preservation of Optimal Path under TPVE]
\label{prop:path-preservation}
Let $p^*$ be the optimal solution to the RCDP problem with total approximate cost 
$\widetilde{\mathcal{C}}^*$ and disambiguation cost $\delta_{p^*} \leq \delta_{\max}$. 
Then the path $p^*$ is retained in the graph after both COGR and LOGR phases of TPVE.
\end{proposition}

\begin{proof}
Let $p^*$ be an optimal $s$–$t$ path such that $\delta_{p^*} \leq \delta_{\max}$ and 
$\widetilde{\mathcal{C}}_{p^*} = \widetilde{\mathcal{C}}^*$.  
Consider any vertex $v \in p^*$. 

In Phase 1 (COGR), $v$ is preserved if:
(i) there exists a feasible path through $v$ with $\delta_p \leq \delta_{\max}$,  
and (ii) at least one such path has cost no greater than the current upper bound $\widetilde{\mathcal{C}}_U$.  
Both conditions hold for $p^*$, so $v$ is not eliminated.

%
In Phase 2 (LOGR), $v$ is preserved if the lower bound of the optimal path cost via $v$ does not exceed the current upper bound, 
i.e., $\Phi(\lambda, p_{\lambda,v}) \leq \widetilde{\mathcal{C}}_{p_U}$.
For $v$ that lies on $p^*$, any identified lower bound $\Phi(\lambda,p_{\lambda,v})$ always satisfies
$\Phi(\lambda,p_{\lambda,v})\leq\tilde{\mathcal{C}}_{p^*}$,
and the current upper bound satisfies $\tilde{\mathcal{C}}_U\geq\tilde{\mathcal{C}}_{p^*}$,
since no feasible path can be cheaper than the optimal one.
Therefore, $v$ is not eliminated. 

Since this reasoning applies to every vertex $v \in p^*$ and both phases of TPVE, 
the entire path $p^*$ remains in the final reduced graph.
\end{proof}

The following corollary isolates the preservation guarantee under Phase 1 of TPVE.  
It formalizes the intuitive observation that the optimal path $p^*$ remains in the graph  
as long as it is both feasible and cost-competitive under initial pruning rules.

\begin{corollary}[Preservation under Phase 1 Only]
\label{cor:phase1-preservation}
Under the same assumptions as Proposition~\ref{prop:path-preservation},  
the optimal path $p^*$ is preserved after Phase 1 of TPVE.
\end{corollary}

As noted earlier in Section~\ref{sec:COLOGR-Lagrangian}, 
the Lagrangian relaxation method may fail to recover the optimal path in certain degenerate cases.  
Nevertheless, in all simulation replications reported in Section~\ref{sec:MC-simulations},  
our proposed algorithm (COLOGR) consistently produced the optimal solution $p^*$.  
This section offers a structural explanation for this empirical robustness, 
focusing on cases with non-unique minimizers of the Lagrangian objective.

One such scenario involves multiple paths minimizing the modified cost  
$\widetilde{\mathcal{C}}_p + \lambda \delta_p$, 
only one of which corresponds to the true optimal path $p^*$ with respect to the original cost $\widetilde{\mathcal{C}}$.  
While the standard Lagrangian framework does not guarantee recovery of $p^*$ in this setting, 
the integration of the two-phase graph reduction process (TPVE) can resolve the ambiguity.  
The following proposition formalizes a sufficient condition for successful identification of $p^*$ in this case.

\begin{proposition}[Identification of $p^*$ via Unique Vertex Signature]
\label{prop-identify optimal case 1}
Suppose that for some Lagrange multiplier $\lambda^*$, multiple $(s,t)$-paths minimize the Lagrangian cost
\[
\widetilde{\mathcal{C}}_p + \lambda^*(\delta_p - \delta_{\max}),
\]
and that the true optimal path $p^*$ for the RCDP problem is among them.  
If there exists a vertex $v \in p^*$ such that no other minimum-cost path (at $\lambda^*$) passes through $v$, 
then the TPVE algorithm correctly identifies $p^*$.
\end{proposition}

\begin{proof}
Since $p^*$ is feasible and attains the minimum relaxed cost at $\lambda^*$, 
it satisfies both the constraint $\delta_{p^*} \leq \delta_{\max}$ and the dual optimality condition.  
By assumption, vertex $v \in p^*$ is exclusive to $p^*$ among all minimum-cost paths at $\lambda^*$.  
The COLOGR algorithm (implementing the TPVE framework) evaluates, for each vertex, the best path passing through it.  
For $v$, this path must be $p^*$, as no other minimum-cost path shares that vertex.  
Thus, the vertex-specific structure uniquely identifies $p^*$, 
allowing the algorithm to recover it even in the presence of ties.
\end{proof}

The previous result demonstrates how structural uniqueness—via exclusive vertex participation—can
resolve ambiguity among multiple dual minimizers and guide the TPVE subroutine to the correct path. 

Beyond this, the COLOGR algorithm often succeeds even in the absence of structural uniqueness or zero duality gap. 
In particular, when the optimal path $p^*$ is feasible and minimizes the Lagrangian cost 
for some multiplier $\lambda \geq 0$, our algorithm still identifies $p^*$ correctly. 
This robustness explains the consistent empirical performance observed in Section~\ref{sec:MC-simulations}, 
even under conditions where strong duality is not theoretically guaranteed.

\begin{proposition}[Correct Recovery under Dual Attainment]
\label{prop-identify optimal case 2}
Suppose the optimal path $p^*$ for the RCDP problem is feasible, 
i.e., $\delta_{p^*} \leq \delta_{\max}$, 
and also minimizes the Lagrangian objective for some $\lambda \geq 0$:
\[
p^* \in \arg\min_{p \in \mathcal{P}(s,t)} \left\{ \widetilde{\mathcal{C}}_p + \lambda (\delta_p - \delta_{\max}) \right\}.
\]
Then the COLOGR algorithm correctly identifies $p^*$ as the optimal path, 
and $\lambda$ serves as an optimal Lagrange multiplier.
\end{proposition}

\begin{proof}
Assume $p^*$ is feasible and minimizes the relaxed cost function at $\lambda$, i.e.,
\[
\widetilde{\mathcal{C}}_{p^*} + \lambda(\delta_{p^*} - \delta_{\max}) = \Phi(\lambda),
\]
where $\Phi(\lambda)$ is the Lagrangian dual function. We consider two cases:

\textbf{Case 1:} $\delta_{p^*} = \delta_{\max}$.  
Then the penalty term vanishes, and $\Phi(\lambda) = \widetilde{\mathcal{C}}_{p^*}$.  
Since $p^*$ is feasible and attains this minimum, strong duality holds, and COLOGR terminates with $p^*$ as the optimal path.

\textbf{Case 2:} $\delta_{p^*} < \delta_{\max}$.  
Although the penalty term becomes negative, $p^*$ still minimizes the Lagrangian objective at $\lambda$.  
Because $p^*$ is feasible, its cost updates the best known upper bound $\widetilde{\mathcal{C}}_U$.  
Since no subsequent path can improve upon this bound, 
the upper bound stabilizes at $\widetilde{\mathcal{C}}_{p^*}$, and COLOGR eventually terminates with $p^*$.  
Moreover, as COLOGR evaluates minimum-cost paths through all vertices in Phase 2, 
and $p^*$ is among these, it is guaranteed to be retained.

Thus, regardless of whether the constraint is active at optimality, the COLOGR algorithm recovers $p^*$ and identifies the corresponding $\lambda$ as optimal.
\end{proof}

Building on Proposition~\ref{prop-identify optimal case 2}, 
we now examine a complementary setting in which 
the optimal RCDP path $p^*$ satisfies $\delta_{p^*} = \delta_{\max}$ 
but is not among the minimizers of the Lagrangian objective at the corresponding dual optimum.
Suppose under the optimal multiplier $\lambda^*$, 
there exist two paths $p_1$ and $p_2$ satisfying:
\begin{equation} \label{optimality_case2}
\widetilde{\mathcal{C}}_{p_1}+\lambda^*(\delta_{p_1}-\delta_{\max})=
\widetilde{\mathcal{C}}_{p_2}+\lambda^*(\delta_{p_2}-\delta_{\max})<\widetilde{\mathcal{C}}^*,
\end{equation}
with $\delta_{p_1} > \delta_{\max}$ and $\delta_{p_2} < \delta_{\max}$. 
Although $p_1$ violates the disambiguation budget and $p_2$ underutilizes it, 
neither satisfies the constraint with equality as $p^*$ does; 
hence, $p^*$ is excluded from the argmin of the dual objective at $\lambda^*$, 
despite being optimal under the original constraint.
The propositions below show that despite this, 
the proposed TPVE algorithm correctly identifies $p^*$ in both special and general cases.

\begin{remark}
\label{rem:assumptions-find-optimal}
\textbf{(Assumptions for Propositions~\ref{prop-identify opt case2-special}--\ref{prop:general-B-disamb})}
We first state some assumptions and introduce notation to be used 
in Propositions~\ref{prop-identify opt case2-special}--\ref{prop:general-B-disamb} below.
Let $X = \{x_1, \dots, x_k\}$ be a finite set of $k \geq 2$ ambiguous obstacles, 
each with disambiguation cost $\delta_x > 0$.
Let $\delta_{\max} > 0$ be the total disambiguation budget, 
and let $B$ be the maximal number of disambiguations permitted under this budget; 
i.e., at most $B$ obstacles $x$ can be disambiguated such that 
$\sum_{x \in S} \delta_x \leq \delta_{\max}$ for all $S \subseteq X$ with $|S| \leq B$.
Suppose there exists a path $p^*$ that intersects a subset of obstacles $X^* \subseteq X$ such that:
(i) $|X^*| \leq B$, and $\sum_{x \in X^*} \delta_x \leq \delta_{\max}$;
(ii) any other path $p \neq p^*$ either intersects a set $X_p$ of obstacles 
    with $\sum_{x \in X_p} \delta_x > \delta_{\max}$ (infeasible), 
    or avoids all obstacles but incurs strictly higher cost than $p^*$.
\end{remark}

\begin{proposition}[Correct Identification in a Two-Obstacle Scenario]
\label{prop-identify opt case2-special}
In an RCDP instance under the assumptions in Remark \ref{rem:assumptions-find-optimal} with $k=2$ and $|X^*| = B = 1$ 
(i.e supposing there are $k=2$ ambiguous obstacles $X = \{x_1, x_2\}$ in the environment, 
with $\min\{\delta_{x_1}, \delta_{x_2}\} \le \delta_{\max} < \max\{\delta_{x_1}, \delta_{x_2}\}$
and $p^*$ intersects exactly one of the obstacles, 
while any other path $p \neq p^*$ intersects both (infeasible) or neither but more costly than $p^*$)
the COLOGR algorithm correctly eliminates non-optimal paths and identifies $p^*$ as optimal.
\end{proposition}

\begin{proposition}[Identification under Single Disambiguation Budget]
\label{prop:single-disamb-induction}
In an RCDP instance under the assumptions in Remark \ref{rem:assumptions-find-optimal} with $k \ge 2$ and $|X^*| = B = 1$ 
(i.e. supposing there are $k \ge 2$ ambiguous obstacles $X = \{x_1, \ldots, x_k\}$, 
and the disambiguation budget satisfies 
$\delta_{\max} < \min_{x_i + x_j} \delta_{x_i} + \delta_{x_j}$ for all distinct $x_i, x_j \in X$, 
and that there exists a feasible path $p^*$ intersecting exactly one obstacle $x^*$ 
such that $\delta_{x^*} \leq \delta_{\max}$)
and all other paths satisfy condition (ii) of Remark~\ref{rem:assumptions-find-optimal}.
COLOGR algorithm eliminates all such 
non-optimal paths and identifies $p^*$ as the optimal path.
\end{proposition}

\begin{proposition}[Identification under $B$-Budget Disambiguation]
\label{prop:general-B-disamb}
Under the assumptions of Remark~\ref{rem:assumptions-find-optimal} 
with general $k \ge 2$, $B \ge 1$ and $|X^*| \le B$
%
COLOGR algorithm eliminates all such 
non-optimal paths and identifies $p^*$ as the optimal path.
\end{proposition}

\begin{remark}
Note that in practice, we assume $B \leq k$, since at most $k$ distinct obstacles can be disambiguated. 
If $B > k$, then $B$ becomes redundant, and the only active constraint is whether the total disambiguation cost 
$\sum_{x \in S} \delta_x$ for any $S \subseteq X$ satisfies $\sum_{x \in S} \delta_x \leq \delta_{\max}$.
\end{remark}

The results in Propositions~\ref{prop-identify opt case2-special}--\ref{prop:general-B-disamb} 
(proof provided in the Appendix (Section \ref{app: proofs_sec3_1}))
establish that even when the Lagrangian objective does not favor the optimal path $p^*$, 
the structural pruning performed by TPVE preserves and ultimately identifies it. 
This includes scenarios where dual-optimal paths $p_1$ and $p_2$ mislead the optimizer 
despite being infeasible or suboptimal, as formalized in Equation~\eqref{optimality_case2}.

To further illustrate the utility of structural pruning, consider the representative case from 
Proposition~\ref{prop-identify opt case2-special}, where three competing paths exist:  
$p_1$ (infeasible with lower cost), $p_2$ (obstacle-free but costly), and $p^*$ 
(feasible and optimal under the constraint).  
Although $p_1$ violates the disambiguation budget ($\delta_{p_1} > \delta^*$) 
and $p_2$ is feasible but longer ($\delta_{p_2} < \delta^*$, 
$\widetilde{\mathcal{C}}_{p_2} > \widetilde{\mathcal{C}}^*$), 
the penalty structure in the dual formulation can cause both to appear superior to $p^*$ 
under certain multipliers, as illustrated in Equation~\eqref{optimality_case2}.

To mitigate this distortion and improve alignment between the dual and primal solutions, 
we adopt a strategy that increases the relative influence of the true traversal cost 
$\widetilde{\mathcal{C}}_p$ over the penalty term $\lambda(\delta_p - \delta_{\max})$. 
This is achieved by tuning the risk function parameters so that the risk term $r_p$ more sharply 
reflects obstacle uncertainty. In the RCDP setting, $r_p$ is determined by the structure of the 
environment and a user-defined exploration parameter (e.g., $\alpha$ in the linear undesirability model), 
along with the blockage probability marks $\pi_X$.

While disambiguation costs and obstacle probabilities are externally specified, 
the planner can still regulate the sensitivity of risk through the choice of risk function and its parameters. 
To formalize this idea, we define the \emph{risk blockage gradient (RBG)} as a measure of how steeply 
the path risk $r_p$ increases with blockage probability. 
This formulation generalizes the role of tuning parameters like $\alpha$ in the LU model 
by framing sensitivity in terms of a gradient concept applicable across risk models.

A higher RBG imposes steeper penalties on ambiguous paths, 
improving the separation between feasible and infeasible solutions in the dual objective. 
This enhances the convergence of COLOGR and reduces the likelihood of spurious dual optima.

Similar arguments apply when $\delta^* < \delta_{\max}$, 
though this case is often ruled out by structural dominance. 
If a feasible path with $\delta_p = \delta_{\max}$ has higher cost than $p^*$, 
and $p^*$ requires less disambiguation, then $p^*$ dominates such paths in both feasibility and cost. 
Realistically, adversarial environments are designed to avoid such dominance relationships, 
further supporting the utility of dual-guided reduction.

This intuition is formalized in the following property (proof provided in the Appendix (Section \ref{app: proofs_sec3_2})).

\begin{property}[Spurious Dual Optima Vanish under High RBG]
\label{prop:risk-blockage-gradient}
Let $p^*$ denote the optimal path for the RCDP problem, and let $\lambda^*$ be a Lagrange multiplier 
such that multiple paths minimize the relaxed objective $\widetilde{\mathcal{C}}_p + \lambda^*(\delta_p - \delta_{\max})$. 
Suppose some of these minimizers are either infeasible or suboptimal under the original constraint.  
Then, as the risk blockage gradient (RBG) increases—i.e., as the risk function becomes more sensitive 
to blockage probabilities—the likelihood that such spurious paths remain dual-optimal diminishes.  
For sufficiently high RBG, the surrogate cost $\widetilde{\mathcal{C}}_p$ dominates the dual objective, 
aligning the relaxed minimizers with the true constrained optimum $p^*$.
\end{property}

We next analyze the worst-case time complexity of the proposed COLOGR algorithm.

\begin{theorem}[Time Complexity of COLOGR Algorithm]
\label{thm:rdp-complexity}
Let $ G = (V, E) $ be a spatial grid graph with $ n_v = |V| $ vertices and $ n_e = |E| $ edges. 
Then the COLOGR algorithm solves the RCDP problem with worst-case time complexity:
\[
\mathcal{O}(n_v n_e^2 \log n_e (n_e + n_v \log n_v)).
\]
\end{theorem}

\begin{proof}
The COLOGR algorithm proceeds in two phases:

In \textbf{Phase 1}, COGR applies vertex pruning based on feasibility and cost bounds.  
Each iteration involves:
\begin{enumerate}
    \item Computing shortest paths with respect to weight and cost via Dijkstra’s algorithm: $ \mathcal{O}(n_e + n_v \log n_v) $.
    \item Evaluating elimination criteria and updating bounds: $ \mathcal{O}(n_v + n_e) $.
\end{enumerate}
This phase may iterate over all $n_v$ vertices, leading to a total cost of $ \mathcal{O}(n_v (n_e + n_v \log n_v)) $.

In \textbf{Phase 2}, LOGR solves the Lagrangian relaxation using breakpoint updates over the multiplier $\lambda$.
The number of breakpoints is bounded by $ \mathcal{O}(n_e^2 \log n_e) $ \citep{juttner2005resource}, 
and each iteration involves a shortest path computation and vertex-level updates, costing $ \mathcal{O}(n_v (n_e + n_v \log n_v)) $.
Hence, the total cost of LOGR is:
\[
\mathcal{O}(n_v n_e^2 \log n_e (n_e + n_v \log n_v)).
\]

The overall complexity is dominated by Phase 2, establishing the stated bound.
\end{proof}

\noindent
In practice, the initial pruning in COGR substantially reduces the graph size, leading to faster convergence in the Lagrangian phase.  
This yields a notable improvement over baseline methods such as \textsc{SNE}, which incur higher complexity of 
$\mathcal{O}(n_v^2 n_e^2 \log^4 n_e (n_e + n_v \log n_v))$.

\section{Design and Analysis of the RCDP Traversal Policy}
\label{sec:RCDP-policy}

Building on the COLOGR algorithm introduced in Section~\ref{sec:COLOGR-comb}, 
we now present the complete \emph{RCDP traversal protocol}, 
which integrates duality-based optimization with two-phase graph reduction. 
This section formalizes the traversal strategy, 
establishes theoretical guarantees, and illustrates its behavior using a stylized grid-based example. 
The decision process is visualized in Figure~\ref{fig:flowchart}, and 
appendix provides detailed path visualizations and execution logs for reproducibility.

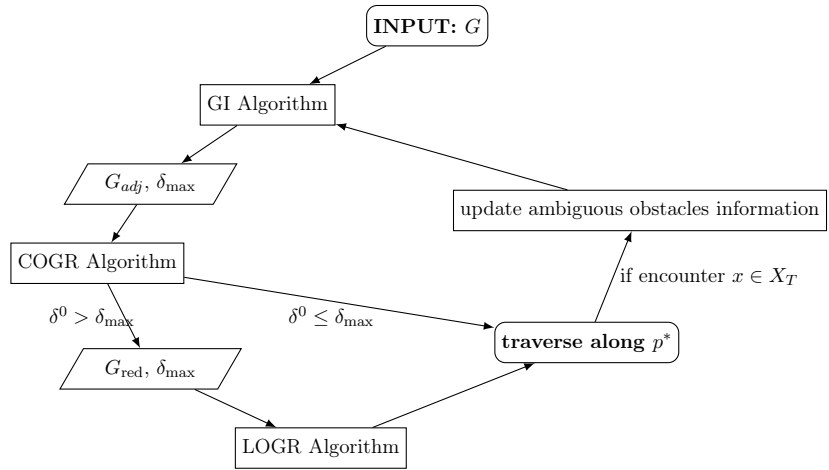
\begin{figure}[h]
	\centering
\begin{tikzpicture}[scale=0.7,transform shape]
    \node [terminator, fill=white!20] at (0.0, 0) (start) {\textbf{INPUT: $G$}};
    \node [process, fill=white!20] at (-3.0, -1.5) (process1) {GI Algorithm};
    \node [data, fill=white!20] at (-5.2, -3) (data1) {$G_{adj}$, $\delta_{\max}$};
    \node [process, fill=white!20] at (-6.2, -4.5) (process2) {COGR Algorithm};

    \node [data, fill=white!20] at (-5.2, -6.5) (data2) {$G_{\text{red}}$, $\delta_{\max}$};
    \node [process, fill=white!20] at (-2.0, -8) (process3) {LOGR Algorithm};

    \node [terminator, fill=white!20] at (3, -6) (result2) {\textbf{traverse along $p^*$}};
    \node [process, fill=white!20] at (4, -3.5) (process4) {update ambiguous obstacles information};

    \node[draw=none] at (-6.3, -5.5) (no) {$\delta^0 > \delta_{\max}$};
    \node[draw=none] at (-1.8, -5.5) (yes1) {$\delta^0 \leq \delta_{\max}$};
    \node[draw=none] at (5.3, -4.75) (yes2) {if encounter $x \in X_T$};

    \path [connector] (start) -- (process1);
    \path [connector] (process1) -- (data1);
    \path [connector] (data1) -- (process2);
    \path [connector] (process2) -- (data2);
    \path [connector] (data2) -- (process3);
    \path [connector] (process3) -- (result2);
    \path [connector] (process2) -- (result2); 
    \path [connector] (result2) -- (process4);
    \path [connector] (process4) -- (process1);
\end{tikzpicture}

    \caption{Schematic overview of the RCDP traversal protocol (Algorithm~\ref{alg:RCDP}).}
	\label{fig:flowchart}
\end{figure}

\subsection{Algorithmic Framework}
\label{sec:policy-description}

The RCDP policy operates on the adjusted graph $G_{\text{adj}}$ generated by Algorithm~\ref{graph_initialization}.  
At each step, it solves the constrained surrogate optimization problem in Equation~\eqref{eqn:WCSPP_restate}, 
examines the first edge of the resulting path, and decides whether to disambiguate based on the available budget.  
If disambiguation is feasible, the agent updates obstacle information; otherwise, it re-plans accordingly.

The full traversal logic is provided in Algorithm~\ref{alg:RCDP}.  
The subroutines comprising COLOGR are detailed in the Appendix.

\begin{algorithm}[H]
\caption{RCDP Traversal Policy}
\label{alg:RCDP}
\begin{algorithmic}[1]
\footnotesize
\Statex \textbf{Input:} $G_{\text{adj}}$, start $s$, target $t$; risk function $r_x$; disambiguation budget $\delta_{\max}$
\Statex \textbf{Output:} Traversal path $p^{*}$
\State $v \gets s$;\; $p \gets \{s\}$
\While{$v \neq t$}
  \State Compute path $p^{\text{next}}$ from $v$ to $t$ via Equation~\eqref{eqn:WCSPP_restate}
  \State Let $e = (v,u)$ be the first edge on $p^{\text{next}}$
  \If{$e$ intersects ambiguous obstacle $x$}
      \State Disambiguate $x$;\; $\delta_{\max} \gets \delta_{\max} - \delta_x$
      \If{$x$ is blocking} \textbf{continue} \Comment{Re-plan due to blocked edge}
      \EndIf
  \EndIf
  \State $v \gets u$;\; $p \gets p \cup \{u\}$
\EndWhile
\State \Return $p^{*} = p$
\end{algorithmic}
\end{algorithm}

\subsection{Theoretical Guarantees and Benchmark Comparison}
\label{sec:theory-policy}

In this section, we establish key theoretical guarantees for the proposed RCDP policy.  
We focus on two aspects: (i) feasibility with respect to the disambiguation budget constraint,  
and (ii) optimality of the surrogate cost under the selected risk function.  
We then compare the performance of RCDP and greedy traversal policies under a unified risk 
modeling framework.

\noindent
\textbf{(i) Feasibility:}
Let $\delta_p$ denote the total disambiguation cost incurred along path $p$.  
By construction, Algorithm~\ref{alg:RCDP} ensures that 
the disambiguation budget constraint $\delta_{p^*} \leq \delta_{\max}$ is respected throughout the traversal.  
This guarantees that the disambiguation budget constraint is never violated throughout the 
traversal.

\noindent
\textbf{(ii) Approximate Optimality in Surrogate Metric:}
At each iteration, the solver for the Lagrangian-relaxed problem in Equation~\eqref{eqn:WCSPP_restate} 
returns a globally optimal path on the reduced graph, minimizing the surrogate cost 
$\widetilde{\mathcal{C}}_p$ over the current feasible set.  
The graph reduction procedure, executed via COLOGR Algorithm, 
preserves $s$–$t$ connectivity and ensures that surrogate cost comparisons 
among retained paths remain valid, thereby supporting consistent surrogate optimization after graph reduction.

\begin{theorem}[Feasibility and Surrogate Optimality of COLOGR]
\label{thm:policy-correctness}
Let $G_{\text{adj}} = (V, E)$ be the reduced graph produced by the COLOGR algorithm for the RCDP problem, 
and let $r_x$ be a fixed risk function used to compute the surrogate edge cost $r_e$ for each edge $e \in E$. 
Let $\widetilde{\mathcal{C}}_p = \sum_{e \in p} (\ell_e + r_e)$ denote
the surrogate cost for path $p \in \mathcal{P}(s,t; G_{\text{adj}})$, as introduced in Equation~\eqref{eqn:surrogate-cost}.
Assume that there exists a path $p^* \in \mathcal{P}(s,t; G_{\text{adj}})$ that 
both minimizes the Lagrangian-relaxed objective at the optimal multiplier $\lambda^*$ 
and satisfies the budget constraint $\delta_{p^*} \le \delta_{\max}$. 
Then $p^*$ is a feasible surrogate-optimal solution:
\[
\delta_{p^*} \leq \delta_{\max}, \quad \text{and} \quad 
\widetilde{\mathcal{C}}_{p^*} = \min \left\{ \widetilde{\mathcal{C}}_p : 
p \in \mathcal{P}(s,t; G_{\text{adj}}), \; \delta_p \leq \delta_{\max} \right\}.
\]
\end{theorem}

\noindent\textit{Proof provided in the Appendix (Section \ref{app: proofs_sec4}).}

\paragraph{Comparison with Greedy Policies.} 
In low-density environments or when the disambiguation budget is relatively high  
(e.g., $n = 20$ or $40$ obstacles, and 
maximum number of disambiguation allowed is $N_{\max} = 3$ or 
a total cost limit $\delta_{\max} = 4,6$),  
some greedy policies may appear competitive with RCDP variants.  
This is expected, as the advantage of informed disambiguation is less pronounced  
in such scenarios.  
Nonetheless, across a wide range of settings and under consistent risk modeling,  
the RCDP policy demonstrably yields lower expected traversal costs  
than its greedy counterparts.

We now formalize this performance guarantee under the assumption that the risk function  
$r_x$ used to define $\widetilde{\mathcal{C}}_p$ is nondecreasing in the disambiguation cost  
$\delta_x$ and nondecreasing in the probability of blockage $(\pi_x)$.

\begin{property}
\label{prop:conditional-dominance}
\textbf{(Conditional Dominance of RCDP Policy)}
Under suitable conditions on the surrogate risk model—specifically, when the risk function is non-decreasing in the probability of obstacle blockage, and surrogate cost ordering is consistent with expected cost ordering—the RCDP policy achieves lower expected traversal cost than any greedy baseline policy that ignores resource constraints in its initial plan.
\end{property}
\noindent
\textit{Formal Statement and Proof.} The full proposition and its proof are presented in the Appendix (Section \ref{app: proofs_sec4}).
This property provides a theoretical justification for the empirical performance advantage of the RCDP policy observed in our experiments (Section~\ref{sec:MC-simulations}). While the result relies on strong assumptions regarding surrogate model fidelity, it illustrates how constraint-aware optimization can outperform greedy heuristics, especially in environments with heterogeneous risk and cost structures.

\subsubsection*{Bayesian Tuning of Linear Undesirability Risk Function and Consistency Result}
To operationalize the linear undesirability (LU) risk function, 
we examined the impact of the tuning parameter $\alpha$ on traversal performance 
across varying obstacle densities and true obstacle proportions.
Our results indicate that higher $\alpha$ values are beneficial in environments with denser true obstacles, 
as they enhance the policy's ability to penalize risky paths.
We also propose a Bayesian-adjusted extension in which $\alpha$ is obstacle-specific, 
determined via posterior estimates of being a true obstacle.
Full details, including simulation results and a principled derivation 
of the Bayesian linear undesirability function $r^{LB}$, 
are provided in the Appendix (Section~\ref{app: sec_alpha}).

Moreover, we show that the RCDP policy using the Bayesian LU risk function achieves asymptotic optimality: 
as sensor accuracy improves, 
its expected cost converges to that of the offline benchmark. 
This convergence result provides further justification 
for using the Bayesian-adjusted LU model in high-precision settings.

\subsection{Illustrative Example}
\label{sec:toy-example}

To illustrate the RCDP traversal policy and benchmark its performance against a greedy baseline, 
we consider a simple test case. 
Specifically, we compare it against the constrained RD policy from \citet{aksakalliari2014}. 
Extensive simulation-based comparisons are deferred to Section~\ref{sec:MC-simulations}.

We define a compact traversal region $\Omega = [,0,22] \times [,0,14]$ and discretize it using an 8-adjacency integer grid. The source and target locations are $s = (11,14)$ and $t = (11,1)$, respectively. Six disk-shaped obstacles of radius $r = 2.5$ are placed in the region, with blockage probabilities $\pi = (0.7, 0.2, 0.1, 0.3, 0.1, 0.4)$ and disambiguation costs $\delta = (1, 2, 1, 1, 2, 2)$. The traversal budget is $\delta_{\max} = 2$.

Figure~\ref{fig:toy-example} presents this scenario and the resulting traversals. The RCDP policy—solved using COLOGR with the undesirability risk function ($\alpha = 15$)—identifies the optimal path on a reduced graph (panel b). The first disambiguation occurs at obstacle 4. Depending on its status:
\begin{itemize}[,leftmargin=*]
\item[(i)] If false, the agent proceeds directly to $t$ (panel c).
\item[(ii)] If true, a replan is triggered from the disambiguation node with updated budget $\delta_{\max} = 1$ (panel d).
\end{itemize}
This results in an expected cost of:

$$
\mathbf{E}[\mathcal{C}_p] = 0.7 \cdot (19.8995 + 1) + 0.3 \cdot (23.5563 + 1) = 21.9966.
$$

By contrast, the constrained RD policy (panel e) initiates disambiguation at obstacle 3. If it is false, the agent attempts obstacle 5, but its cost exceeds the remaining budget. It then redirects to obstacle 4. If obstacle 3 is true, the agent bypasses directly to obstacle 4. The expected cost is:
\begin{align*}
\mathbf{E}[\mathcal{C}_p] = & (0.9 \cdot 0.7)(24.7280 + 2) + (0.1 \cdot 0.7)(20.4853 + 2) \\
& + (0.1 \cdot 0.3)(24.1421 + 2) + (0.9 \cdot 0.3)(31.3136 + 2) = 28.1916,
\end{align*}
which is substantially higher than under the RCDP policy.

This example highlights the key advantage of constraint-aware planning: the RCDP policy integrates the disambiguation budget directly into traversal decisions, yielding more cost-effective plans than greedy approaches that lack global constraint awareness.

\begin{figure}[htbp]
\centering
\begin{subfigure}[,b]{0.48\textwidth}
\centering
\includegraphics[,width=\textwidth]{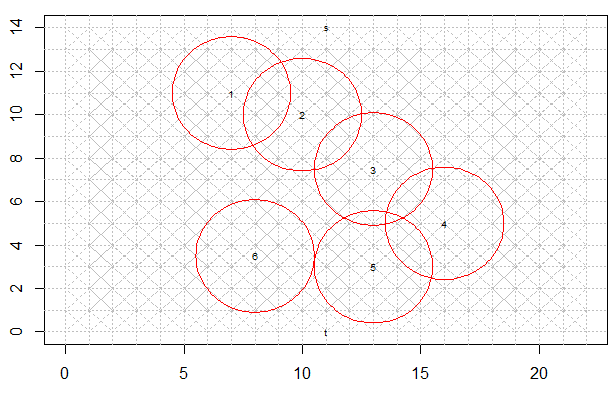}
\caption{Initial obstacle configuration}
\end{subfigure}
\hfill
\begin{subfigure}[,b]{0.48\textwidth}
\centering
\includegraphics[,width=\textwidth]{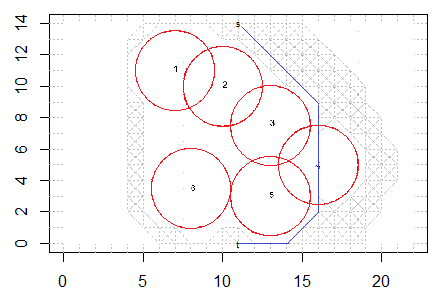}
\caption{Reduced graph and planned path (RCDP)}
\end{subfigure}
\\[0.3cm]
\begin{subfigure}[,b]{0.48\textwidth}
\centering
\includegraphics[,width=\textwidth]{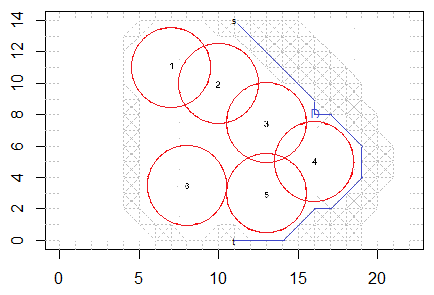}
\caption{Path after false obstacle (RCDP)}
\end{subfigure}
\hfill
\begin{subfigure}[,b]{0.48\textwidth}
\centering
\includegraphics[,width=\textwidth]{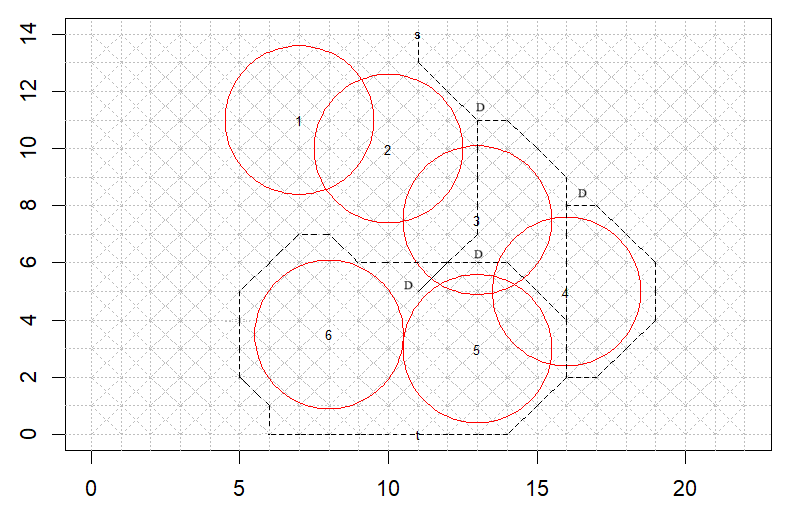}
\caption{Constrained RD policy execution}
\end{subfigure}
\caption{Toy example comparing RCDP with constrained RD policy. 
Panels (a)--(c) show the RCDP setup and execution under $\delta_{\max} = 2$. 
Panel (d) shows the less efficient behavior of the constrained RD policy.}
\label{fig:toy-example}
\end{figure}

\section{Empirical Evaluation via Monte Carlo Simulation}
\label{sec:MC-simulations}

We evaluate the performance of the proposed RCDP policy across a variety of sensor conditions, 
obstacle densities, and disambiguation budgets using a Monte Carlo simulation framework. 
The experimental setup is designed to align with a real-world dataset known as COBRA,
a minefield navigation dataset frequently used in the literature for testing traversal strategies 
(\cite{fishkind2005}, \cite{ye2010graph}, \cite{aksakalli2011}, \cite{aksakalli2016based}). 
We use a rectangular traversal region $\Omega = [0, 100] \times [0, 50]$, with start and target 
points fixed at $s = (50, 50)$ and $t = (50, 1)$, respectively.  
Obstacles are modeled as disks with centers placed in $[10, 90] \times [10, 40]$, 
promoting route selection through cluttered regions rather than entirely avoiding obstacles.




To maintain clarity, we present illustrative outcomes in the main paper and relegate detailed 
replication results, extended settings, and sensitivity analyses to the Appendix.

\subsection{Simulation Design and Experimental Setup}
\label{sec:experiment-setup}
We simulate traversal scenarios within a rectangular region $\Omega = [0,100] \times [0,50]$ 
using a discretized 8-adjacency grid. 
The start and target locations are fixed at $s=(50,50)$ and $t=(50,1)$, respectively.
Obstacle centers are placed within the subregion $[10,90] \times [10,40]$ 
to ensure adequate coverage while allowing traversable paths.

\paragraph{Obstacle Configuration.} 
We consider three density levels by varying the total number of ambiguous obstacles $n \in \{20, 40, 80\}$, with corresponding numbers of true obstacles set as $n_T = \{4, 8, 16\}$ (i.e., 20\% of the obstacles are true). Obstacles are circular with fixed radius 5 and their spatial layout follows a Strauss process with inhibition distance 7 and interaction parameter $\gamma = 0.5$ to model realistic clustering and mild repulsion \citep{baddeley2010}.

\paragraph{Sensor Precision.} 
Sensor-reported blockage probabilities $\pi_x$ follow a two-component beta model:
\[
\pi_x \sim
\begin{cases}
\text{Beta}(6,2), & \text{if } x \text{ is a true (blocking) obstacle}, \\
\text{Beta}(2,6), & \text{if } x \text{ is a false (non-blocking) obstacle}.
\end{cases}
\]
These asymmetric distributions model a sensor providing higher probabilities for true obstacles
and lower probabilities for false obstacles.

\paragraph{Disambiguation Budget.}
Two scenarios are examined:
\begin{itemize}
    \item[(i)] \textit{Simplified scenario:} All obstacles require equal cost $\delta_x = 5$; 
    the constraint is interpreted as a hard cap on the number of disambiguations $N_{\max} \in \{1,2,3\}$.
    \item[(ii)] \textit{General scenario:} Each obstacle has an individual cost $\delta_x \in \{2,3,4,5,6\}$, 
    depending on whether it is a true obstacle and its isolation; 
    the total budget is enforced via total cost constraints $\delta_{\max} \in \{4,6,8,10\}$.
\end{itemize}

\paragraph{Policies and Risk Functions.}
We evaluate seven policies:
\begin{itemize}
    \item[] \textit{Greedy baselines:} constrained RD and constrained DT policies,
    \item[] \textit{Proposed RCDP variants:} employing the reset risk function, 
    the DT risk function, and the linear undesirability function with $\alpha \in \{\delta_x, 15, 30\}$.
\end{itemize}

For each trial, we record
total traversal cost $\mathcal{C}_p$,
number of disambiguations used,
gap from offline-optimal benchmark (i.e., full-information shortest path),
constraint satisfaction rate (proportion of paths with $\delta_p \leq \delta_{\max}$).

\begin{table}[H]
\centering
\caption{Simulation setting parameters.}
\label{tab:sim-params}
\begin{tabular}{cccc}
\toprule
$n$ & $n_T$ & $N_{\max}$ & $\delta_{\max}$ \\
\midrule
20 & 4 & 1, 2, 3 & 4, 6 \\
40 & 8 & 1, 2, 3 & 6, 8 \\
80 & 16 & 1, 2, 3 & 8, 10 \\
\bottomrule
\end{tabular}
\end{table}

\subsection{Performance Comparison}

Figure~\ref{fig:performance-comparison} shows the average traversal cost and standard deviation of each risk-based policy 
across environments: different obstacle density levels, neutral sensor marks, and varying budget levels.

\begin{figure}[H]
	\centering
\includegraphics[width=0.45\textwidth]{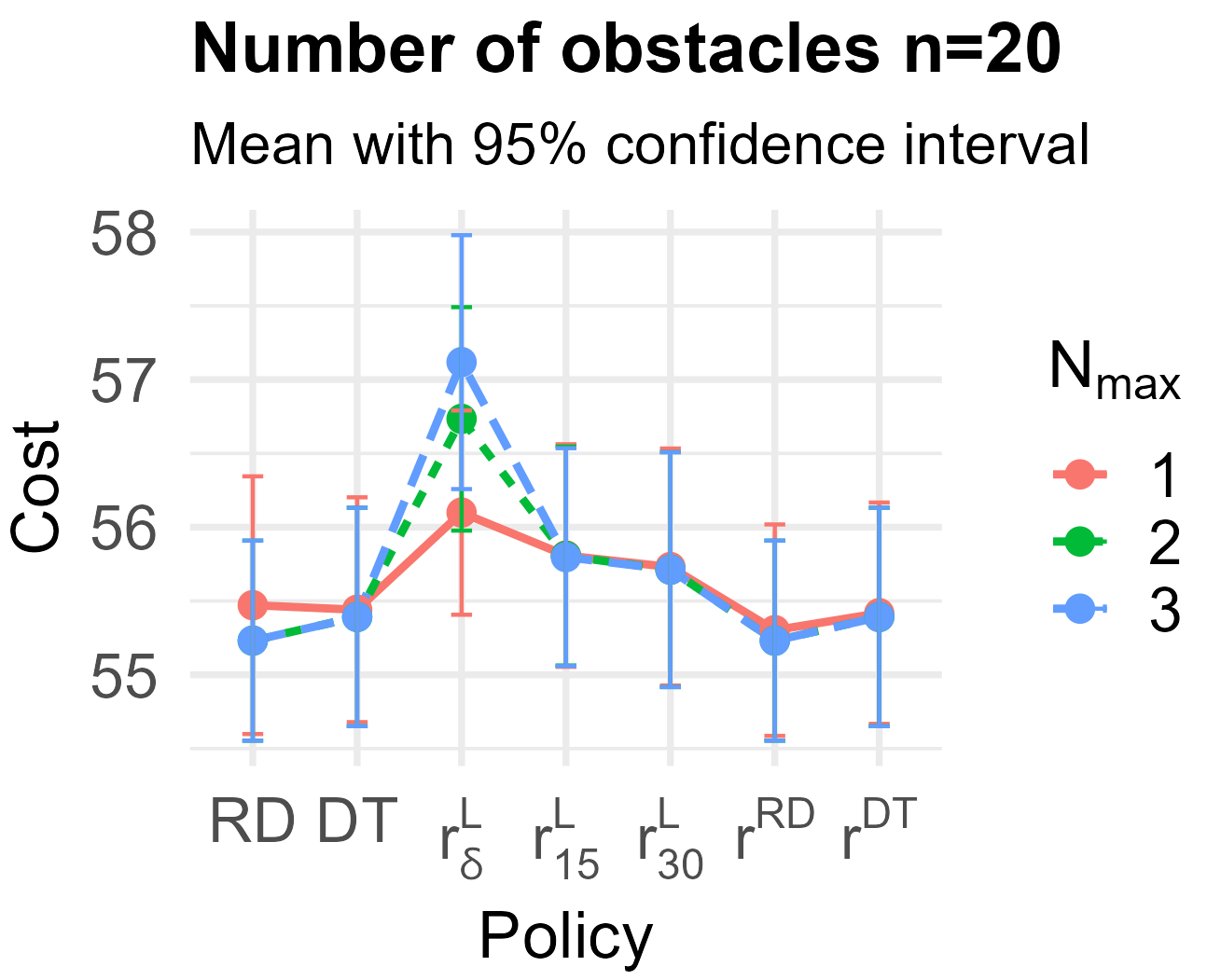} 
\includegraphics[width=0.45\textwidth]{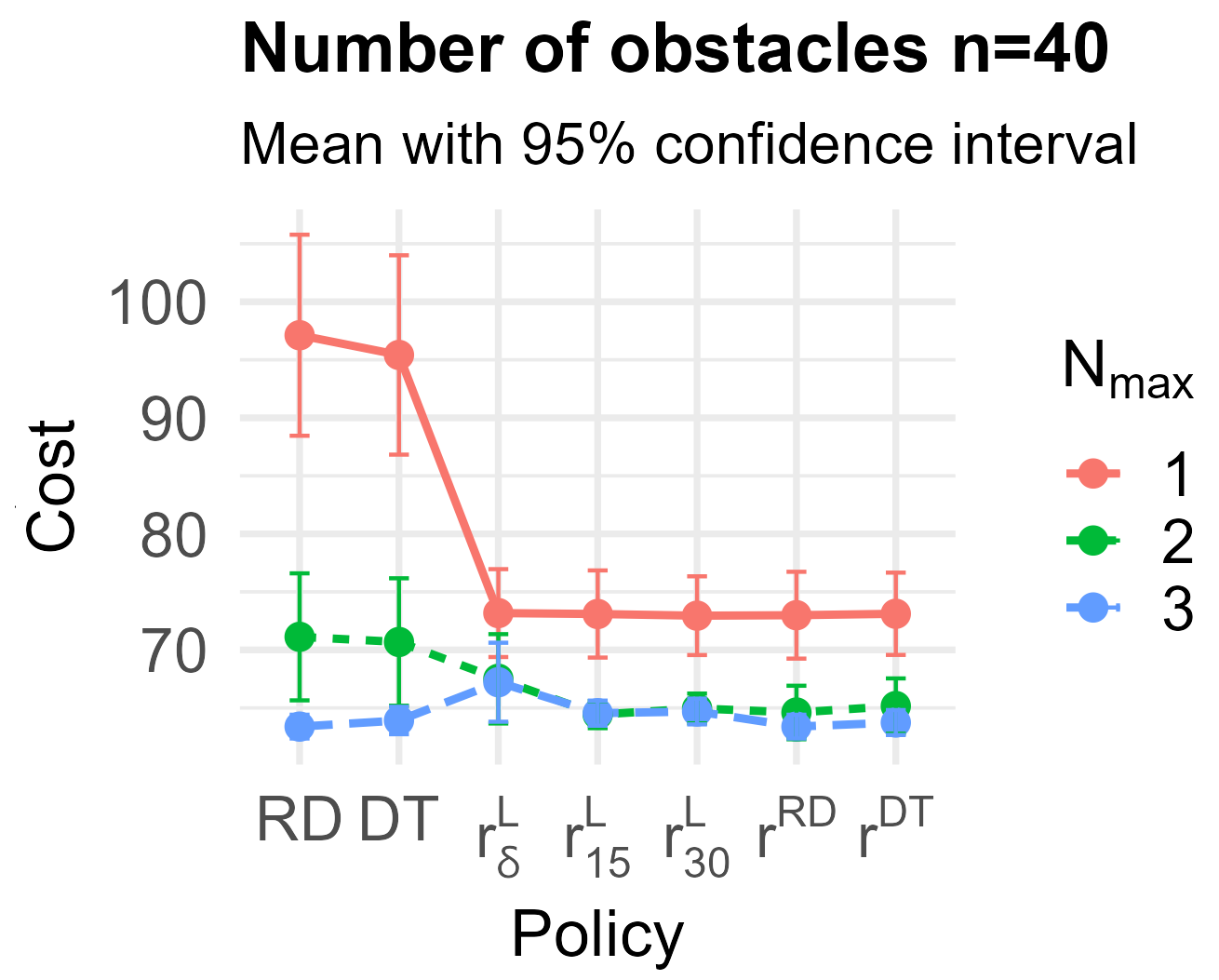} 
\includegraphics[width=0.45\textwidth]{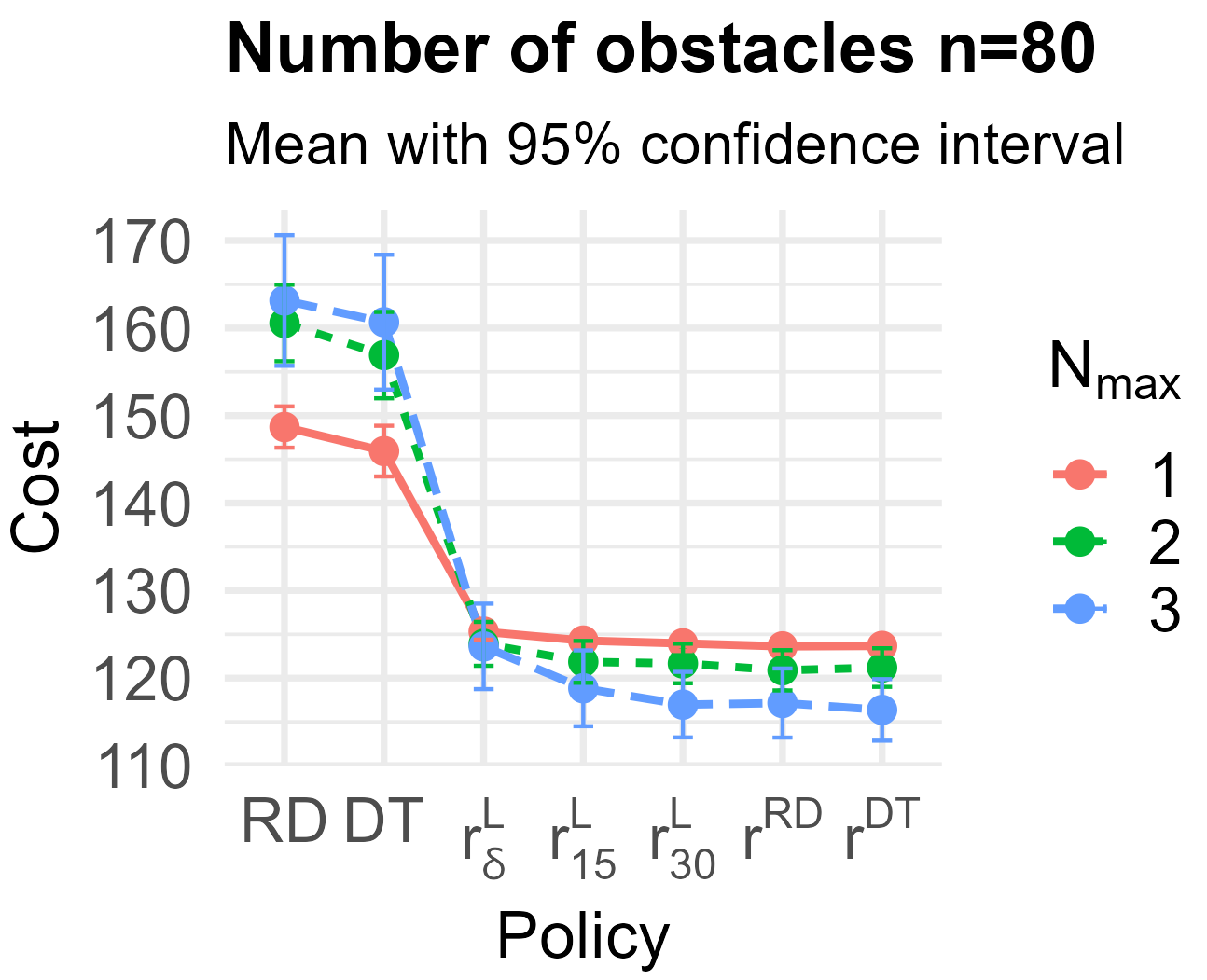} 
\includegraphics[width=0.45\textwidth]{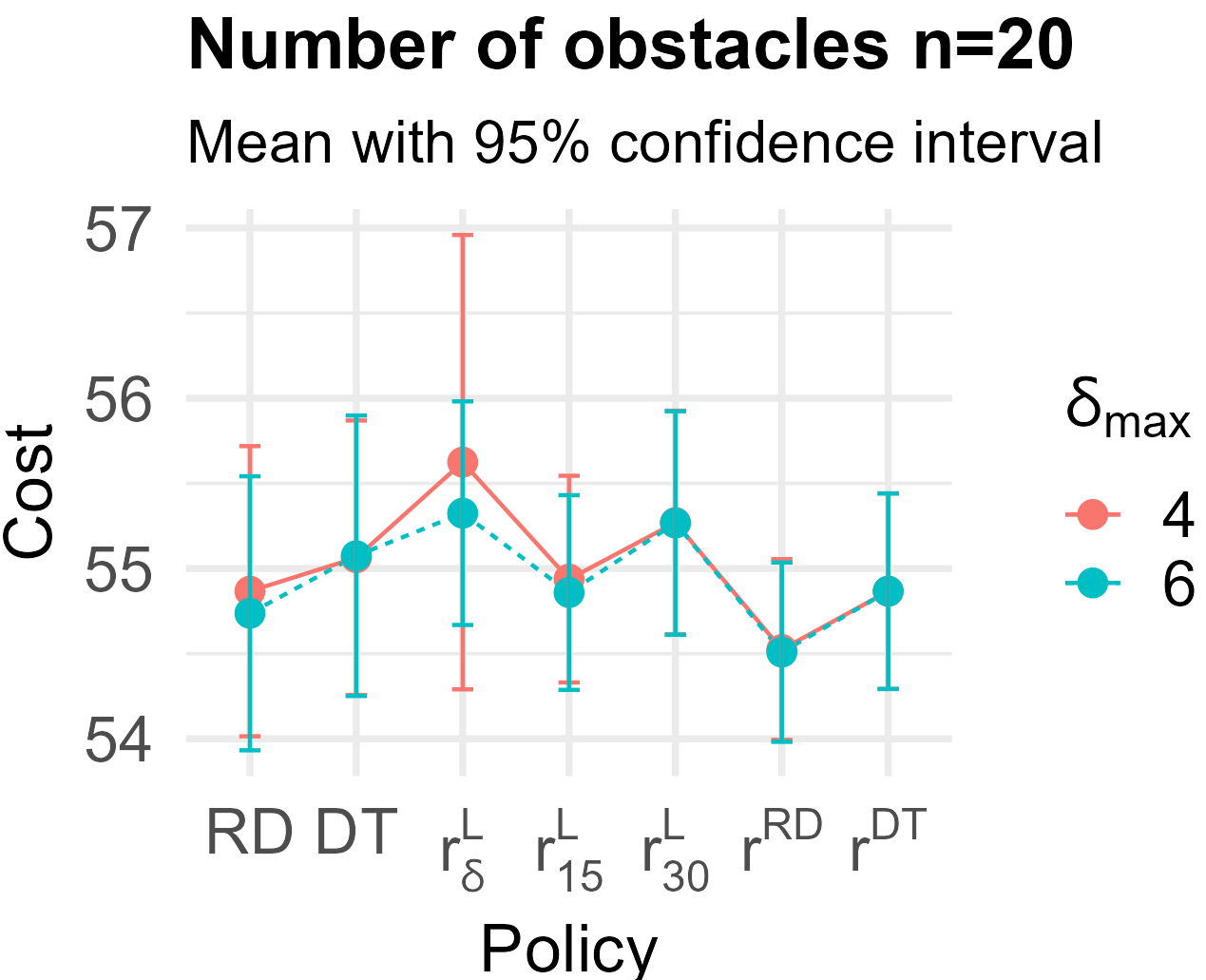} 
\includegraphics[width=0.45\textwidth]{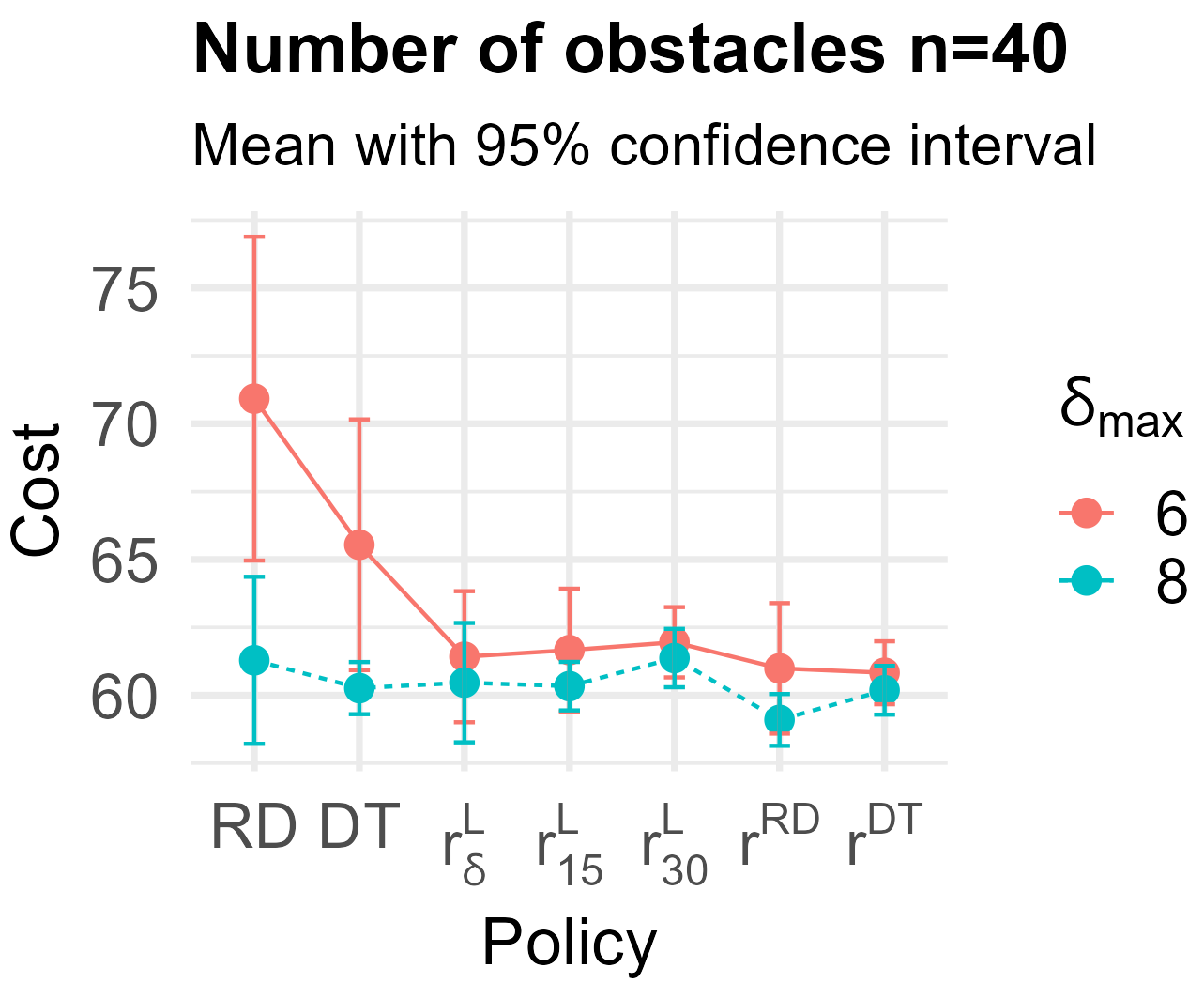} 
\includegraphics[width=0.45\textwidth]{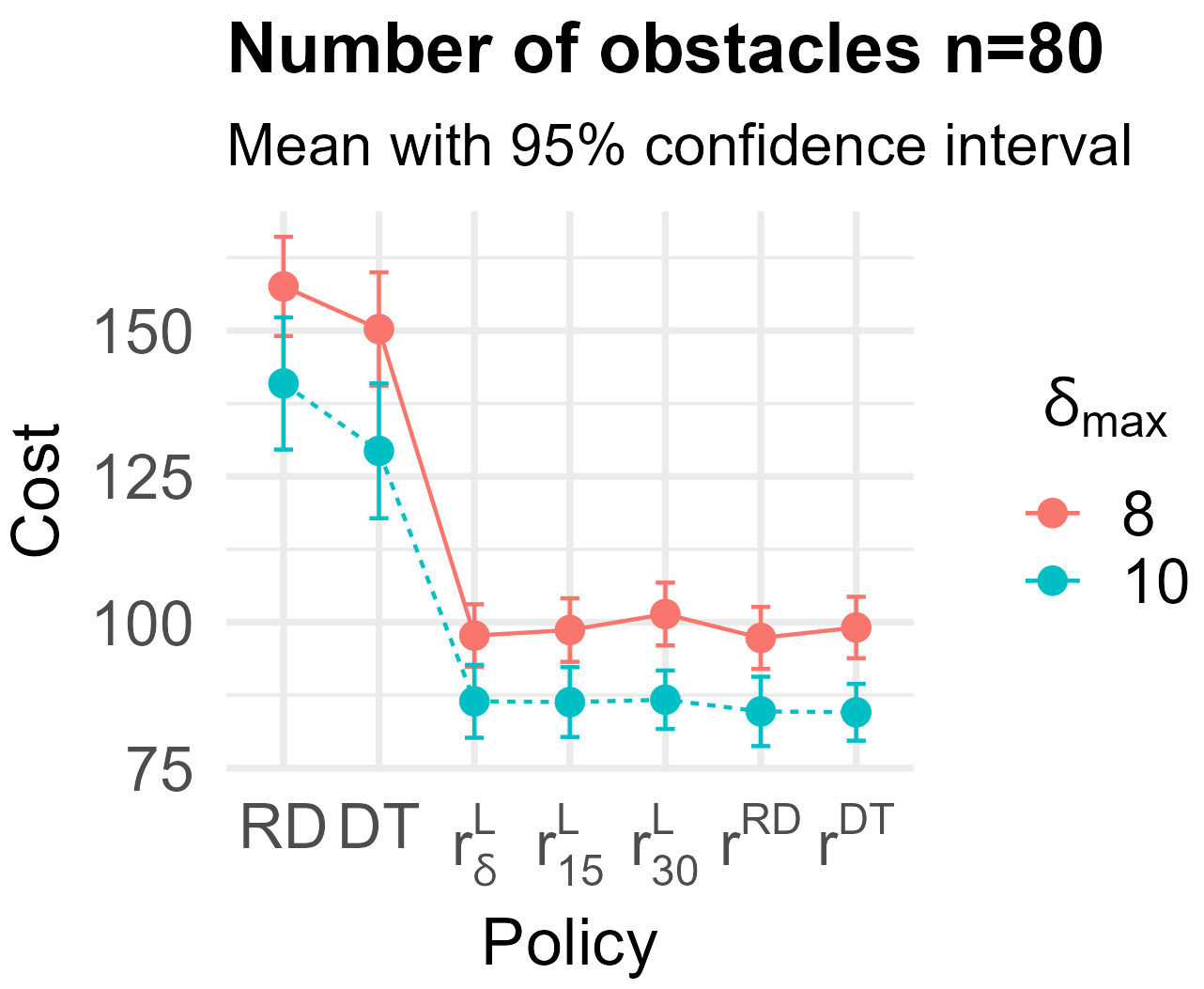}
	\caption{The average traversal cost with $95\%$ confidence interval of two greedy policies, five RCDP policies with different risk functions}
    \label{fig:performance-comparison}
\end{figure}

\begin{figure}[H]
\centering
\includegraphics[width=0.45\textwidth]{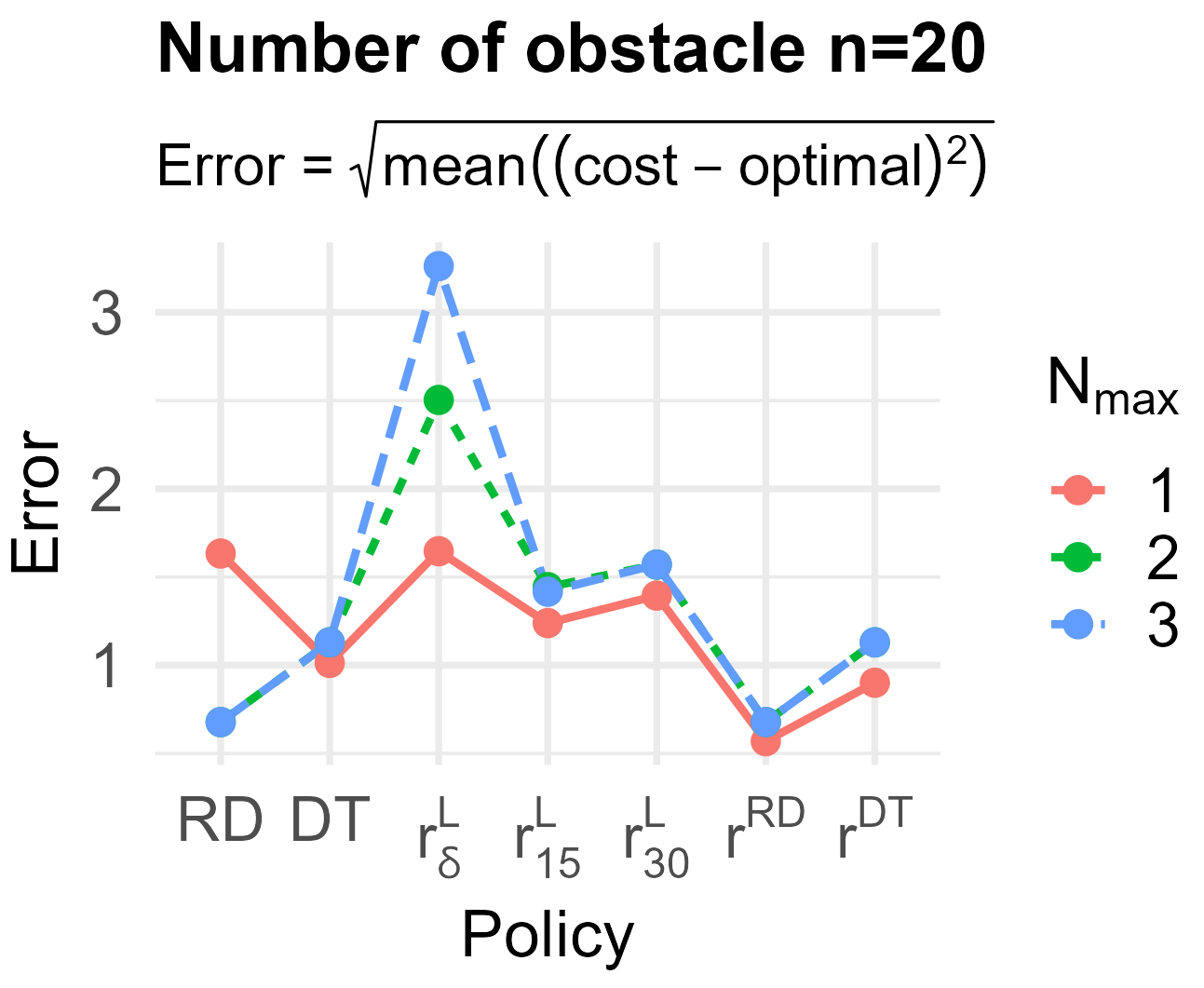} 
\includegraphics[width=0.45\textwidth]{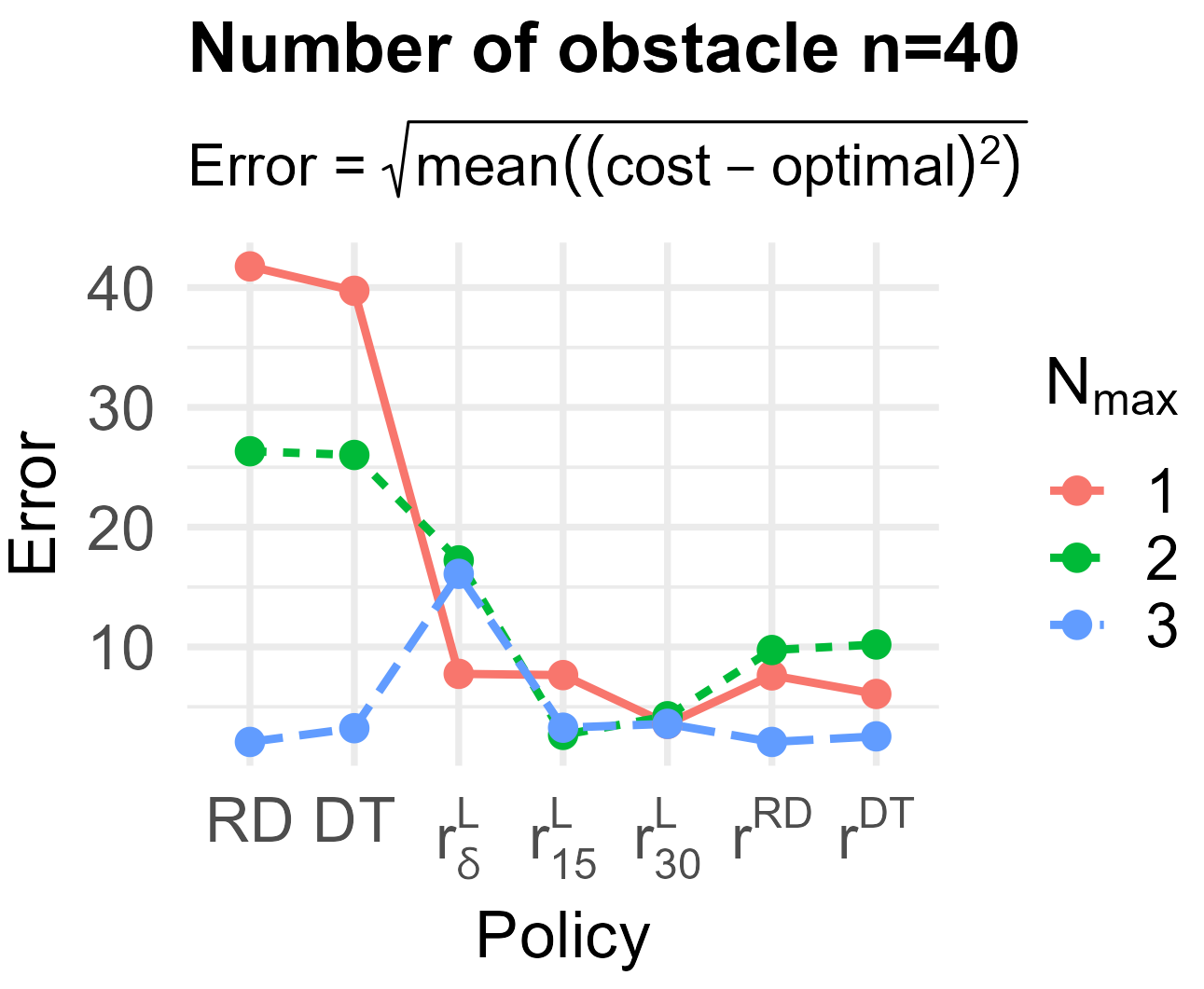} 
\includegraphics[width=0.45\textwidth]{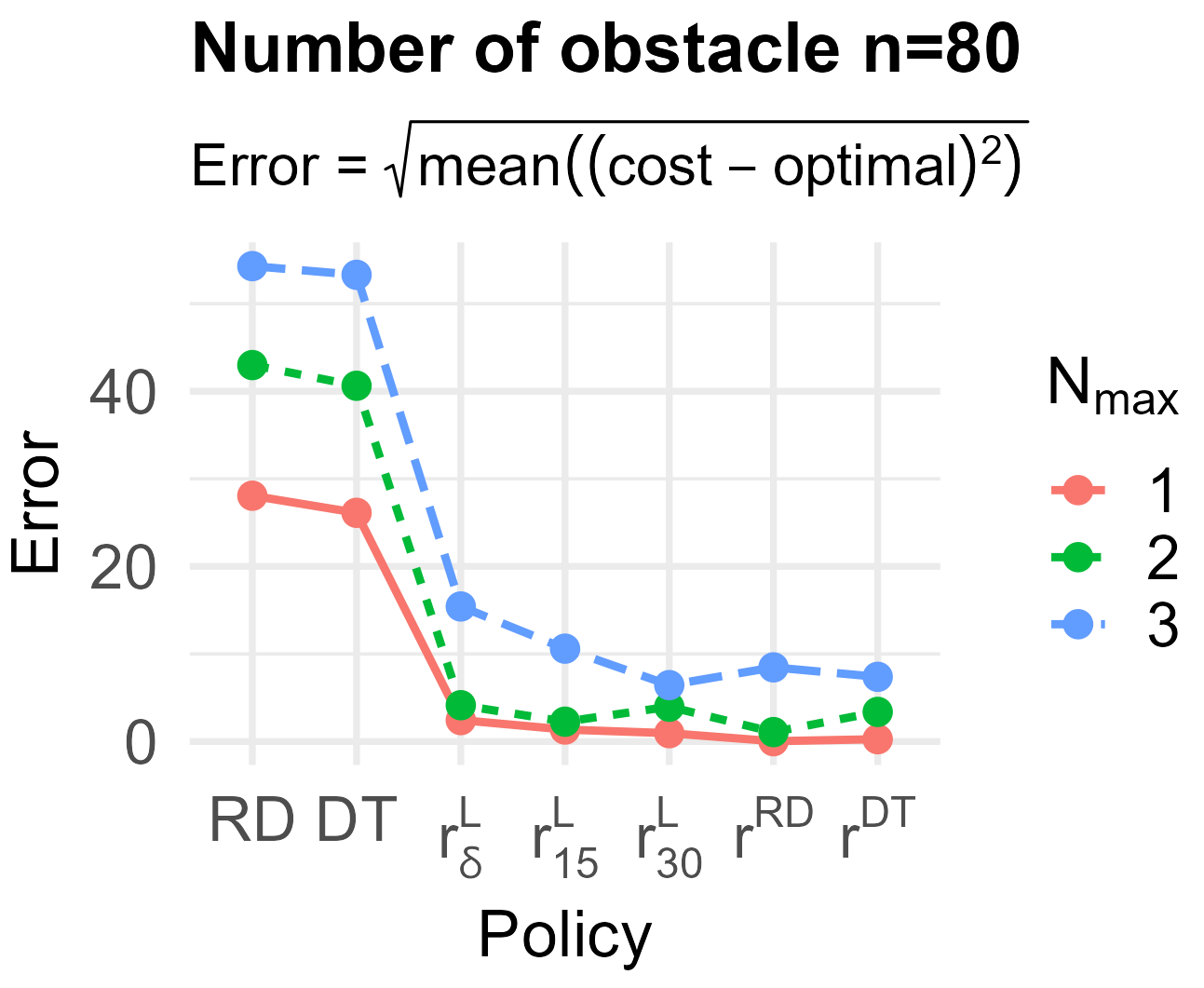} 
\includegraphics[width=0.45\textwidth]{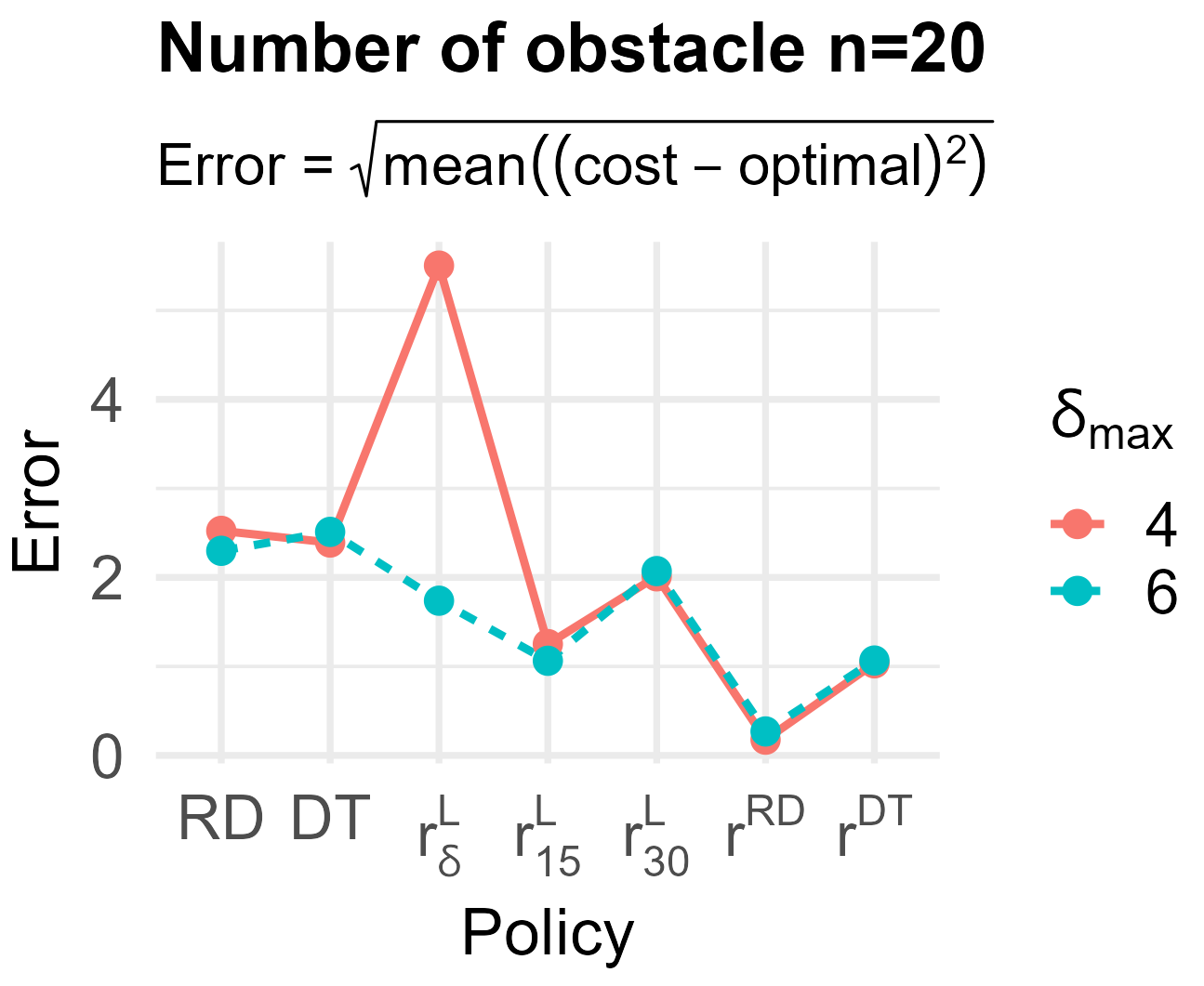} 
\includegraphics[width=0.45\textwidth]{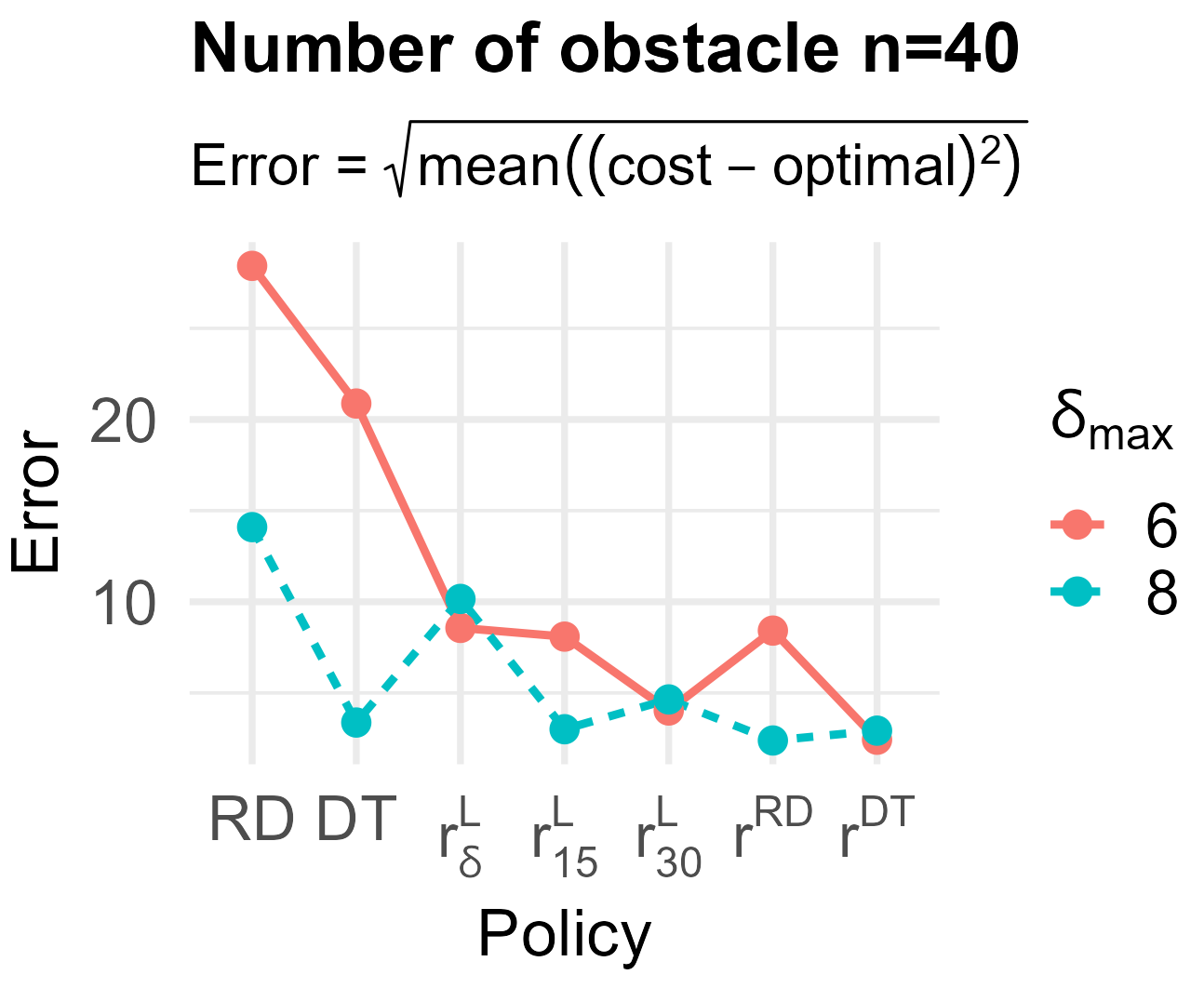} 
\includegraphics[width=0.45\textwidth]{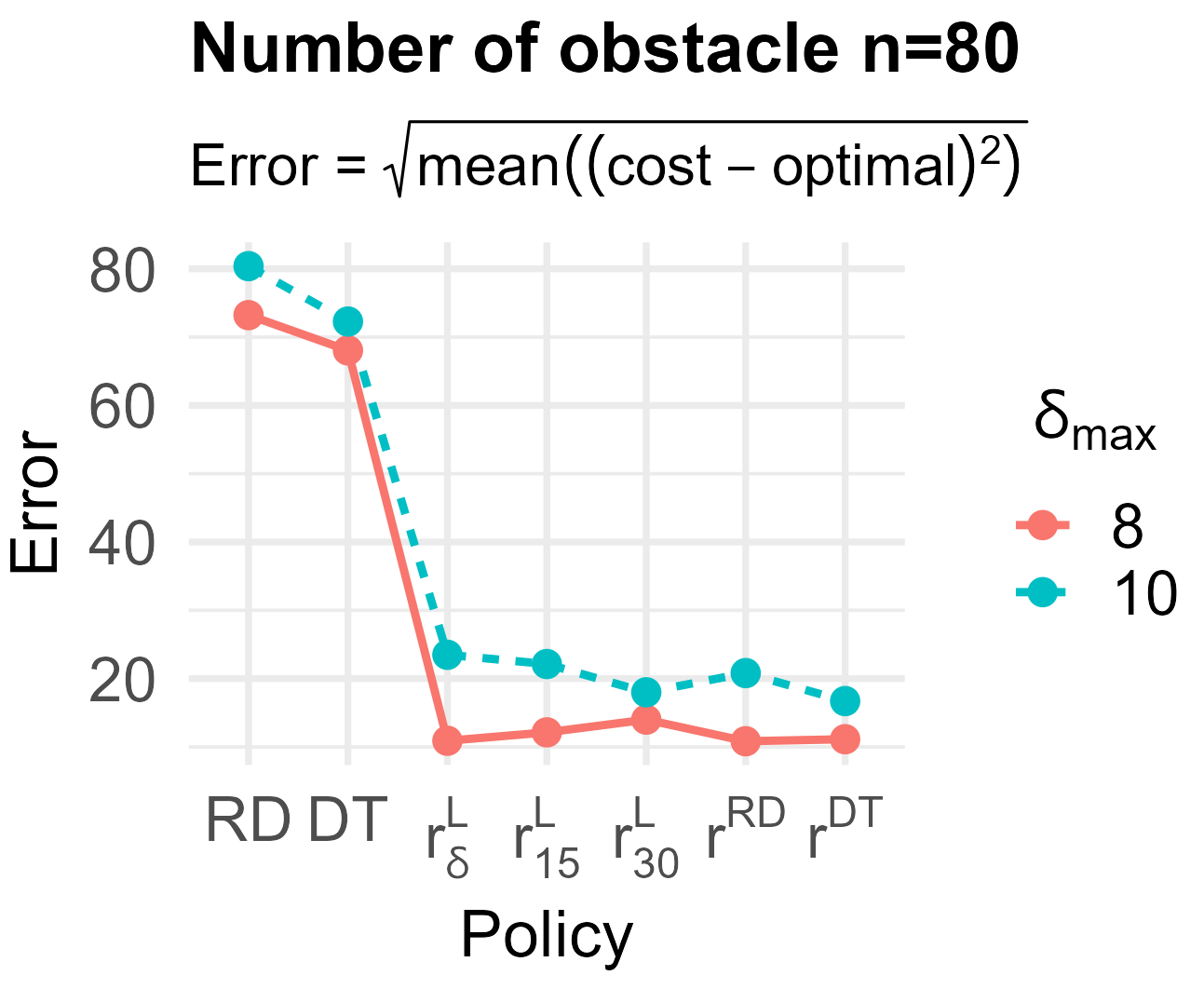}
\caption{The deviation of two greedy policies, five RCDP policies with different risk functions 
from optimal solutions}
\label{fig:error}
\end{figure}

Beyond average cost, we assess each policy’s proximity to the offline-optimal benchmark.
Figure~\ref{fig:error} illustrates the deviation from optimal across environments, confirming that RCDP policies maintain closer adherence to ideal paths, especially under dense obstacles and tight budgets.

\begin{table}[H]
\centering
	\resizebox{7cm}{!}{%
	 \begin{tabular}{@{}c|cc|ccc@{}}
			\toprule
			& \multicolumn{2}{c|}{Greedy Policies} & \multicolumn{3}{c}{WCSPP Policies} \\
			& RD   & DT   & $r^{RD}$ & $r^{DT}$ & \textbf{$r^{L}_{\alpha^*}$} \\ \midrule
			\multirow{2}{*}{$N_{\max}=1$} & 149.50  & 149,50  & 125.53   & 125.53   & \textbf{122.91} \\
			& 151.25  & 151.25  & 117.25   & 121.43   & \textbf{117.25} \\ \midrule
			\multirow{2}{*}{$N_{\max}=2$} & 183.23  & 183.23  & 140.83   & 124.36   & \textbf{121.33} \\
			& 138.98  & 138.10  & 98.46 & 98.46 & \textbf{97.10}  \\ \midrule
			\multirow{2}{*}{$N_{\max}=3$} & 184.91  & 184.91  & 165.23   & 97.05 & \textbf{97.05}  \\
			& 161.25  & 147.15  & 129.60   & 129.60   & \textbf{126.77} \\ \midrule
			\multirow{2}{*}{$\delta_{\max}=4$} & 67.13  & 67.13  & 53.14   & 53.14 & \textbf{53.00}  \\
			& 85.77  & 85.77  & 64.11   & 66.11   & \textbf{64.11} \\ \midrule
			\multirow{2}{*}{$\delta_{\max}=6$} & 111.18  & 70.11  & 70.11   & 70.11 & \textbf{69.67}  \\
			& 154.50  & 173.12  & 71.91   & 71.91   & \textbf{69.43} \\ \midrule
			\multirow{2}{*}{$\delta_{\max}=8$} & 188.40  & 193.37  & 90.59   & 90.59 & \textbf{89.67}  \\
			& 185.18  & 189.28  & 122.25   & 124.94   & \textbf{121.91} \\ \midrule
			\multirow{2}{*}{$\delta_{\max}=10$} & 171.23  & 194.10  & 80.60   & 89.43 & \textbf{80.60}  \\
			& 184.84  & 184.84  & 68.11   & 83.08  & \textbf{68.11} \\ \bottomrule
		\end{tabular}%
	}
	\caption{Traversal costs using different policies (for illustration, 
we only show 2 cases for each setting from 100 replications)}
\label{tab: comparison different functions}
\end{table}

\begin{table}[H]
\centering
 \begin{subtable}{\textwidth}
	\centering
		\resizebox{10cm}{!}{%
		 \begin{tabular}{@{}cc|cc|cccccc@{}}
				\toprule
				\multicolumn{2}{c|}{Setting} & \multicolumn{2}{c|}{Greedy Policies} & \multicolumn{6}{c}{RCDP Policies}     \\ \midrule
				$N_{\max}$ & $n$ & RD & DT   & $r^L_\delta$ & $r^L_{15}$ & $r^L_{30}$ & $r^{RD}$ & \multicolumn{1}{c|}{$r^{DT}$} & Benchmark (Optimal) \\ \midrule
				\multirow{3}{*}{1} & 20   & 0.26  & 0.20 & 0.75 & 0.54 & 0.35 & 0.24  & \multicolumn{1}{c|}{0.19}  & 0.32 \\
				& 40   & 1.31  & 1.26 & 0.93 & 0.86 & 0.78 & 0.77  & \multicolumn{1}{c|}{0.72}  & 0.82 \\
				& 80   & 2.00  & 1.87 & 0.73 & 0.38 & 0.18 & 0.07  & \multicolumn{1}{c|}{0.03}  & 0.09 \\ \midrule
				\multirow{3}{*}{2} & 20   & 0.26  & 0.20 & 1.09 & 0.62 & 0.38 & 0.26  & \multicolumn{1}{c|}{0.20}  & 0.34 \\
				& 40   & 1.42  & 1.35 & 1.70 & 1.57 & 1.37 & 1.30  & \multicolumn{1}{c|}{1.26}  & 1.33 \\
				& 80   & 2.96  & 2.75 & 1.60 & 0.78 & 0.37 & 0.24  & \multicolumn{1}{c|}{0.15}  & 0.32 \\ \midrule
				\multirow{3}{*}{3} & 20   & 0.26  & 0.20 & 1.21 & 0.63 & 0.28 & 0.26  & \multicolumn{1}{c|}{0.20}  & 0.34 \\
				& 40   & 1.42  & 1.36 & 2.25 & 1.89 & 1.53 & 1.42  & \multicolumn{1}{c|}{1.34}  & 1.43 \\
				& 80   & 3.82  & 3.54 & 2.52 & 1.47 & 0.80 & 0.74  & \multicolumn{1}{c|}{0.54}  & 0.70 \\ \bottomrule
			\end{tabular}%
		}
		\caption{Results for $N_{\max}$ values of 1, 2 and 3}
	\end{subtable}
	
    \vspace{0.5cm} 
    
    \begin{subtable}{\textwidth}
    \centering
    	\resizebox{10cm}{!}{%
    	 \begin{tabular}{@{}cc|cc|cccccc@{}}
    			\toprule
    			\multicolumn{2}{c|}{Setting} & \multicolumn{2}{c|}{Greedy Policies} & \multicolumn{6}{c}{RCDP Policies}     \\ \midrule
    			$\delta_{\max}$ & $n$ & RD & DT   & $r^L_\delta$ & $r^L_{15}$ & $r^L_{30}$ & $r^{RD}$ & \multicolumn{1}{c|}{$r^{DT}$} & Benchmark (Optimal) \\ \midrule
    			\multirow{1}{*}{4} & 20   & 1.14  & 1.16 & 2.92 & 1.66 & 1.04 & 1.10  & \multicolumn{1}{c|}{0.84}  & 1.54 \\ \midrule
    			\multirow{2}{*}{6} & 20   & 1.16  & 0.84 & 3.38 & 1.80 & 1.12 & 1.14  & \multicolumn{1}{c|}{0.85}  & 1.66 \\
    			& 40   & 4.38  & 3.74 & 5.04 & 3.78 & 3.58 & 3.72  & \multicolumn{1}{c|}{3.30}  & 4.04 \\ \midrule
    			\multirow{2}{*}{8} & 40   & 4.60  & 3.62 & 6.46 & 4.54 & 3.78 & 4.50  & \multicolumn{1}{c|}{3.60}  & 5.04 \\
    			& 80   & 8.82  & 8.72 & 7.8 & 5.46 & 4.26 & 5.78  & \multicolumn{1}{c|}{4.42}  & 6.50 \\ \midrule
    			\multirow{1}{*}{10} & 80   & 9.56  & 9.38 & 9.72 & 8.72 & 7.40 & 8.92  & \multicolumn{1}{c|}{8.78}  & 9.00 \\ \bottomrule
    		\end{tabular}%
    	}
    	\caption{Results for $\delta_{\max}$ values of 4, 6, 8 and 10}
    \end{subtable}

	\caption{The average number of intersected obstacles $\bar{N_p}$ and 
the average resource needed $\bar{\delta_p}$ on the traversal path $p$ generated two greedy policies, 
five RCDP policies with different risk functions, and the benchmark policy}
\label{tab: comparison_sevenalgs2}
\end{table}
RCDP policies also intersect fewer true obstacles and require lower disambiguation costs on average. 
Greedy policies often violate budget constraints, particularly in dense environments, 
as evident from consistently higher average $\bar{N}_p$ and $\bar{\delta}_p$ values.

\subsection{Summary of Key Empirical Findings}
\label{sec:summary-emp-find}

\begin{itemize}
    \item[] \textbf{Linear Undesirability (LU) Policies:} RCDP variants with LU risk functions—both fixed ($\alpha = 15, 30$) and adaptive ($\alpha = \delta_x$)—consistently yield lower mean traversal costs and reduced variability compared to greedy baselines. No single $\alpha$ value dominates across all settings, but $\alpha = 15$ often achieves a favorable trade-off between risk aversion and exploration.
        Further simulation experiments (see Appendix Section \ref{sup-sec:Linear-Experiments}) validate that a Bayesian version of the undesirability function, denoted $r^{LB}$, provides improved performance and asymptotic convergence to the benchmark as sensor precision improves.
    
    \item[] \textbf{Reset Disambiguation (RD) and Distance-to-Termination (DT) Risk:} 
    RCDP policies using these risk functions exhibit competitive performance, 
    with the reset variant ($r^{RD}$) often outperforming others in sparse and moderately dense environments. 
    The DT variant performs better under generous budgets but degrades under tight constraints.

    \item[] \textbf{Greedy Baselines:} 
    Both constrained RD and DT policies frequently violate the disambiguation budget, especially in dense settings. 
    Their lack of integrated budget reasoning leads to higher traversal cost, greater variability, 
    and in some cases, longer-than-optimal detours.

    \item[] \textbf{Budget Impact:} 
    Increasing the disambiguation budget generally reduces traversal cost and error, 
    while performance gaps between policies widen under tighter constraints. 
    RCDP policies maintain constraint compliance and outperform greedy counterparts 
    even under severe budget limitations.

    \item[] \textbf{Robustness and Percentile Behavior:} 
    RCDP policies demonstrate better performance not only on average 
    but also in worst-case (75th percentile) and favorable (25th percentile) scenarios. 
    This robustness is particularly pronounced in obstacle-dense environments ($n=80$).
\end{itemize}

In addition to average cost, RCDP policies also demonstrate superior robustness under uncertainty. 
Their traversal cost distributions exhibit lower standard deviation and narrower confidence intervals than greedy baselines, 
as seen in Appendix Figures \ref{sup-fig:avg with CI}-\ref{sup-fig:error}. 
Notably, their 75th percentile cost remains bounded even under high obstacle density and tight budgets, 
suggesting consistent performance in worst-case scenarios.

\subsection{Evaluation of Policy Robustness}

We assess policy robustness via \emph{relative efficiency}, defined as:
\[
\text{Efficiency}(r) = \frac{\mathcal{C}_{\text{offline}}}{\mathcal{C}_r},
\]
where $\mathcal{C}_{\text{offline}}$ is the expected cost under full obstacle knowledge, 
and $\mathcal{C}_r$ is the average cost under risk model $r$.

Relative efficiency offers a normalized performance comparison against an idealized offline solution. 
We compute this metric across all simulation regimes and for each risk-based policy variant.

A detailed box plot visualization of relative efficiency is provided in the Appendix (Figure~\ref{fig:efficiency-boxplot}). 
These results highlight that risk-aware path planning with adaptive disambiguation policies 
yields consistent improvements over naive baselines. 
While RD and DT serve specialized roles, LU-based policies—with either fixed or 
cost-weighted scaling—deliver the most balanced performance under realistic sensing 
and resource constraints.

Full simulation configurations, additional tables (e.g., mean and variance of $\mathcal{C}_p$ 
across 15 regimes), runtime statistics, and complete visualizations are available in the Appendix.

\section{Discussion and Conclusions}
\label{sec:disc-conc}

This paper addresses a resource-constrained generalization of the Random Disambiguation 
Path (RDP) problem, extending the classical Stochastic Obstacle Scene (SOS) framework 
to reflect realistic operational constraints. 
In the RDP setting, a navigating agent (NAVA) seeks an optimal route through a spatial 
domain populated with disk-shaped obstacles of uncertain blockage. 
The agent may actively query (disambiguate) obstacle status, but incurs a cost for 
doing so. 
Unlike earlier work that neglects or simplifies resource limitations, we introduce an 
explicit disambiguation budget, formulating the resulting \textit{RDP with Constrained 
Disambiguation} (RCDP) problem as a \emph{Weight Constrained Shortest Path Problem} 
(WCSPP).

Our contribution is twofold: a theoretical reformulation of RCDP into a constrained 
optimization framework, and a practical algorithmic solution that integrates cost 
approximation with constraint-aware search. 
We develop surrogate cost models based on additive risk functions, combining Euclidean 
path length with sensor-informed obstacle risk. 
These approximations allow the constrained problem to be approached via a Lagrangian 
relaxation method, transforming the hard constraint into a penalized objective. 
To accelerate computation while preserving optimality, we introduce a two-phase vertex 
elimination (TPVE) strategy, which prunes non-promising vertices through feasibility 
and dominance tests.

Empirical validation across diverse simulation regimes demonstrates that the proposed 
RCDP policies consistently outperform greedy baselines, yielding lower expected 
traversal costs and reduced variability. 
The comparative analysis spans multiple risk functions—Reset Disambiguation (RD), 
Distance-to-Termination (DT), and Linear Undesirability (LU)—and evaluates performance 
using average cost, error rates, variability measures, and quantile-based statistics. 
The LU-based policies, particularly those with moderate or cost-sensitive scaling, 
exhibit strong robustness across a wide range of obstacle densities and budget levels.

The proposed TPVE-enhanced Lagrangian framework not only improves runtime efficiency 
but often eliminates the duality gap altogether, delivering provably optimal solutions 
in practice. 
We further support these findings with theoretical guarantees, including complexity 
bounds, feasibility preservation, and dominance over baseline policies.

\paragraph{Limitations and Future Work.}
Despite the effectiveness of our approach, several avenues remain open for further 
exploration:
(i) \textbf{Multi-constraint extension:} Incorporating multiple disambiguation 
    budgets (e.g., time, energy, or risk exposure) transforms the problem into a 
    multi-dimensional WCSPP, necessitating new algorithmic designs.
(ii) \textbf{Correlated obstacle statuses:} Our current model assumes independence 
    across obstacles. 
    Future work could explore spatial dependency structures (e.g., Markov random fields) 
    and adapt pathfinding policies accordingly.
(iii) \textbf{General obstacle geometries:} We currently model obstacles as uniform 
    disks with binary status. 
    More nuanced models could include irregular shapes or continuous risk fields, 
    capturing partial blockage.
(iv) \textbf{Dynamic and adversarial environments:} Introducing moving or 
    adversarial obstacles would align the model with real-time navigation scenarios 
    and call for adaptive or online decision-making under uncertainty.
(v) \textbf{Tuning the $\alpha$ parameter for linear undesirability (LU) risk function:}
While this study employs a fixed linear undesirability (LU) risk function with manually selected $\alpha$, 
we have explored data-driven strategies for tuning $\alpha$ and 
adapting it through a Bayesian adjustment based on prior information about obstacle types and sensor accuracy. 
Although these extensions were not incorporated into the core simulation experiments, 
they offer a principled mechanism to enhance adaptability and robustness across diverse environments.
The methodological details and empirical illustrations are provided
in the Appendix (Section \ref{app: sec_alpha}). 
Incorporating these adaptive risk formulations into future simulation pipelines remains a promising direction.

Overall, this work bridges constrained optimization, probabilistic reasoning, and 
graph-based navigation in uncertain environments. 
The RCDP formulation and solution framework offer a principled foundation for future 
developments in autonomous planning under partial observability and limited resources.

\section*{Acknowledgements}
Most of the Monte Carlo simulations presented in this article
were executed at Easley HPC Laboratory of Auburn University.
LZ and EC were supported by Office of Naval Research Grant N00014-22-1-2572
and EC were supported by NSF Award \# 2319157.


\appendix
\renewcommand{\thesection}{A\arabic{section}}
\renewcommand{\thesubsection}{A\arabic{section}.\arabic{subsection}}
\renewcommand{\thesubsubsection}{A\arabic{section}.\arabic{subsection}.\arabic{subsubsection}}

\renewcommand{\thefigure}{A\arabic{figure}}
\renewcommand{\thetable}{A\arabic{table}}

\section{Pseudocode for COGR and LOGR Algorithms}
\label{app:pseudo-code}

%

\begin{algorithm}[H]
	\caption{\textbf{C}ost- and \textbf{O}bstacle-based \textbf{G}raph \textbf{R}eduction (COGR Algorithm)}
	\label{Lac_ini_reduc}
	\begin{algorithmic}[1]
		\footnotesize
		\Statex \textnormal{\textbf{Input:} A graph with adjusted costs and weights $G_{adj}$ from GI Algorithm, the start vertex $s$, the target vertex $t$, the disambiguation budget $\delta_{\max}$}
		\Statex \textnormal{\textbf{Output:} Reduced graph $G_{\text{red}}$, optimal path $p^*$ (or path with the lowest upper bound $p_{U}$ and path with the highest lower bound $p_{L}$)}
		\State \textnormal{Initialize $V_{\text{del}} \gets \emptyset$ and $E_{\text{del}} \gets \emptyset$}
		\State \textnormal{Find the shortest obstacle-free path $p^\infty$ using Dijkstra's algorithm with $\lambda = \infty$}
		\State \textnormal{Find the minimum cost path $p^0$ using Dijkstra's algorithm with $\lambda = 0$}
		\If{\textnormal{$\delta^0 \leq \delta_{\max}$}}
		\State \textbf{Output} \textnormal{$p^* = p^0$, \textbf{STOP}}
		\EndIf
		\State \textnormal{$\widetilde{\mathcal{C}}_{U} \gets \widetilde{\mathcal{C}}^\infty$; $p_{U} \gets p^\infty$}
		\Repeat
		\For{\textnormal{$v \in V(G_{adj})$}}
		\State \textnormal{Find paths $p_{v}^\infty$ and $p_{v}^0$ through $v$ with updated costs and weights}
		\If{\textnormal{$\delta_{v}^\infty > \delta_{\max}$}}
		\State \textnormal{$V_{\text{del}} \gets V_{\text{del}} \cup \{v\}$}
		\State \textnormal{$E_{\text{del}} \gets E_{\text{del}} \cup \{e : e \text{ is incident to } v\}$}
		\State \textnormal{$G_{adj} \gets (V \setminus V_{\text{del}}, E \setminus E_{\text{del}})$}
		\ElsIf{\textnormal{$\widetilde{\mathcal{C}}_{v}^\infty < \widetilde{\mathcal{C}}_{U}$}}
		\State \textnormal{$\widetilde{\mathcal{C}}_{U} \gets \widetilde{\mathcal{C}}_{v}^\infty$; $p_{U} \gets p_{v}^\infty$}
		\State \textnormal{Find $p_{v}^0$ through $v$ with updated costs and weights}
		\If{\textnormal{$\widetilde{\mathcal{C}}_{v}^0 < \widetilde{\mathcal{C}}_{U}$ and $\delta_{v}^0 \leq \delta_{\max}$}}
		\State \textnormal{$\widetilde{\mathcal{C}}_{U} \gets\widetilde{\mathcal{C}}_{st,v}^0$; $p_{U} \gets p_{v}^0$}
		\ElsIf{\textnormal{$\widetilde{\mathcal{C}}_{v}^0 > \widetilde{\mathcal{C}}_{U}$}}
		\State \textnormal{$V_{\text{del}} \gets V_{\text{del}} \cup \{v\}$}
		\State \textnormal{$E_{\text{del}} \gets E_{\text{del}} \cup \{e : e \text{ is incident to } v\}$}
		\State \textnormal{$G_{adj} \gets (V \setminus V_{\text{del}}, E \setminus E_{\text{del}})$}
		\EndIf
		\EndIf
		\EndFor
		\Until{\textnormal{$V_{\text{del}} = \emptyset$}}
		\State \textnormal{$p^0 \gets \min_{v} p_{v}^0$; $p^\infty \gets \min_{v} p_{v}^\infty$}
		\If{\textnormal{$\delta^0 \leq \delta_{\max}$}}
		\State \textbf{Output} \textnormal{$\{p^* = p^0\}$}
		\Else
		\State \textbf{Output} \textnormal{$\{p_{U} = p^\infty, p_{L} = p^0\}$}
		\EndIf
	\end{algorithmic}
\end{algorithm}

\begin{algorithm}[H]
	\caption{\textbf{L}agrangian \textbf{O}ptimization and \textbf{G}raph \textbf{R}eduction (LOGR Algorithm)}
	\label{modified_SNE}
	\begin{algorithmic}[1]
		\footnotesize
		\Statex \textnormal{\textbf{Input:} The reduced-size graph $G_{\text{red}}$, path with the smallest upper bound $p_{U} = p^\infty$, path with the greatest lower bound $p_{L} = p^0$, the disambiguation budget $\delta_{\max}$}
		\Statex \textnormal{\textbf{Output:} Optimal path (or approximately optimal path) $p^*$}
		\State \textnormal{Initialize $i\gets0$}
		\State \textnormal{Set initial bounds: $\widetilde{\mathcal{C}}_i^+ \gets \widetilde{\mathcal{C}}_L$; $\delta_{i}^+ \gets \delta_L$; $\widetilde{\mathcal{C}}_{i}^- \gets \widetilde{\mathcal{C}}_U$; $\delta_{i}^- \gets \delta_U$}
		\While{\textnormal{$p^*$ is not found}}
		\State \textnormal{$\lambda_{i+1} \gets \frac{\widetilde{\mathcal{C}}_{i}^- - \widetilde{\mathcal{C}}_{i}^+}{\delta_{i}^+ - \delta_{i}^-}$}
		\State \textnormal{Find $p_{i+1}$ in $G_{\text{red}}$ with cost $\widetilde{\mathcal{C}}_{i+1} + \lambda_{i+1} \delta_{i+1}$}
		\If{\textnormal{$\delta_{i+1} = \delta_{\max}$}}
		\State \textbf{Output} \textnormal{$p^* = p_{i+1}$, \textbf{break}}
		\EndIf
		\For{\textnormal{$v \in V(G_{\text{red}})$}}
		\State \textnormal{Find $p_{i+1,v}$ via $v$}
		\If{\textnormal{$\widetilde{\mathcal{C}}_{i+1,v} < \widetilde{\mathcal{C}}_{U}$ and $\delta_{i+1,v} \leq \delta_{\max}$}}
		\State \textnormal{$\widetilde{\mathcal{C}}_{U} \gets \widetilde{\mathcal{C}}_{i+1,v}$; $p_{U} \gets p_{i+1,v}$}
		\EndIf
		\If{\textnormal{$\Phi(\lambda_{i+1},p_{i+1,v}) > \widetilde{\mathcal{C}}_{U}$}}
		\State \textnormal{$V_{\text{del}} \gets V_{\text{del}} \cup \{v\}$}
		\State \textnormal{$E_{\text{del}} \gets E_{\text{del}} \cup \{e : e \text{ is incident to } v\}$}
		\EndIf
		\EndFor
		\State \textnormal{$\widetilde{\mathcal{C}}_L \gets \min_{v \notin V_{\text{del}}} \Phi(\lambda_{i+1},p_{i+1,v})$}
		\If{\textnormal{$\widetilde{\mathcal{C}}_L = \widetilde{\mathcal{C}}_{U}$ or $\widetilde{\mathcal{C}}_{i+1} + \lambda_{i+1}\delta_{i+1} = \widetilde{\mathcal{C}}_{i}^+ + \lambda_{i+1}\delta_{i}^+$ (or $\widetilde{\mathcal{C}}_{i+1} + \lambda_{i+1}\delta_{i+1} = \widetilde{\mathcal{C}}_{i}^- + \lambda_{i+1}\delta_{i}^-$)}}
		\State \textbf{Output} \textnormal{$p^* = p_{U}$, \textbf{break}}
		\EndIf
		\State \textnormal{$G_{\text{red}} \gets (V \setminus V_{\text{del}}, E \setminus E_{\text{del}})$}
		\If{\textnormal{$\delta_{i+1} < \delta_{\max}$}}
		\State \textnormal{Update: $\widetilde{\mathcal{C}}_{i+1}^- \gets \widetilde{\mathcal{C}}_{i+1}$, $\delta_{i+1}^- \gets \delta_{i+1}$}
		\Else
		\State \textnormal{Update: $\widetilde{\mathcal{C}}_{i+1}^+ \gets \widetilde{\mathcal{C}}_{i+1}$, $\delta_{i+1}^+ \gets \delta_{i+1}$}
		\EndIf
		\EndWhile
	\end{algorithmic}
\end{algorithm}

\section{Proofs of Main Theoretical Results}
\subsection{Cut-Based Cost Bounds}
\subsubsection*{Proof of Property \ref{prop:cut_bounds}}
\begin{proof}
	Let $\mathcal{S}_{\min}$ be a minimum cardinality $(s,t)$-vertex cut in $G$.  
	By definition, any path $p \in \mathcal{P}(s,t)$ must pass through at least one vertex in $\mathcal{S}_{\min}$.  
	Therefore, the cost of the optimal constrained path $p^*$ satisfies:
	\[
	\widetilde{\mathcal{C}}^* \geq \min_{\substack{p \in \mathcal{P}(s,t)\\ p \cap \mathcal{S}_{\min} \neq \emptyset}} \widetilde{\mathcal{C}}_p^0,
	\]
	as this cost must be no less than that of the best unconstrained path through the cut.
	
	For the upper bound, suppose there exists a feasible path $p^\dagger$ with $\delta_{p^\dagger} \leq \delta_{\max}$.  
	Since $p^\dagger$ must intersect $\mathcal{S}_{\min}$, it follows that:
	\[
	\widetilde{\mathcal{C}}^* \leq \min_{\substack{p \in \mathcal{P}(s,t)\\ \delta_p \leq \delta_{\max},\; p \cap \mathcal{S}_{\min} \neq \emptyset}} \widetilde{\mathcal{C}}_p^\infty,
	\]
	establishing the upper bound.
	
	The cut $\mathcal{S}_{\min}$ need not be unique.  
	Any $(s,t)$-vertex cut suffices to derive valid bounds,  
	though using a minimum cut often yields tighter estimates.
	
	Finally, analogous arguments apply to edge cuts:  
	one may replace $\mathcal{S}_{\min}$ with a minimum set of edges that disconnect $s$ from $t$,  
	with the same bounding principles applied to edge-disjoint paths.
\end{proof}
\subsection{Identification Results under Lagrangian Ambiguity}
\label{app: proofs_sec3_1}
This section presents detailed proofs for 
Propositions~\ref{prop-identify opt case2-special}--\ref{prop:general-B-disamb}, 
which establish the correctness of the TPVE algorithm 
under increasingly general disambiguation budget scenarios.

\begin{proof} (Proof of Proposition \ref{prop-identify opt case2-special})
We begin by analyzing the feasibility of the three candidate paths:
\begin{itemize}
    \item $p_1$ is infeasible since it intersects both obstacles and 
    $\delta_{p_1} = \delta_{x_1} + \delta_{x_2} > \delta_{\max}$.
    Therefore, during Phase 1 of TPVE, any vertex $v \in p_1$ lying in the intersection of $x_1$ and $x_2$ 
    fails the feasibility test and is eliminated.
    Consequently, $p_1$ is pruned.
    
    \item $p_2$ is feasible with $\delta_{p_2} = 0$ and may appear attractive due to its low cost.
    However, if $p^*$ has higher modified cost but better satisfies the constraint 
    (i.e., utilizes part of the budget for informative disambiguation), 
    then $p_2$ cannot update the incumbent upper bound $\widetilde{\mathcal{C}}_U$.
    Specifically, for any $v \in p^* \setminus p_2$, the shortest feasible path through $v$ 
    will necessarily involve $x^*$ and yield cost $\widetilde{\mathcal{C}}_{p^*}$, 
    which defines $\widetilde{\mathcal{C}}_U$.
    
    \item Since $p^*$ intersects only one obstacle, say $x^*$, and satisfies 
    $\delta_{p^*} = \delta_{x^*} \leq \delta_{\max}$, it passes the feasibility test.
    Furthermore, the dual cost 
    $\widetilde{\mathcal{C}}_{p^*} + \lambda (\delta_{p^*} - \delta_{\max})$ 
    can become optimal at some $\lambda \geq 0$, depending on the cost-risk tradeoff.
    In Phase 2, such vertices on $p^*$ remain active, and the path $p^*$ will be recovered either directly 
    or as the vertex-constrained minimum modified cost path.
\end{itemize}

Thus, both $p_1$ (infeasible) and $p_2$ (suboptimal) are eliminated by the algorithm, 
while $p^*$ survives and is ultimately selected as the optimal path.
This concludes the proof.
\end{proof}

\begin{proof} (Proof of Proposition \ref{prop:single-disamb-induction})
We proceed by induction on the number of obstacles $k \geq 2$.

\textbf{Base case ($k = 2$):} This case corresponds to Proposition~\ref{prop-identify opt case2-special}.
Under the given assumptions, the TPVE algorithm eliminates the infeasible path 
(intersecting both obstacles) and the obstacle-free path (if suboptimal), 
and successfully identifies $p^*$ as the unique feasible and optimal path.

\textbf{Inductive step:} Assume the proposition holds for $k = m \geq 2$ obstacles.
That is, if at most one obstacle can be disambiguated and a path $p^*$ exists 
intersecting exactly one such obstacle while all competing paths either:
(i) intersect multiple obstacles, making them infeasible, or
(ii) avoid all obstacles but incur higher cost,
then $p^*$ is retained and identified by TPVE.

\textbf{Now consider $k = m+1$ obstacles.} Let $X = \{x_1, x_2, \dots, x_{m+1}\}$, and suppose again that:
- $p^*$ intersects a single obstacle $x^* \in X$ such that $\delta_{x^*} \leq \delta_{\max}$,
- all other paths either intersect more than one obstacle 
(i.e., $\delta_p > \delta_{\max}$), or avoid all obstacles 
(i.e., $\delta_p = 0$ but $\widetilde{\mathcal{C}}_p > \widetilde{\mathcal{C}}_{p^*}$).

Let $V(p)$ denote the set of vertices on path $p$.
During Phase 1 of TPVE:
- All paths that intersect multiple obstacles are pruned due to feasibility failure, and
- All paths avoiding obstacles are retained but do not update the current upper bound 
$\widetilde{\mathcal{C}}_U = \widetilde{\mathcal{C}}_{p^*}$, since they are suboptimal.

During Phase 2:
- Any vertex on $p^*$ lies on a feasible path (namely, $p^*$) whose cost equals 
the current best known cost $\widetilde{\mathcal{C}}_U$.
- Competing paths either do not improve the bound or are infeasible, 
hence do not eliminate vertices on $p^*$.

Thus, all non-optimal paths are eliminated by TPVE, while all vertices on $p^*$ are retained.
Therefore, $p^*$ is preserved and ultimately identified as the optimal solution.

By the principle of mathematical induction, the result holds for all $k \geq 2$.
\end{proof}

\begin{proof} (Proof of Proposition \ref{prop:general-B-disamb})
We prove the claim by induction on the disambiguation budget parameter $B \geq 1$.

\textbf{Base case ($B=1$):} This corresponds to Proposition~\ref{prop:single-disamb-induction}, already established.
Under the assumptions, the TPVE algorithm prunes all infeasible or suboptimal paths and preserves $p^*$.

\textbf{Inductive step:} Suppose the result holds for disambiguation budgets up to $B = m$ obstacles.
That is, if a path $p^*$ intersects at most $m$ disjoint obstacles 
whose total disambiguation cost does not exceed $\delta_{\max}$, 
and all other paths either exceed budget or are cost-suboptimal, 
then TPVE preserves $p^*$ and eliminates all competing paths.

\textbf{Now consider $B = m+1$.} Let $p^*$ intersect obstacles $X^* = \{x_{i_1}, \dots, x_{i_{m+1}}\}$ such that:
\[
\sum_{x \in X^*} \delta_x \leq \delta_{\max},
\]
and $p^*$ is the minimal-cost such feasible path.

During Phase 1 of TPVE:
\begin{itemize}
    \item Any vertex $v \in p^*$ belongs to a feasible path (namely $p^*$).
    Since $|X^*| = m+1$ and the total cost of disambiguation is within budget, 
    the feasibility test does not eliminate $v$.
    \item Any path $p$ intersecting a set $X_p$ of more than $m+1$ obstacles, 
    or with $\sum_{x \in X_p} \delta_x > \delta_{\max}$, is deemed infeasible and eliminated.
    \item Any path avoiding all obstacles is retained, 
    but does not improve the upper bound cost 
    $\widetilde{\mathcal{C}}_U = \widetilde{\mathcal{C}}_{p^*}$.
\end{itemize}

During Phase 2:
\begin{itemize}
    \item Since $p^*$ is optimal and feasible under budget $\delta_{\max}$, 
    its vertices remain in the graph.
    \item Competing paths that either violate feasibility or fail to improve the modified cost 
    are not retained.
\end{itemize}

Thus, $p^*$ is never eliminated and is ultimately selected by the algorithm.

By induction, the result holds for all $B \geq 1$.
\end{proof}

\subsection{Spurious Dual Optima Vanish under High RBG}
\label{app: proofs_sec3_2}
\subsubsection*{Proof of Property \ref{prop:risk-blockage-gradient}}
\begin{proof} 
	Suppose paths $p_1$ and $p_2$ satisfy conditions (i)–(iii). 
	Let $p^*$ be the optimal path,
	with cost $\widetilde{\mathcal{C}}^*$ and disambiguation cost $\delta^* \leq \delta_{\max}$.
	
	From (iii), we have:
	\[
	0 < \widetilde{\mathcal{C}}_{p_2} - \widetilde{\mathcal{C}}^* < \lambda^*(\delta^* - \delta_{p_2}),
	\]
	implying that the apparent suboptimality of $p_2$ is offset under the dual formulation by its lower disambiguation cost.
	
	As RBG increases, the risk component $r_p$ increases for any path intersecting uncertain obstacles. 
	In particular, the gap $r^* - r_{p_2}$ increases, and for sufficiently high RBG:
	\[
	\ell_{p_2} - \ell^* < r^* - r_{p_2} \quad \Rightarrow \quad \widetilde{\mathcal{C}}_{p_2} < \widetilde{\mathcal{C}}^*,
	\]
	contradicting the optimality of $p^*$.
	
	For $p_1$, since it violates the disambiguation constraint ($\delta_{p_1} > \delta_{\max}$), 
	it is infeasible, yet (iii) implies:
	\[
	0 < \lambda^*(\delta_{p_1} - \delta_{\max}) < \widetilde{\mathcal{C}}^* - \widetilde{\mathcal{C}}_{p_1}.
	\]
	As RBG increases, so does $r_{p_1}$ relative to $r^*$, and for large enough RBG:
	\[
	\ell^* - \ell_{p_1} < r_{p_1} - r^* \quad \Rightarrow \quad \widetilde{\mathcal{C}}_{p_1} < \widetilde{\mathcal{C}}^*,
	\]
	again contradicting optimality.
	
	Thus, under high RBG, the coexistence of such dual-optimal 
	but distinct paths becomes increasingly unlikely. 
\end{proof}

\subsection{ Theorems and Propositions in Section \ref{sec:theory-policy}}
\label{app: proofs_sec4}
\subsubsection*{Proof of Theorem \ref{thm:policy-correctness}}
\begin{proof}
We divide the argument into two parts: feasibility and surrogate optimality.

\noindent
\textbf{(i) Feasibility.}
The COLOGR algorithm computes the optimal path over $G_{\text{adj}}$ subject to the disambiguation constraint $\delta_p \le \delta_{\max}$. This constraint is explicitly enforced during the graph reduction and optimization steps (e.g., via constrained Dijkstra variants). Thus, the selected solution $p^*$ is guaranteed to be feasible.

\noindent
\textbf{(ii) Surrogate Optimality.}
Let $\lambda^* \geq 0$ be the optimal Lagrange multiplier that maximizes the dual function:
\[
\Phi(\lambda) = \min_{p \in \mathcal{P}(s,t; G_{\text{adj}})} \left\{ 
\widetilde{\mathcal{C}}_p + \lambda (\delta_p - \delta_{\max}) 
\right\}.
\]
Since $\Phi(\lambda)$ is piecewise linear and concave, standard results from Lagrangian relaxation ensure that, if a path $p^*$ is both feasible and achieves the dual minimum at $\lambda^*$, then
\[
\widetilde{\mathcal{C}}_{p^*} + \lambda^*(\delta_{p^*} - \delta_{\max}) = \Phi(\lambda^*).
\]
Because $\delta_{p^*} \le \delta_{\max}$, the penalty term vanishes, and $p^*$ minimizes the surrogate cost over all feasible paths:
\[
\widetilde{\mathcal{C}}_{p^*} = \min \left\{ \widetilde{\mathcal{C}}_p : 
p \in \mathcal{P}(s,t; G_{\text{adj}}),\; \delta_p \le \delta_{\max} \right\}.
\]
\end{proof} 

\section{Tuning the Scaling Parameter $\alpha$ in the Linear Undesirability Function}
\label{app: sec_alpha}
To tailor traversal behavior to environmental uncertainty, 
we incorporate an $\alpha$-selection process into the policy. 
Since true obstacle status is unknown until disambiguation, 
each obstacle is modeled as a random variable with blockage probability $\pi_x$. 
Due to this stochasticity and the large number of obstacles, 
brute-force simulation across many $\alpha$ values is impractical.

Instead, we investigate how the optimal $\alpha$ varies with prior information accessible before traversal, including:
(i) disambiguation resources $\delta_x$,  
(ii) the proportion $\rho_T$ of true obstacles, and  
(iii) the probability distribution over obstacle types.

Assuming perfect sensing ($\pi_x = 1$ for true, $\pi_x = 0$ for false obstacles), 
true obstacles incur infinite penalty, while false obstacles ideally impose risk proportional to $\delta_x$. 
However, the standard linear undesirability function,
\[
r_p^L = -\sum_{x \in X: p \cap x \neq \emptyset} \alpha \log(1 - \pi_x),
\]
assigns negligible risk as $\pi_x \to 0$, underweighting the cost of false obstacle traversal. 
Thus, incorporating $\delta_x$ into the base penalty yields a better approximation.

We assess how the optimal $\alpha$ varies with $\rho_T$. 
Using the same simulation setup as in Section \ref{sec:MC-simulations}, 
we fix $n = 20, 40, 80$ potential obstacles and vary $\rho_T \in \{0, 0.1, 0.2, 0.4, 0.8, 1\}$. 
Figure \ref{sup-fig:different alpha and no} shows the average traversal cost across 100 replications per setting.

As expected, traversal cost increases with $\rho_T$ due to reduced path availability. 
Moreover, optimal $\alpha$ also increases with $\rho_T$. When $\rho_T$ is small, 
the constraint drives disambiguation selection toward low-risk paths, and a small $\alpha$ suffices. 
As $\rho_T$ rises, avoiding true obstacles becomes critical, 
and larger $\alpha$ values are needed to accentuate risk differences.

\begin{figure}[H]
	\centering
	\begin{tabular}{c}
		\begin{subfigure}[b]{0.333\textwidth}
			\centering
			\includegraphics[scale=0.333]{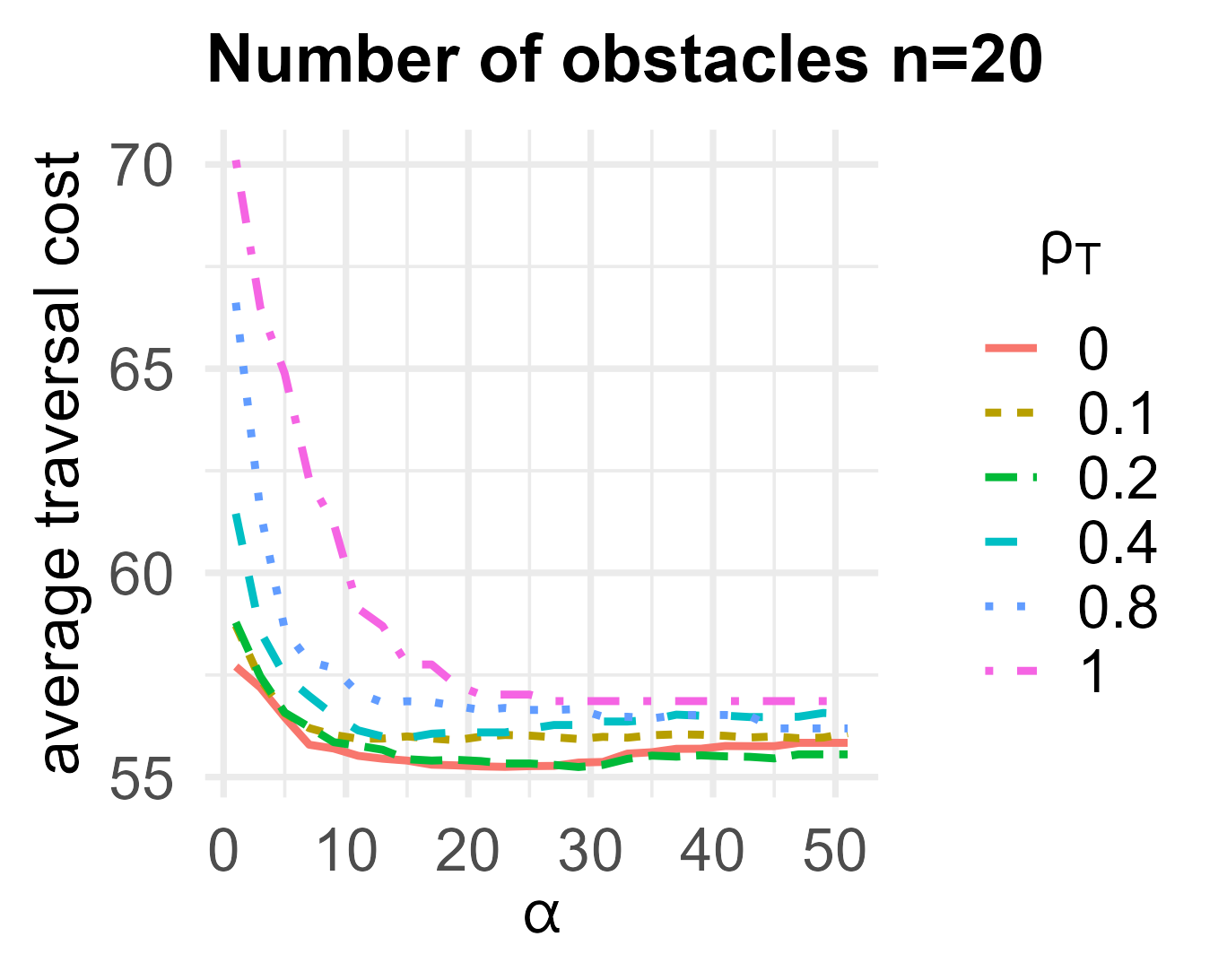} \caption{}
		\end{subfigure}
		\begin{subfigure}[b]{0.333\textwidth}
			\centering
			\includegraphics[scale=0.333]{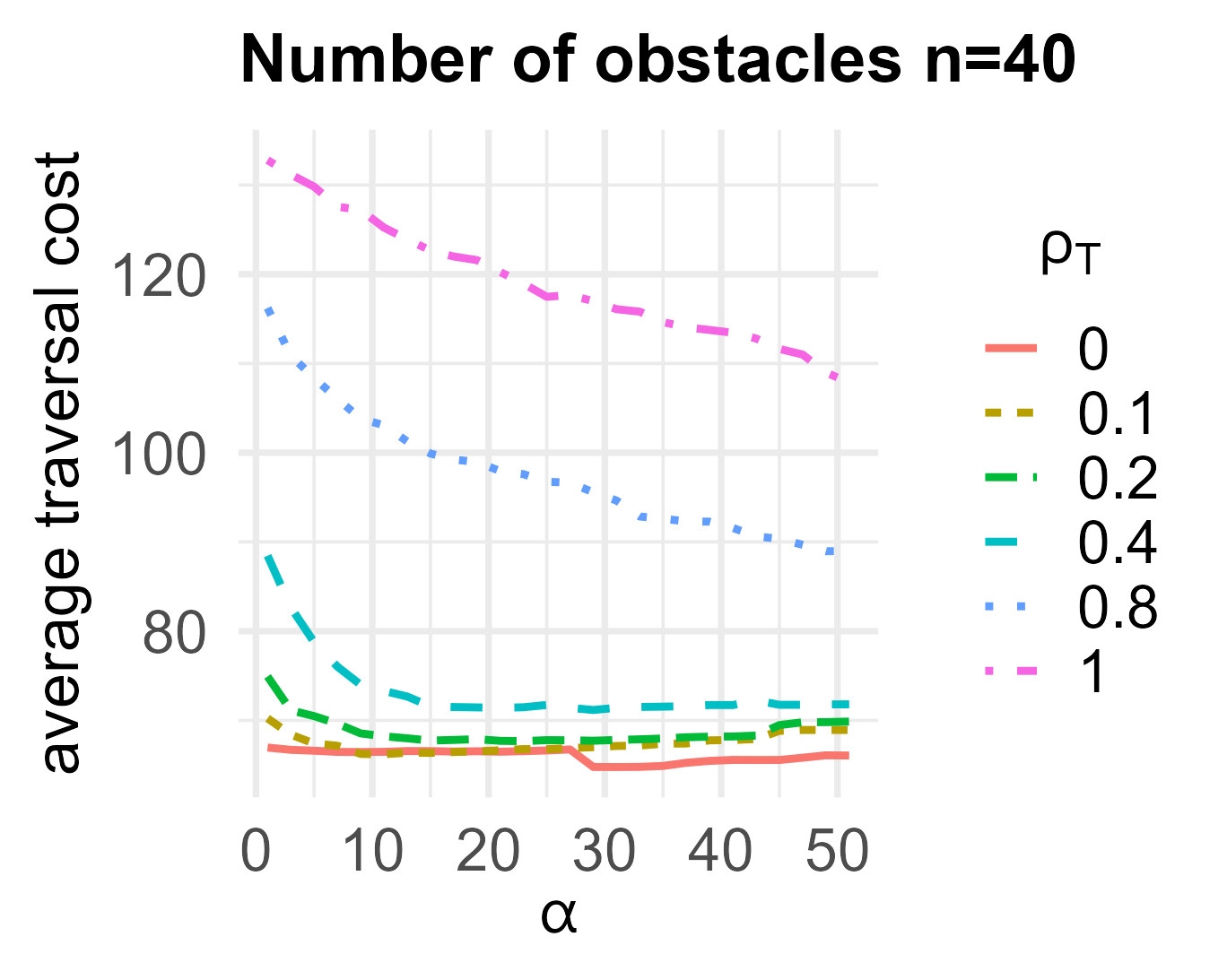} \caption{}
		\end{subfigure}
		\begin{subfigure}[b]{0.333\textwidth}
			\centering
			\includegraphics[scale=0.333]{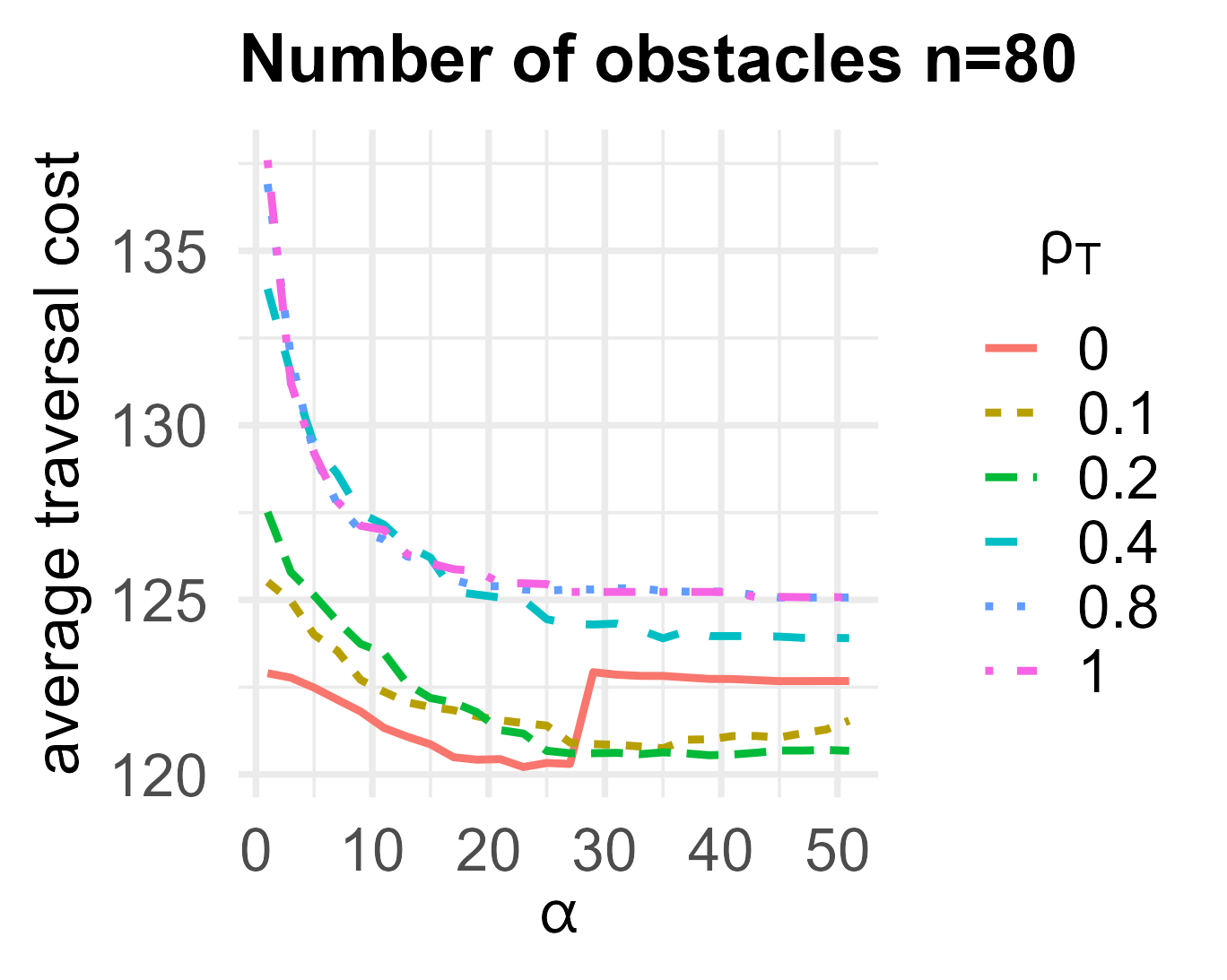} \caption{}
		\end{subfigure}
	\end{tabular}
	\caption{Effect of $\alpha$ and true obstacle proportion $\rho_T$ on mean traversal cost. Subplots: (a) $n = 20$, (b) $n = 40$, (c) $n = 80$.}
	\label{sup-fig:different alpha and no}
\end{figure}

To improve adaptivity, we propose a modified risk function:
\[
r_p = \sum_{x \in X: p \cap x \neq \emptyset} \left(\delta_x - \alpha_x \log(1 - \pi_x)\right),
\]
where $\alpha_x$ is a local scaling parameter. 
This form satisfies $r_p \to \sum \delta_x$ as $\pi_x \to 0$, 
ensuring alignment with disambiguation cost for low-risk obstacles.

The key challenge is determining $\alpha_x$. 
We propose a Bayesian framework that uses prior estimates of sensor reliability and 
obstacle prevalence to compute adjusted probabilities $\pi_{x,\text{adj}}$. 
Let $\bar{\pi}$ denote the average prior probability of obstacles being true. 
If the sensor produces outputs drawn from $f_T$ for true and $f_F$ for false obstacles, 
then the adjusted probability becomes:
\[
\pi_{x,\text{adj}} = P(f_T \mid \pi_x) = 
\frac{L(\pi_x \mid f_T) \cdot \bar{\pi}}{L(\pi_x \mid f_T) \cdot \bar{\pi} + L(\pi_x \mid f_F) \cdot (1 - \bar{\pi})}.
\]
Assuming a preselected upper bound $\alpha_{\max}$, we set:
\[
\alpha_x = \pi_{x,\text{adj}} \cdot \alpha_{\max},
\]
yielding the Bayesian linear undesirability function:
\[
r_p^{LB} = \sum_{x \in X: p \cap x \neq \emptyset} \left(\delta_x - \alpha_x \log(1 - \pi_x)\right).
\]

This formulation improves traversal policy fidelity by modulating risk sensitivity based on both obstacle-level uncertainty and prior knowledge, with performance benefits validated in our simulations (Section \ref{app: simulation_comp}).

\subsection{Convergence to the Benchmark under Perfect Sensor Precision}
We now show that the RCDP policy equipped with the Bayesian undesirability function 
asymptotically achieves the benchmark traversal cost as sensor accuracy becomes perfect. 

\begin{theorem}[Convergence to Benchmark under Perfect Sensing]
	\label{thm:convergence}
	Let $ \mathcal{C}_p^{LB} $ denote the traversal cost of the path selected by the RCDP policy 
	using the Bayesian linear undesirability risk function $ r^{LB} $, 
	and let $ \mathcal{C}_p^{bm} $ denote the cost of the optimal path under perfect obstacle knowledge.
	
	Assume:
	\begin{enumerate}[label=(\alph*)]
		\item Sensor probabilities converge pointwise: $\pi_x \to \mathbb{I}\{x \in X_T\}$ for all $x \in X$.
		\item $\alpha_x$ is uniformly bounded and independent of $\pi_x$.
		\item Edge lengths $\ell_e$ and disambiguation costs $\delta_e$ are fixed.
		\item For all $p$, the actual traversal cost $\mathcal{C}_p$ is bounded above.
	\end{enumerate}
	
	Then:
	\[
	\mathbb{E}[\mathcal{C}_{p^{LB}}] \to \mathcal{C}_{p^{bm}},
	\]
	where $p^{LB}$ is any path selected by the RCDP policy under $r^{LB}$.
\end{theorem}

\begin{proof}
	Let $G = (V,E)$ denote the spatial graph. Under perfect sensing, edge costs become:
	\[
	\mathcal{C}_e^{bm} = 
	\begin{cases}
		\ell_e + \frac{1}{2} \sum_{x \in X_F: x \cap e \neq \emptyset} \delta_x, & \text{if } e \cap X_T = \emptyset, \\
		\infty, & \text{otherwise}.
	\end{cases}
	\]
	
	Now consider the Bayesian surrogate cost:
	\[
	\widetilde{\mathcal{C}}_e^{LB} = \ell_e + \frac{1}{2} \sum_{x \in X: x \cap e \neq \emptyset} 
	\left( \delta_x - \alpha_x \log(1 - \pi_x) \right).
	\]
	
	As $\pi_x \to \mathbb{I}\{x \in X_T\}$, we have:
	- For $x \in X_T$, $-\log(1 - \pi_x) \to \infty$ $\Rightarrow$ $\widetilde{\mathcal{C}}_e^{LB} \to \infty$ if $e$ intersects a true obstacle.
	- For $x \in X_F$, $-\log(1 - \pi_x) \to 0$ $\Rightarrow$ $\widetilde{\mathcal{C}}_e^{LB} \to \mathcal{C}_e^{bm}$.
	
	Hence, $\widetilde{\mathcal{C}}_e^{LB} \to \mathcal{C}_e^{bm}$ uniformly over $e$.
	
	By continuity of path cost with respect to edge costs, this implies:
	\[
	\widetilde{\mathcal{C}}_p^{LB} \to \mathcal{C}_p^{bm} \quad \text{for all feasible paths } p.
	\]
	
	The feasibility constraint $\delta_p \leq \delta_{\max}$ remains unchanged, as it depends only on obstacle geometry.
	
	Since the argmin over continuous and bounded costs is stable under uniform convergence (Berge's Maximum Theorem), 
	the selected path $p^{LB}$ satisfies:
	\[
	\mathbb{E}[\mathcal{C}_{p^{LB}}] \to \mathcal{C}_{p^{bm}}.
	\]
\end{proof}

\noindent
Simulation results in Figure~\ref{sup-fig: rLB_convergence} support this convergence, 
showing that as sensor precision increases, the traversal cost under $r^{LB}$ closely approximates the benchmark.

\section{Details of the Monte Carlo Simulations}
\subsection{Graph Reduction Performance: TPVE vs. SNE}
\label{app: simulation_red}
We conduct Monte Carlo simulations (1000 replications) to compare the performance of 
the proposed Two-Phase Vertex Elimination (TPVE) method with the original Simple Node Elimination (SNE).  
Table \ref{app: Table_reduction compare} reports results over all replications and 
for selected cases that exhibit either a nonzero duality gap or 
noticeable cost differences—highlighting meaningful distinctions between the two reduction strategies.

The simulation domain is $\Omega = [0, 100] \times [0, 50]$ and 
contains 20, 40, or 80 disk-shaped obstacles, 
among which 4, 8, or 16 are true obstacles, respectively.  
The source and target are fixed at $s = (50, 50)$ and $t = (50, 1)$, 
resulting in an initial graph with 5151 vertices.

To isolate the impact of graph reduction, 
we adopt the proposed traversal policy with 
the linear undesirability function ($\alpha = 15, 30$) and consider both:
(i) heterogeneous disambiguation costs, and  
(ii) a uniform cost setting where the constraint simplifies to a bound on the number of disambiguated obstacles.

Further simulation and policy comparisons are provided in Section~\ref{app: simulation_comp}.

As seen in Table \ref{app: Table_reduction compare}, 
TPVE consistently produces smaller final graphs and closes the duality gap in all runs (gap = 0), 
while SNE shows a nonzero gap in aggregate and selected cases (e.g., 0.026 on average in selected cases).  
Although both methods sometimes yield paths with identical costs, 
SNE more frequently fails to identify the true optimal path due to suboptimal pruning.  
This results either from unresolved duality gaps or from ties in surrogate cost, 
where TPVE better captures the true optimum.


\begin{table}[H]
	\centering
	\resizebox{13cm}{!}{%
		\begin{tabular}{@{}c|ccc|ccc@{}}
			\toprule
			& \multicolumn{3}{c|}{Two-Phase Vertex Elimination (TPVE)} & \multicolumn{3}{c}{Simple Node Elimination (SNE)} \\
			& \begin{tabular}[c]{@{}c@{}}Graph Size\\ (number of vertices)\end{tabular} & 
			\begin{tabular}[c]{@{}c@{}}Duality Gap\\ ($\frac{\text{upper-lower}}{\text{lower}}$)\end{tabular} & 
			$\mathcal{C}_p$ & 
			\begin{tabular}[c]{@{}c@{}}Graph Size\\ (number of vertices)\end{tabular} & 
			\begin{tabular}[c]{@{}c@{}}Duality Gap\\ ($\frac{\text{upper-lower}}{\text{lower}}$)\end{tabular} & 
			$\mathcal{C}_p$ \\ \midrule
			Case 1 & 728 & 0 & 75.658 & 499 & 0.016 & 75.658 \\
			Case 2 & 136 & 0 & 58.490 & 4231 & 0.064 & 59.405 \\
			Case 3 & 158 & 0 & 62.450 & 278 & 0.023 & 62.450 \\
			Case 4 & 225 & 0 & 59.282 & 4213 & 0 & 63.090 \\
			Case 5 & 137 & 0 & 64.065 & 829 & 0.879 & 64.065 \\
			Average (50 selected cases) & 142 & 0 & 57.304 & 3948 & 0.026 & 61.015 \\
			Average (1000 cases) & 2423 & 0 & 58.444 & 3745 & 0.001 & 58.667 \\ \bottomrule
		\end{tabular}%
	}
	\caption{Comparison of TPVE and SNE on final graph size, duality gap, and traversal cost.}
	\label{app: Table_reduction compare}
\end{table}

\subsection{Comprehensive Comparisons of Traversal Policies}
\label{app: simulation_comp}
In this section, we provide a comprehensive comparison of the policies, 
including detailed analyses using the metrics introduced in Section \ref{sec:MC-simulations} as well as additional evaluation metrics.

\begin{enumerate}
	\item \emph{Average traversal cost with confidence interval}: this primary metric measures the 
	mean traversal cost from starting point to target point across different environments, 
	directly reflects policy's quality, where lower values indicate superior overall performance.
	\item \emph{Error relative to optimal}: calculated as 
	$\sqrt{\frac{\sum_i(\mathcal{C}_{i}-\mathcal{C}^*)^2}{n}}$, where $\mathcal{C}_{i}$ is the 
	traversal cost in $i^{th}$ simulation, quantifying the deviation from an optimal cost 
	generated by an offline benchmark assuming perfect knowledge of obstacle status. 
	This metric assesses how closely each policy approximates the idealized best scenario.
	\item \emph{Variability}: using standard deviation of traversal costs across different 
	environments, reflecting each policy's stability and reliability under varying uncertain 
	conditions.
	\item \emph{Quantile metrics}: specifically, the $25^{th}$ and $75^{th}$ percentiles of 
	traversal costs across different environments are analyzed to reveal performance under both 
	favorable and challenging scenarios, offering additional insights into each policy's 
	robustness and consistency.
\end{enumerate}

Together, these metrics provide comprehensive assessment of each policy's strengths and weaknesses.

Figure~\ref{Figure4} visualizes an example comparison under the simplified scenario with $N_{\max} = 2$, 
contrasting the RCDP policy and the constrained RD policy.

\begin{figure}[H]
	\centering
	\begin{tabular}{c}
		\begin{subfigure}[b]{0.5\textwidth}
			\centering
			\includegraphics[width=\textwidth]{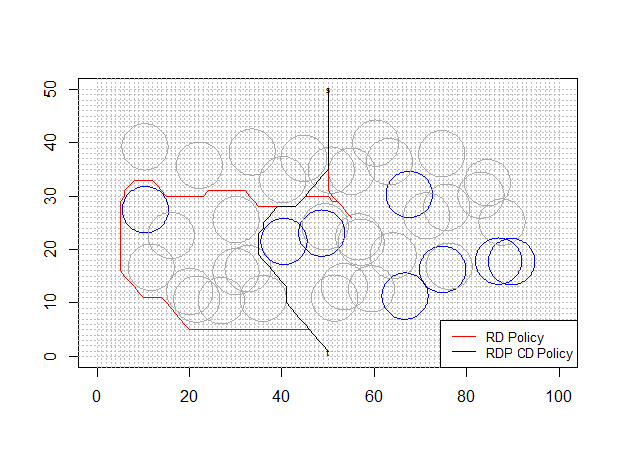}
			\caption{}
		\end{subfigure}
	\end{tabular}
	\caption{Example traversal paths under $N_{\max}=2$ for RCDP and constrained RD policies.}
	\label{Figure4}
\end{figure}

Figure~\ref{sup-fig:avg with CI} shows the mean traversal cost with 95\% confidence intervals for all policies. 
As expected, traversal cost increases with obstacle count due to greater obstruction, 
while larger disambiguation budgets yield lower costs by enabling more effective exploration and obstacle avoidance.

\begin{figure}[H]
	\centering
	\begin{tabular}{ccc}
		\begin{subfigure}[b]{0.33\textwidth}
			\centering
			\includegraphics[width=\textwidth]{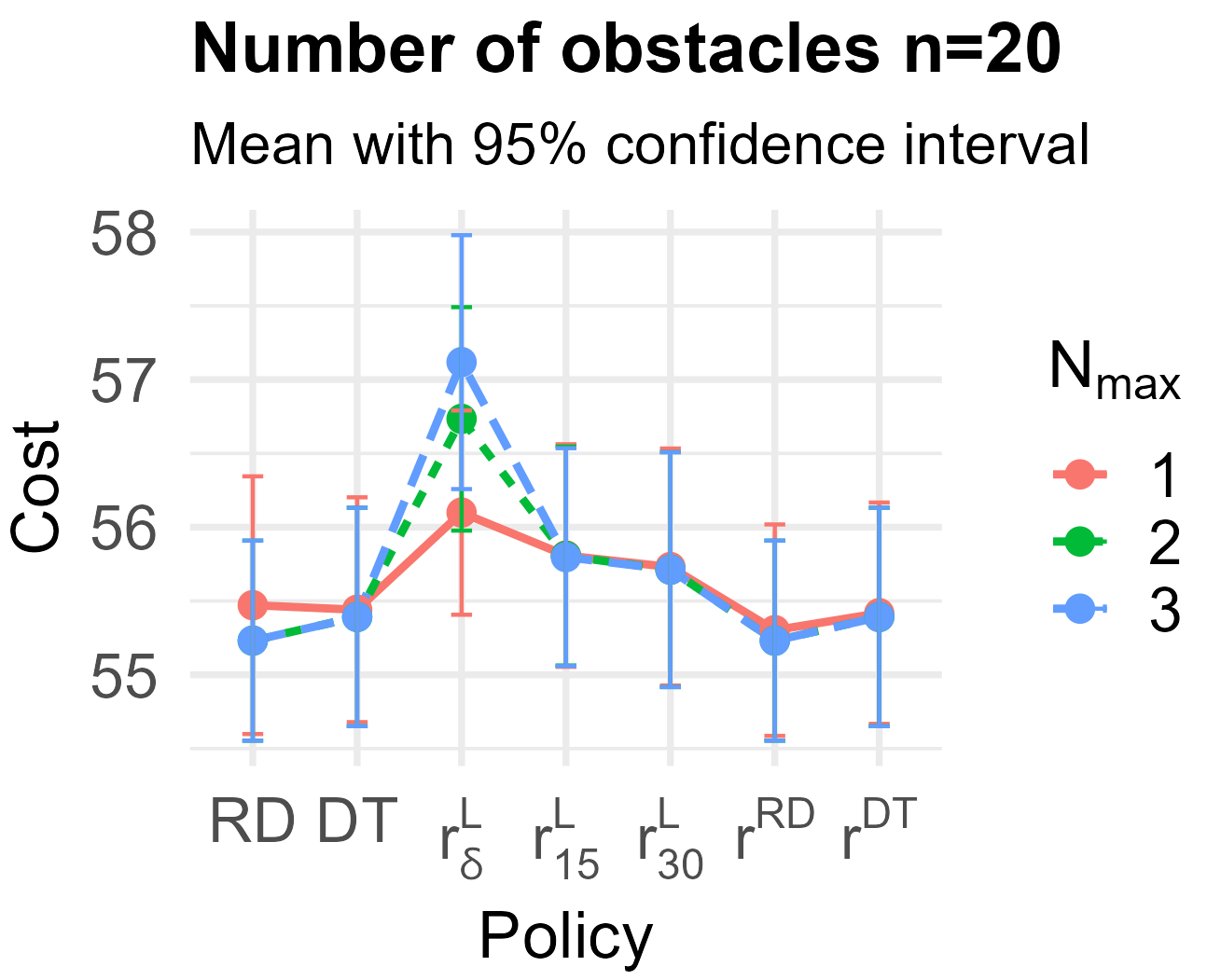}
		\end{subfigure} &
		\begin{subfigure}[b]{0.33\textwidth}
			\centering
			\includegraphics[width=\textwidth]{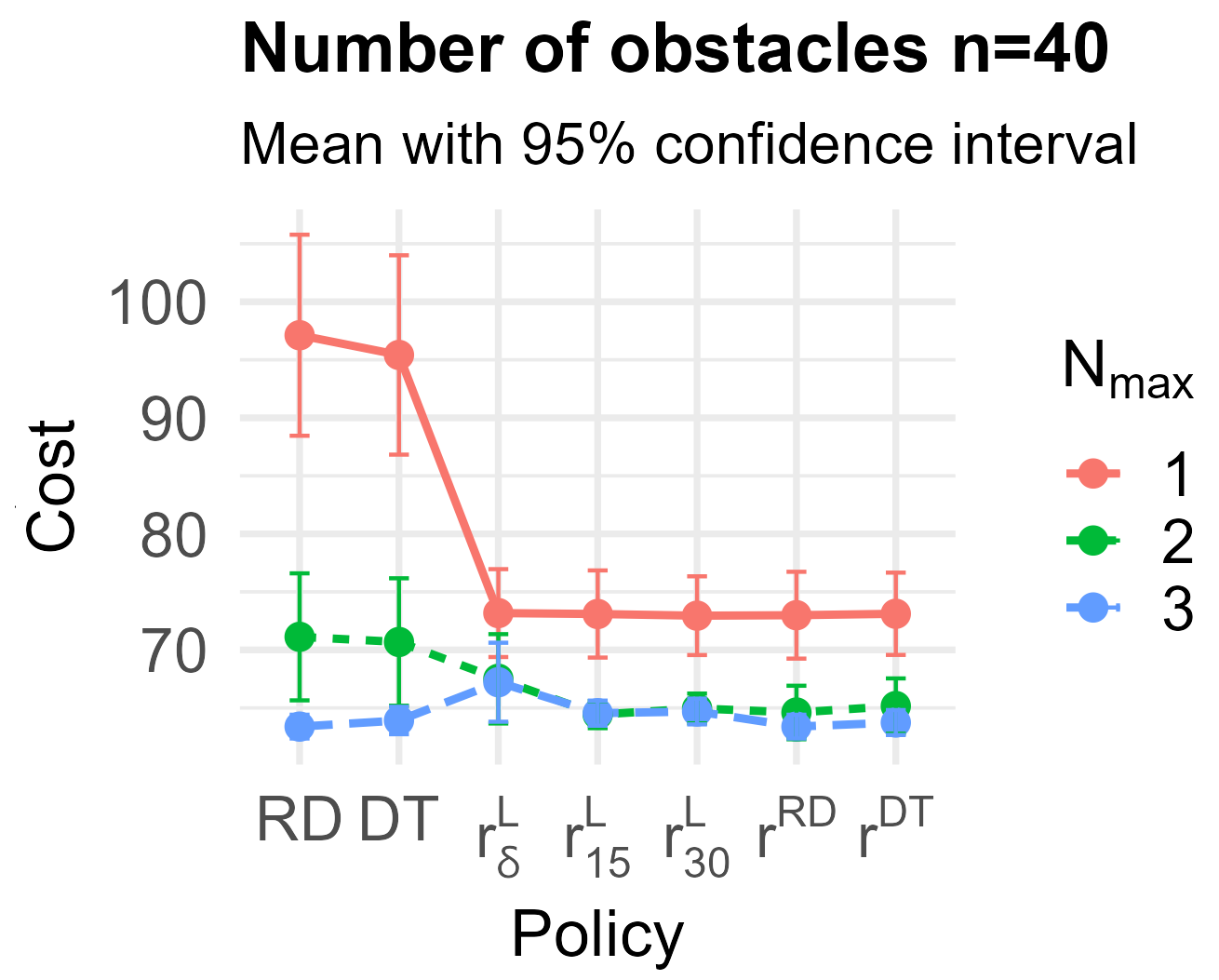}
		\end{subfigure} &
		\begin{subfigure}[b]{0.33\textwidth}
			\centering
			\includegraphics[width=\textwidth]{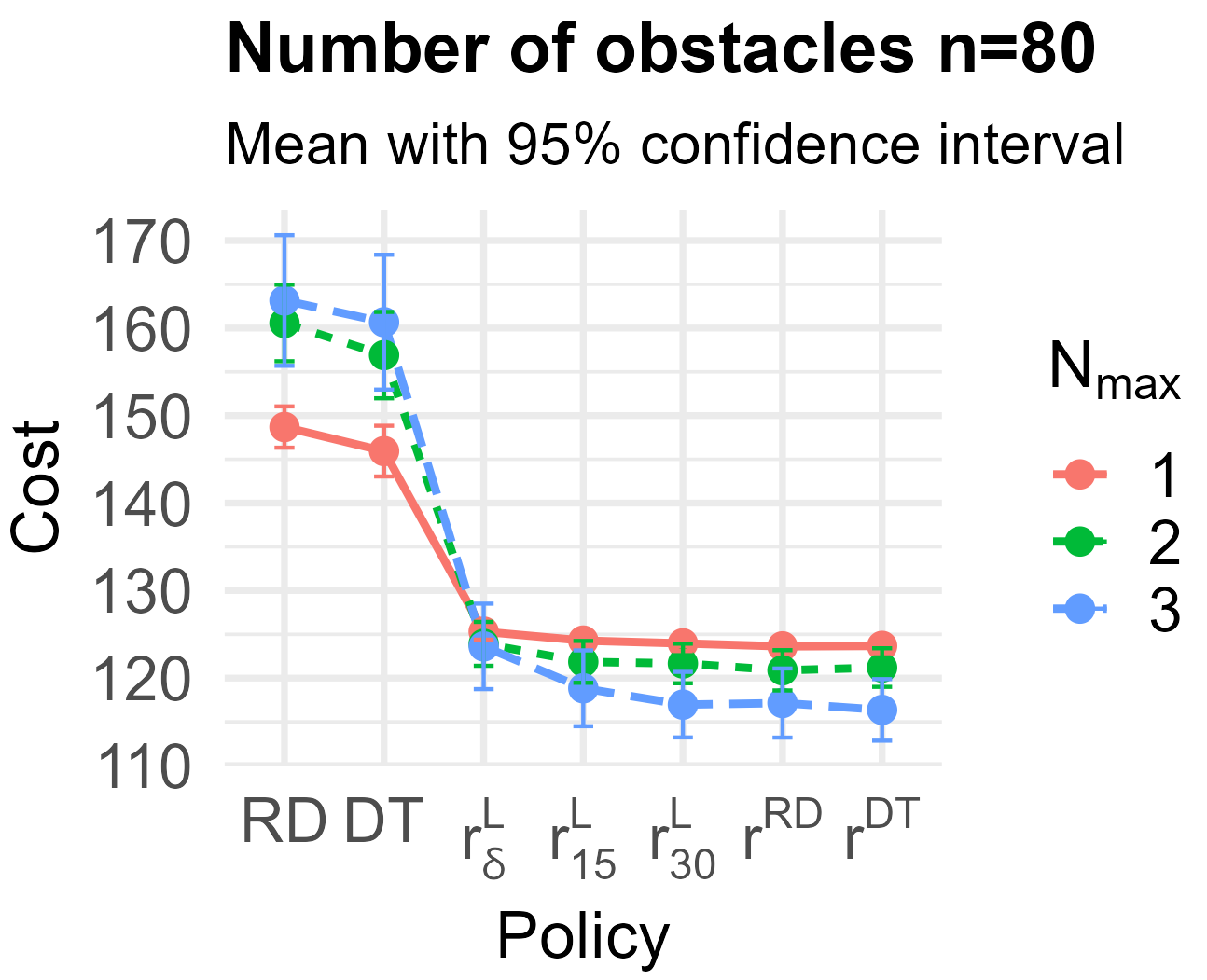}
		\end{subfigure} \\
		\begin{subfigure}[b]{0.33\textwidth}
			\centering
			\includegraphics[width=\textwidth]{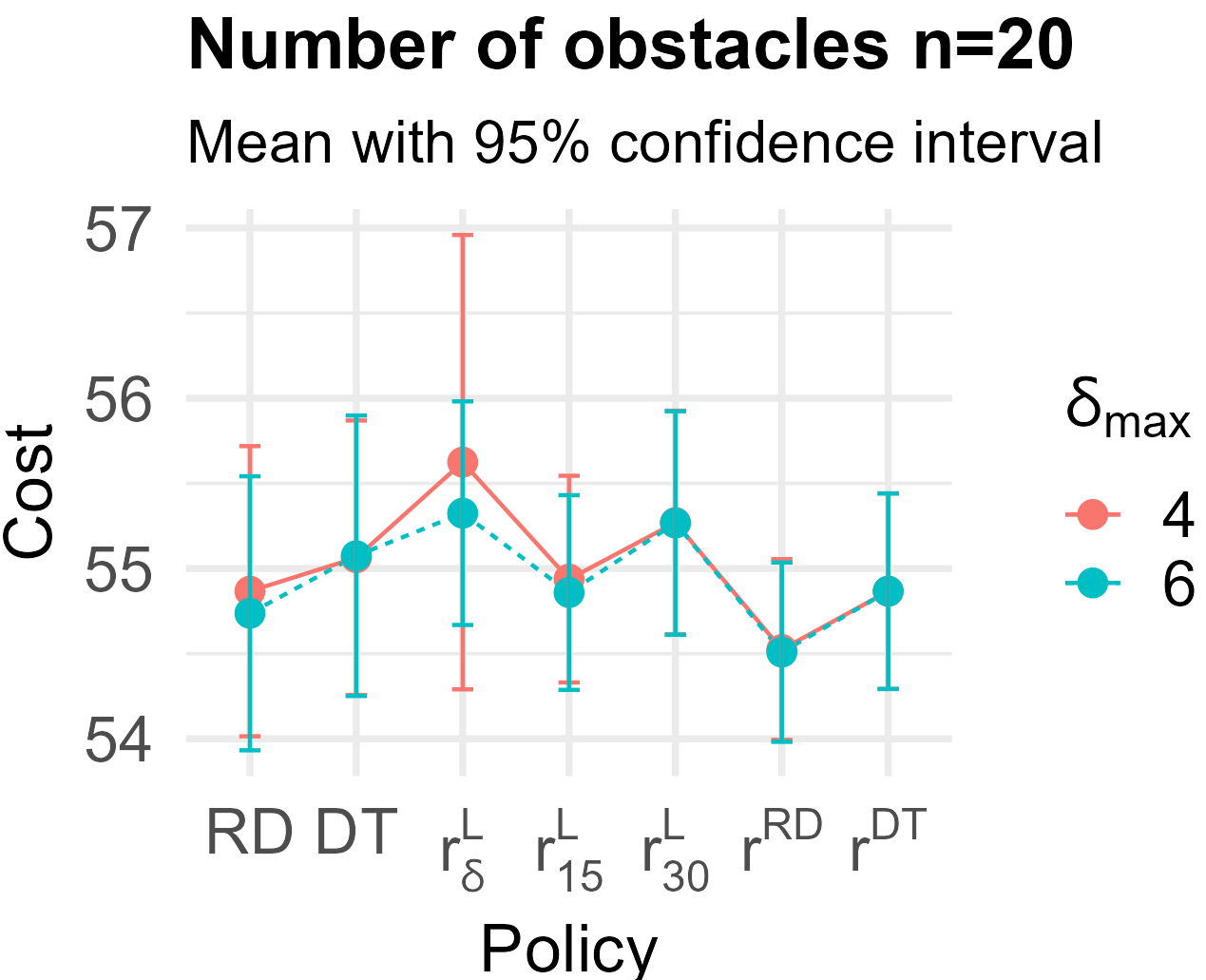}
		\end{subfigure} &
		\begin{subfigure}[b]{0.33\textwidth}
			\centering
			\includegraphics[width=\textwidth]{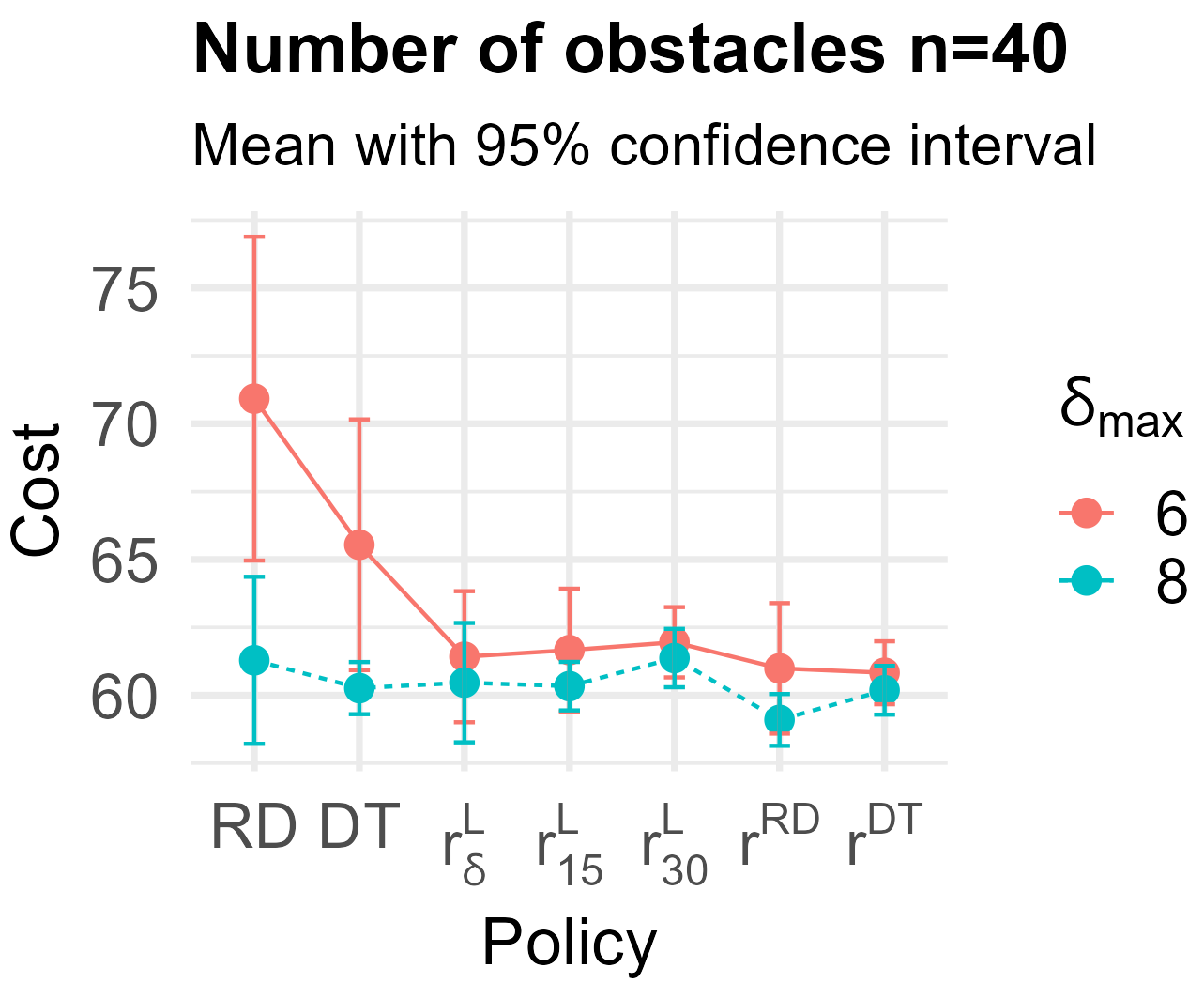}
		\end{subfigure} &
		\begin{subfigure}[b]{0.33\textwidth}
			\centering
			\includegraphics[width=\textwidth]{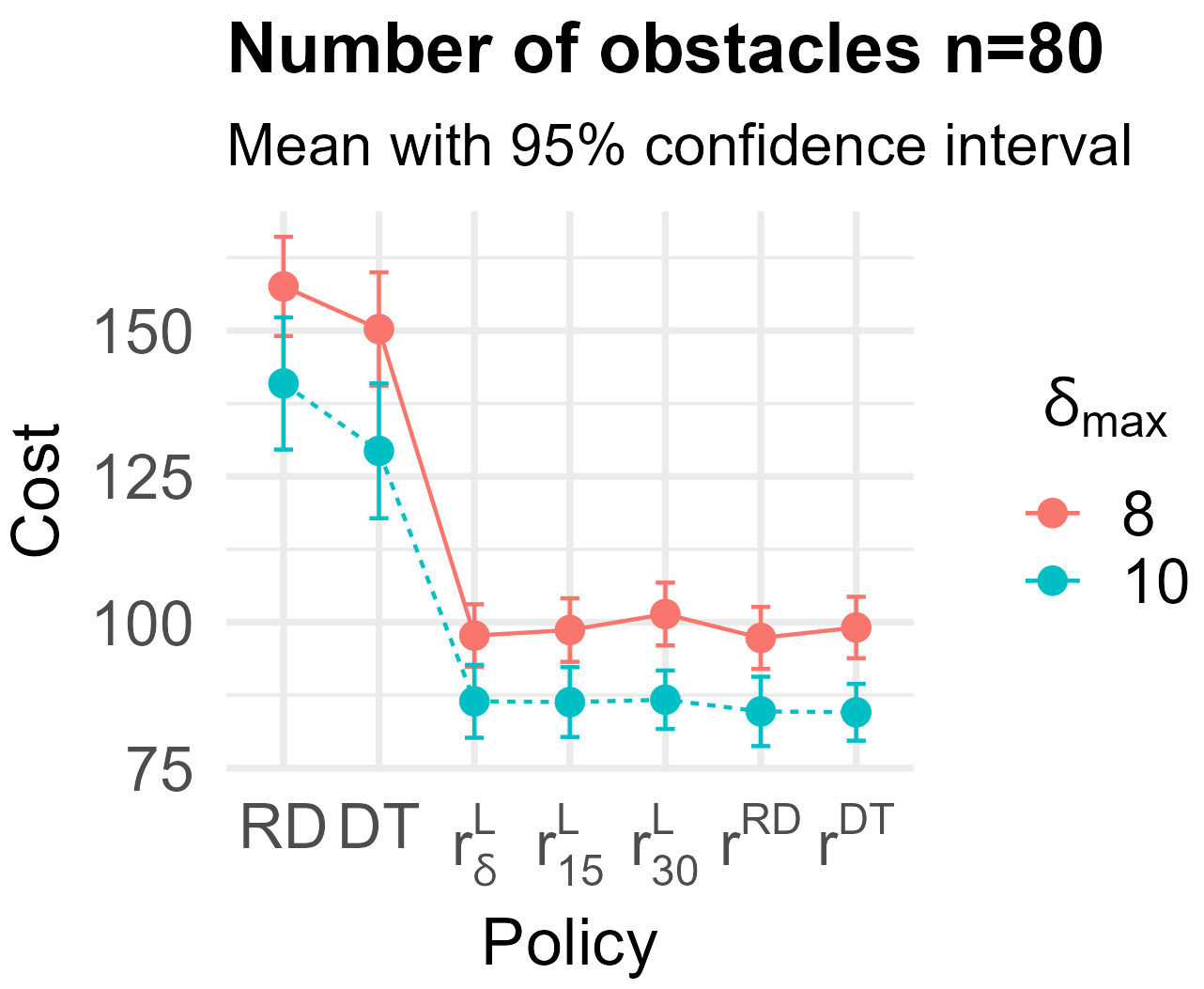}
		\end{subfigure}
	\end{tabular}
	\caption{Average traversal cost with $95\%$ confidence interval across two greedy policies and five RCDP policies using different risk functions.}
	\label{sup-fig:avg with CI}
\end{figure}

The RCDP policies consistently outperform the constrained RD and DT baselines. 
This improvement stems from the RCDP's global consideration of the disambiguation budget, 
in contrast to the myopic cost-minimization strategy of greedy methods.

As obstacle density increases or disambiguation budgets become tighter, 
the performance gap widens. 
RCDP policies exhibit both lower average cost and smaller confidence intervals, 
indicating greater robustness. 
In contrast, greedy policies sometimes exceed 
even the length of the longest obstacle-free path (150 units), a shortcoming not observed with RCDP.

These findings are further substantiated in Table~\ref{sup-tab:comparison_sevenalgs2}, 
which records the average number of obstacles intersected (in the simplified scenario) and 
disambiguation cost used (in the general scenario). 
Greedy methods often breach the constraint, while RCDP policies remain feasible and efficient.

Figure~\ref{sup-fig:sd} presents the standard deviations of traversal costs. 
Again, RCDP policies display smaller variability, 
reinforcing their consistency across heterogeneous environments.

\begin{figure}[H]
	\centering
	\begin{tabular}{ccc}
		\begin{subfigure}[b]{0.33\textwidth}
			\centering
			\includegraphics[width=\textwidth]{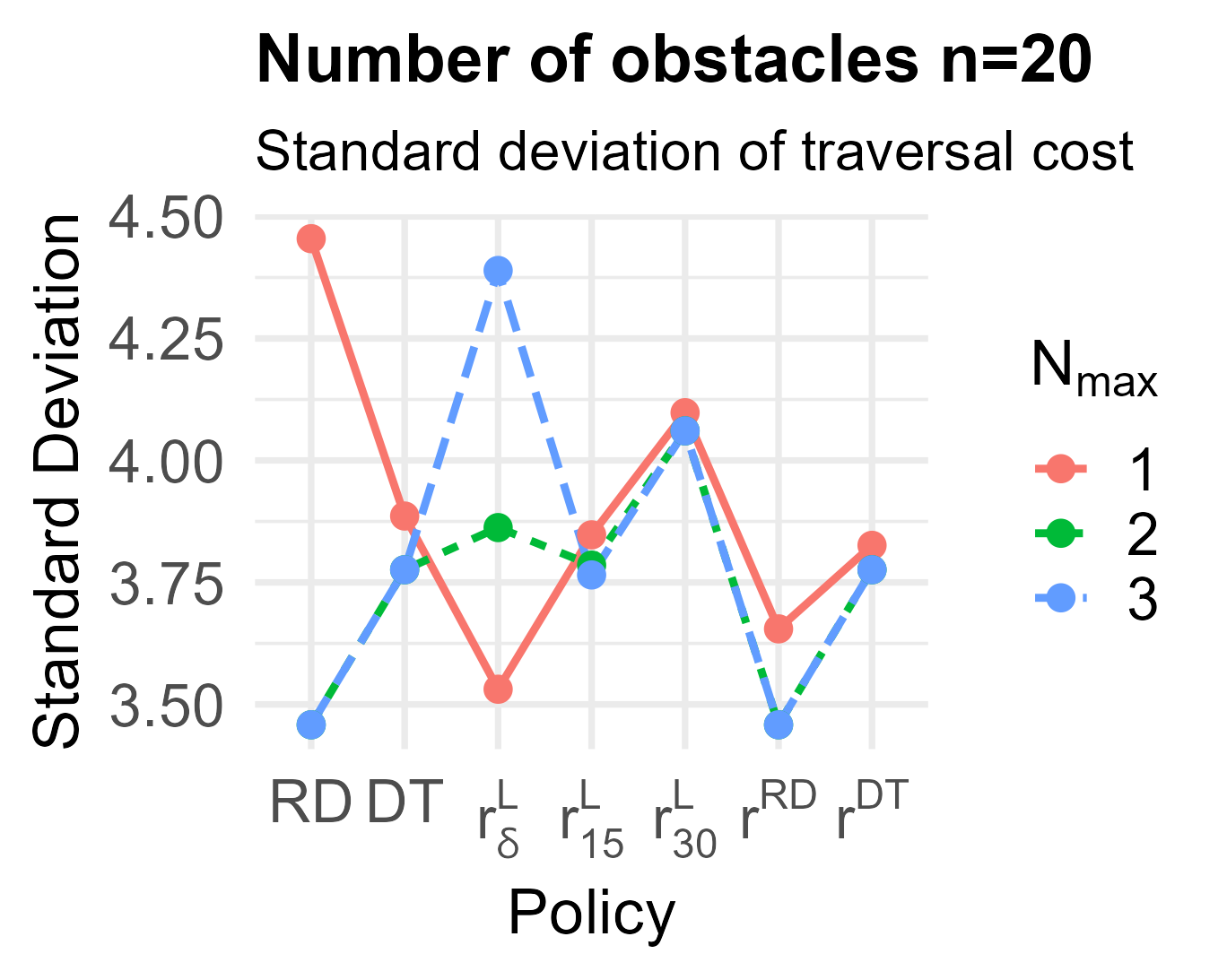}
		\end{subfigure} &
		\begin{subfigure}[b]{0.33\textwidth}
			\centering
			\includegraphics[width=\textwidth]{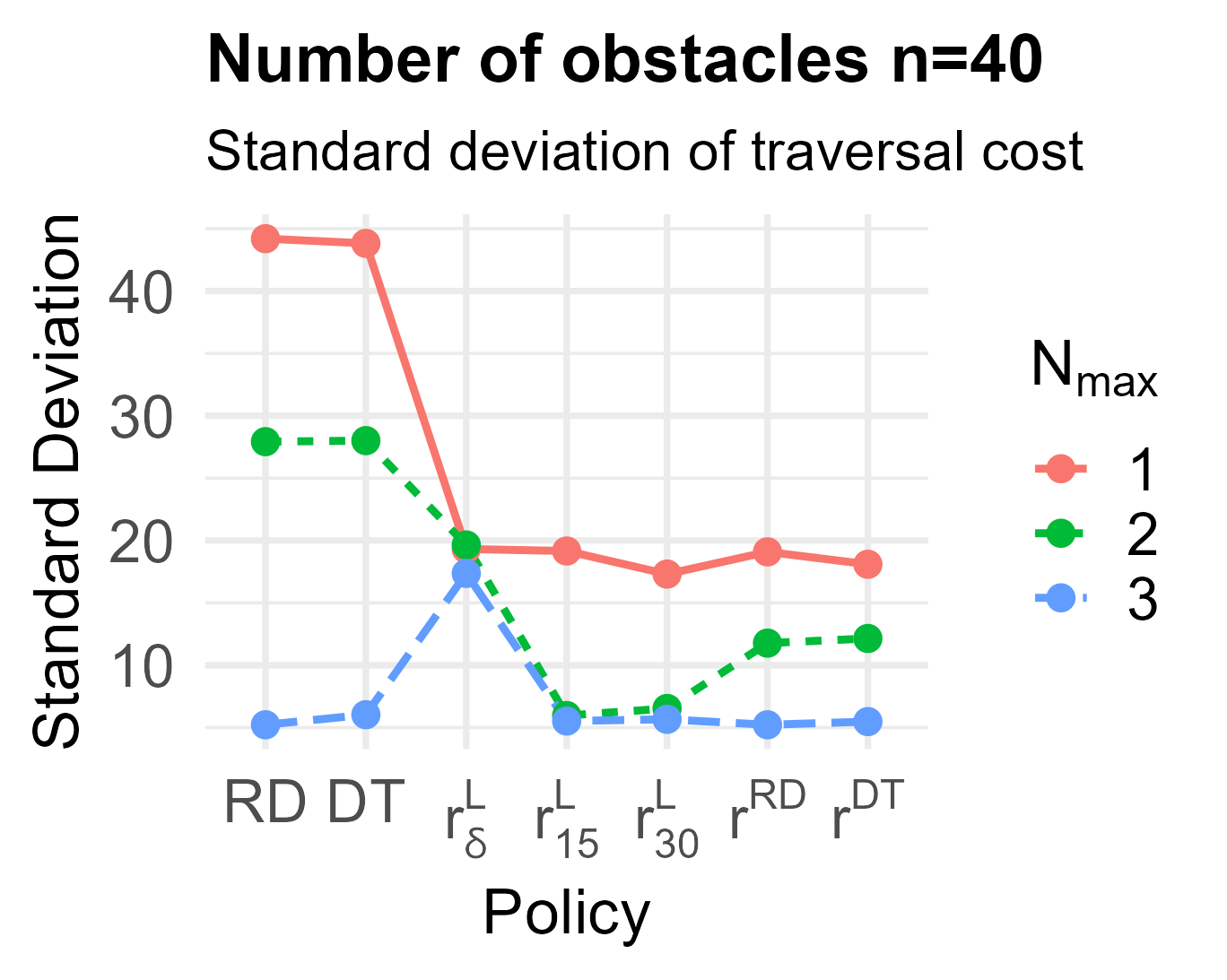}
		\end{subfigure} &
		\begin{subfigure}[b]{0.33\textwidth}
			\centering
			\includegraphics[width=\textwidth]{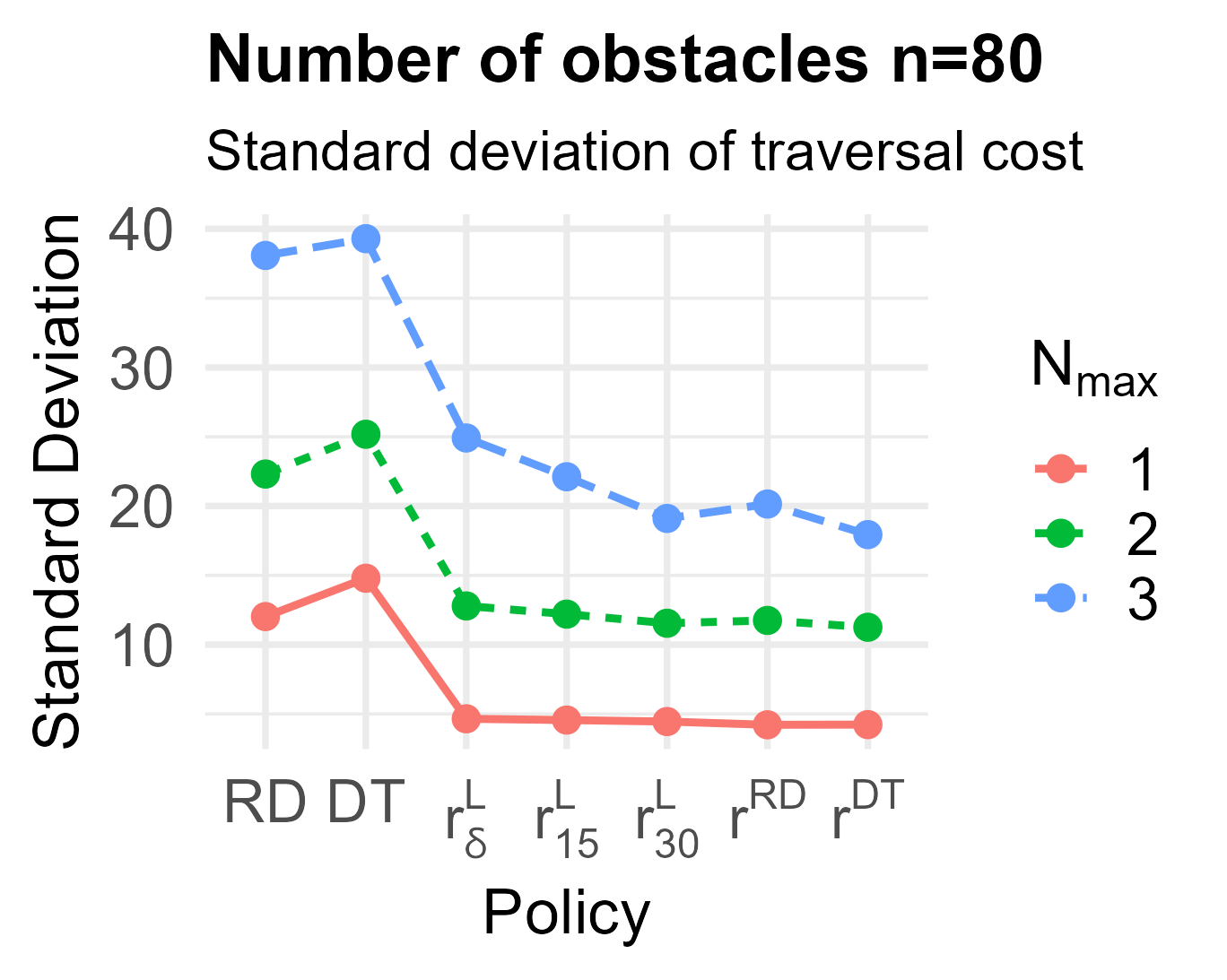}
		\end{subfigure} \\
		\begin{subfigure}[b]{0.33\textwidth}
			\centering
			\includegraphics[width=\textwidth]{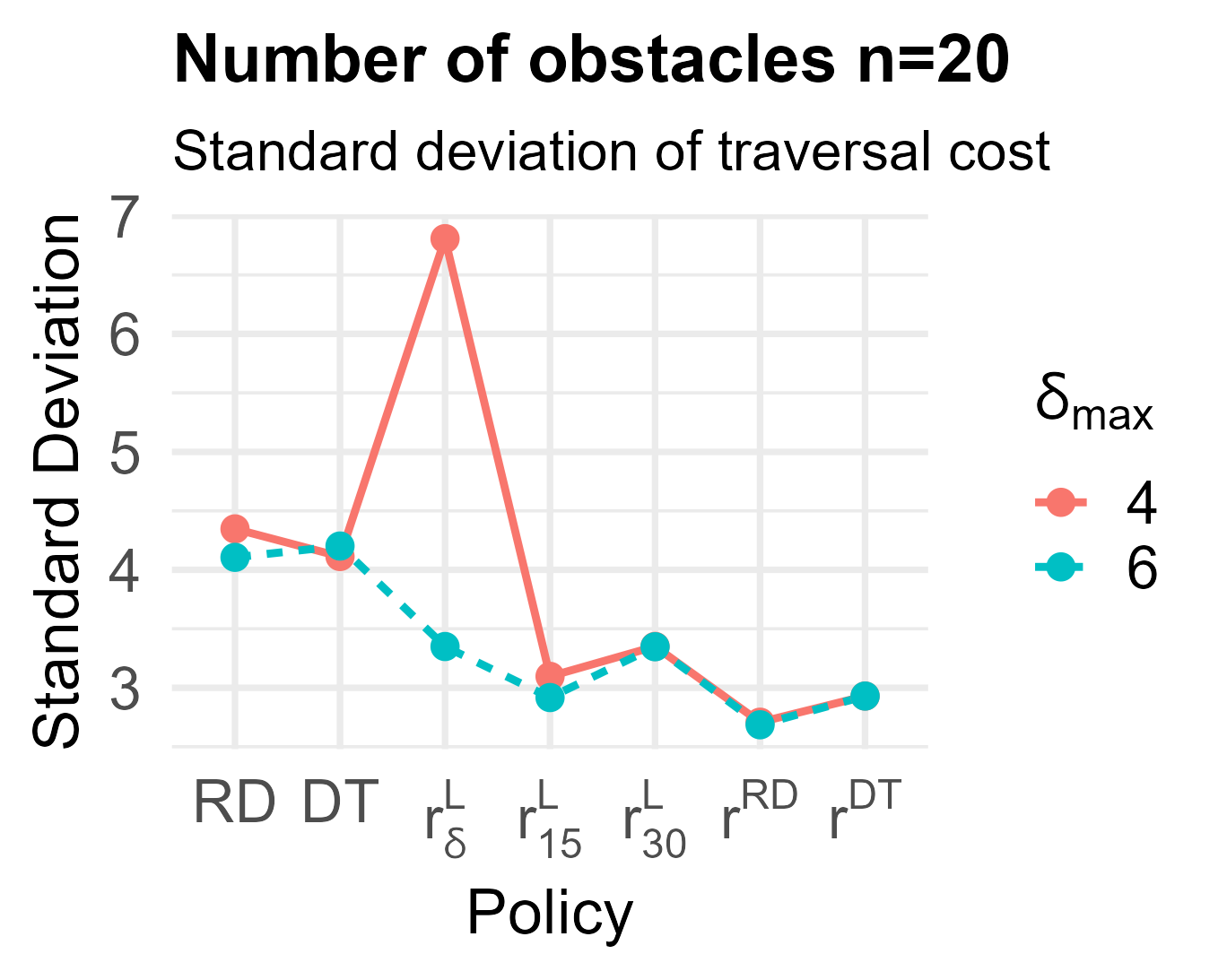}
		\end{subfigure} &
		\begin{subfigure}[b]{0.33\textwidth}
			\centering
			\includegraphics[width=\textwidth]{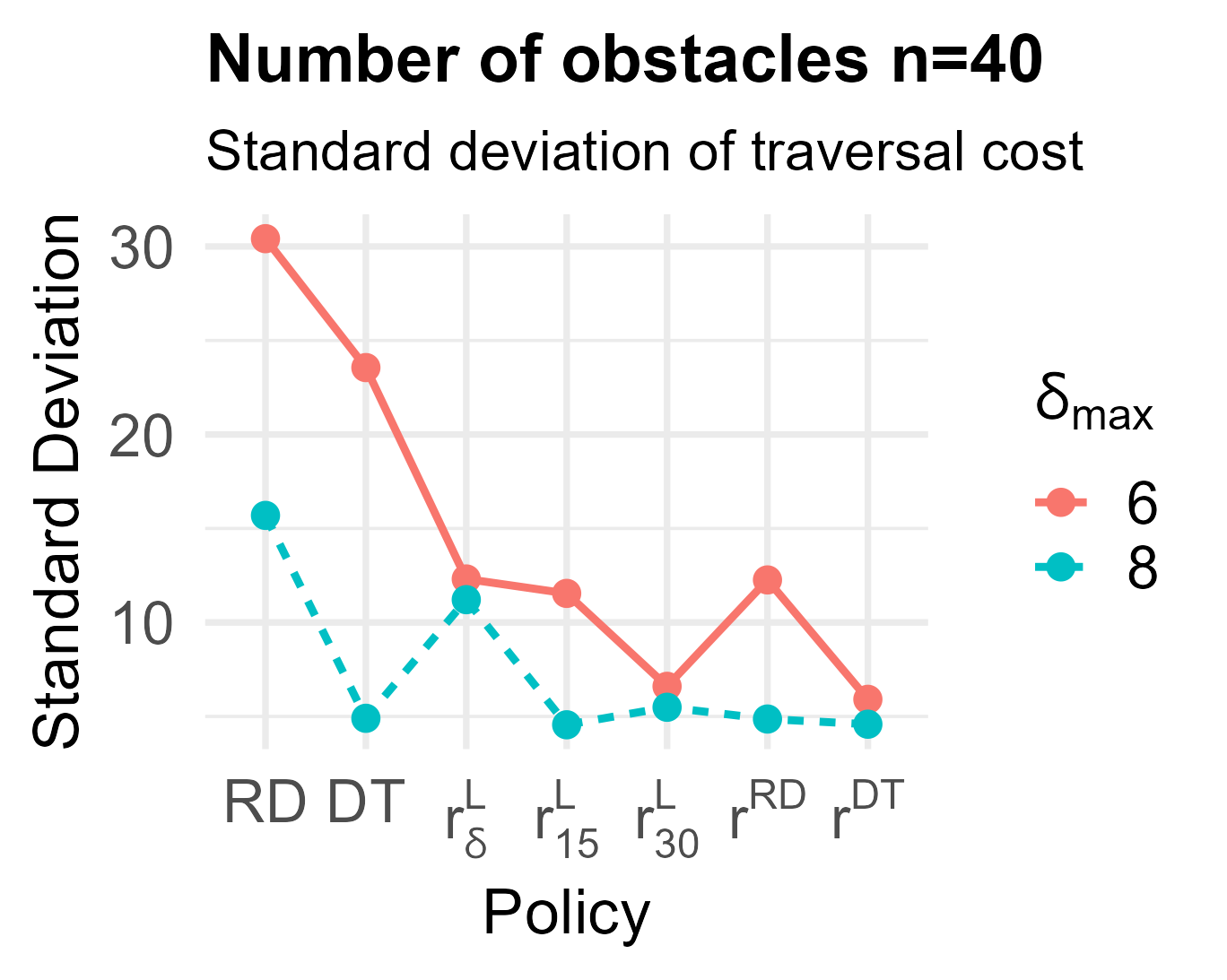}
		\end{subfigure} &
		\begin{subfigure}[b]{0.33\textwidth}
			\centering
			\includegraphics[width=\textwidth]{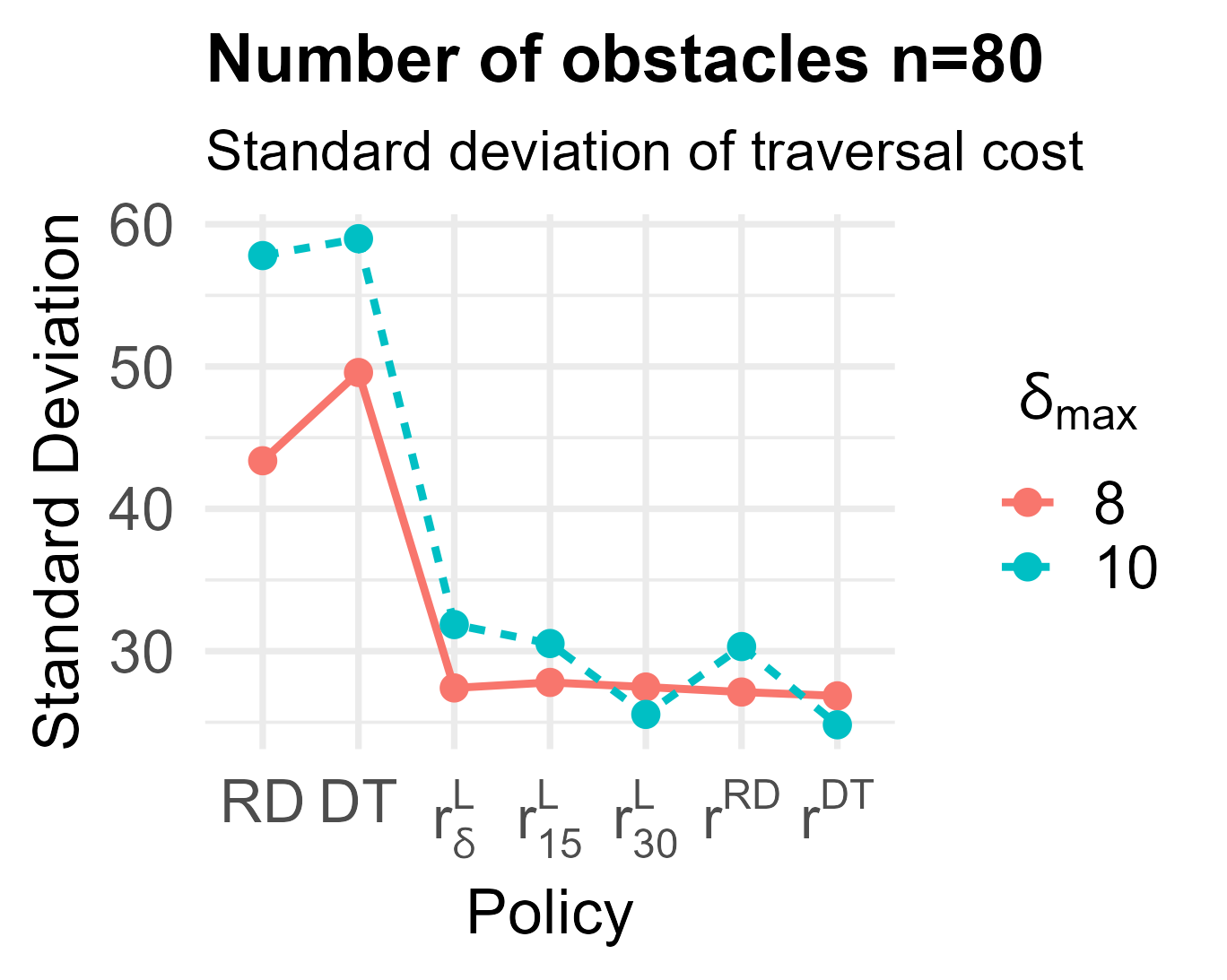}
		\end{subfigure}
	\end{tabular}
	\caption{Standard deviation of traversal costs across policies and environments. 
		Lower variance for RCDP policies indicates higher stability under uncertainty.}
	\label{sup-fig:sd}
\end{figure}

Figure~\ref{sup-fig:error} compares policies using the error metric, 
which captures the deviation of traversal cost from the offline optimal. 
As obstacle density increases and disambiguation budgets tighten, 
error magnitudes rise—reflecting the greater complexity of navigating such environments. 
Consistent with other performance metrics, 
RCDP policies achieve lower errors than greedy alternatives, 
especially under challenging conditions, 
highlighting their superior adaptivity and planning efficiency.

\begin{figure}[H]
	\centering
	\begin{tabular}{ccc}
		\begin{subfigure}[b]{0.33\textwidth}
			\centering
			\includegraphics[width=\textwidth]{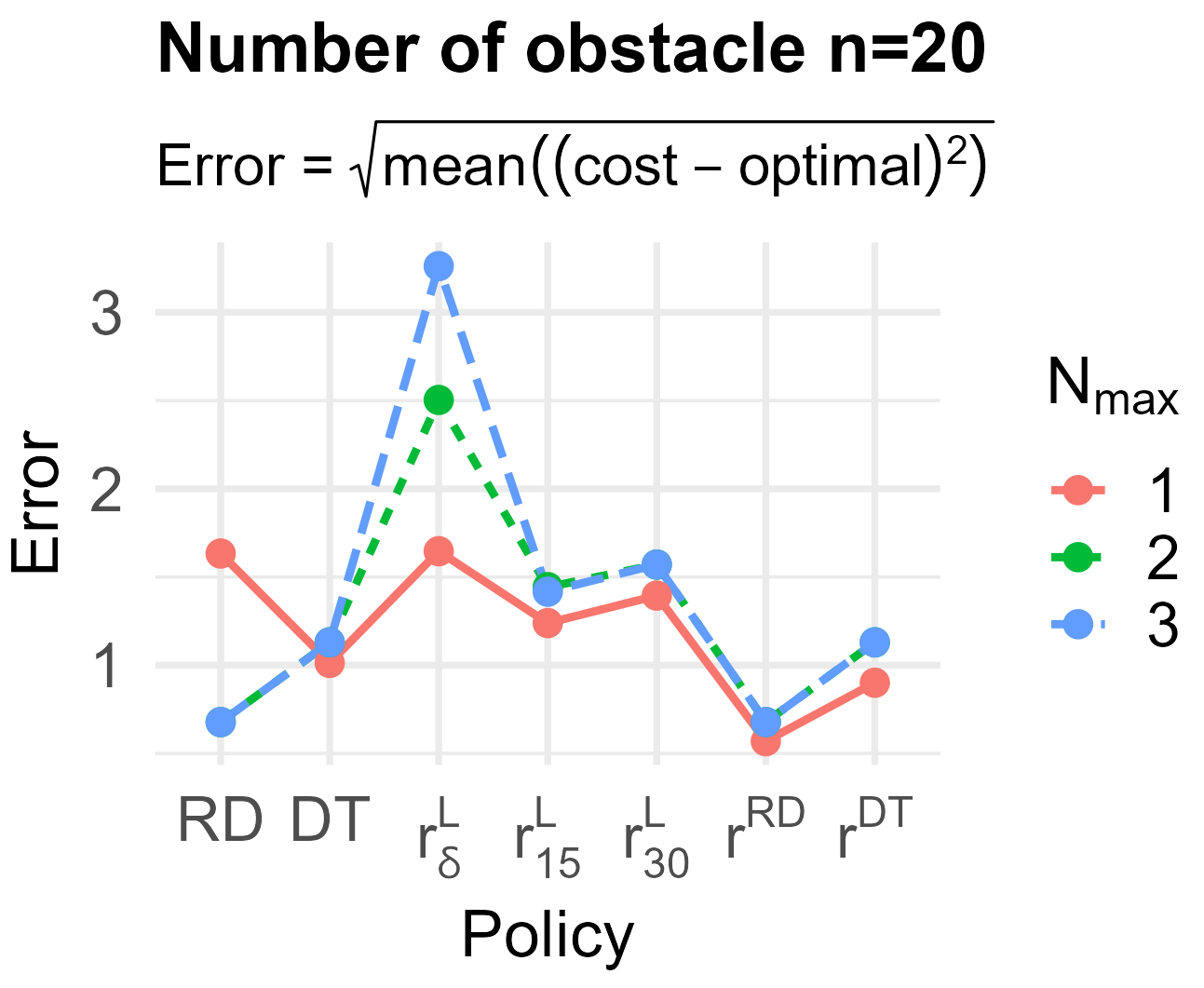}
		\end{subfigure} &
		\begin{subfigure}[b]{0.33\textwidth}
			\centering
			\includegraphics[width=\textwidth]{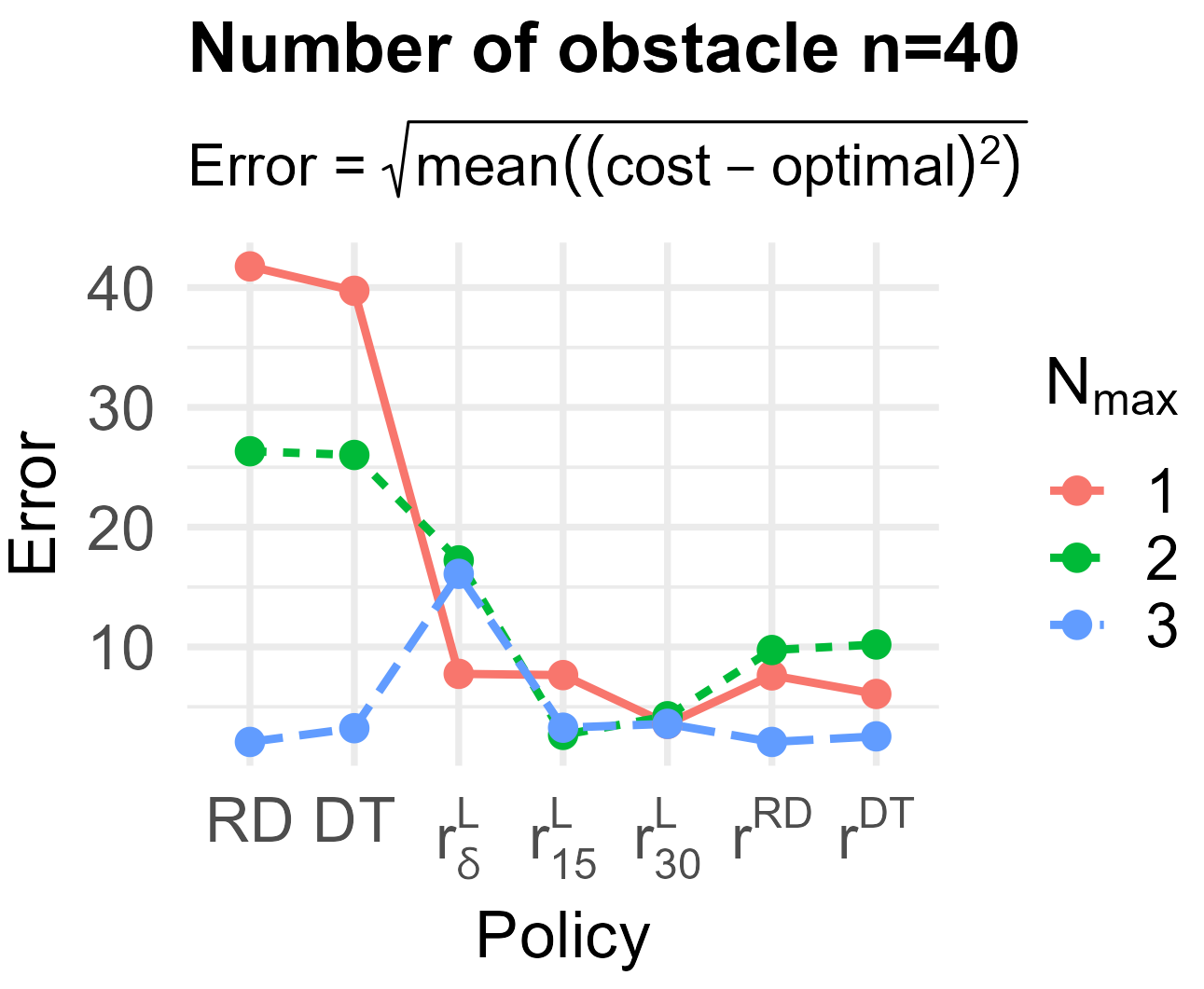}
		\end{subfigure} &
		\begin{subfigure}[b]{0.33\textwidth}
			\centering
			\includegraphics[width=\textwidth]{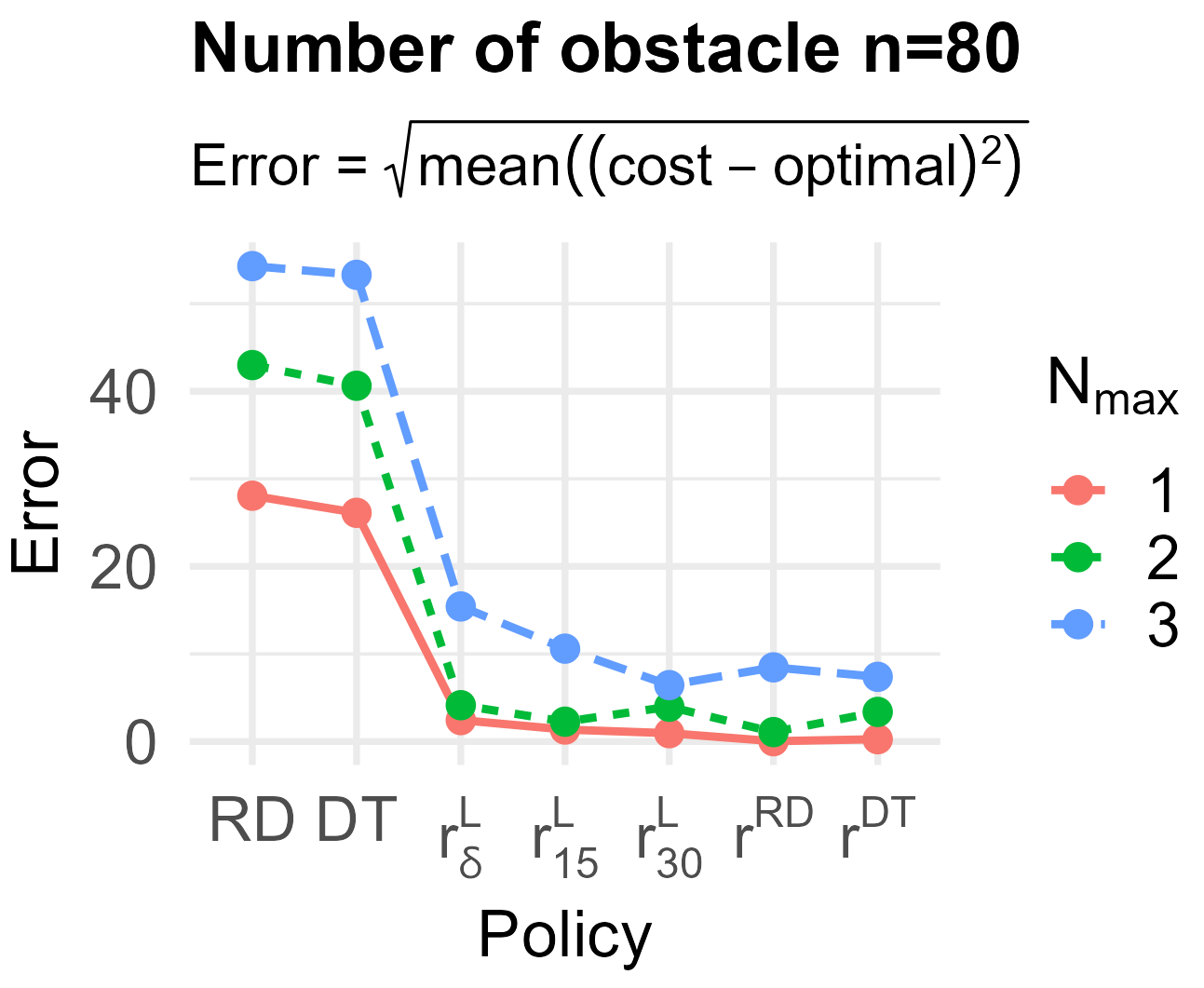}
		\end{subfigure} \\
		\begin{subfigure}[b]{0.33\textwidth}
			\centering
			\includegraphics[width=\textwidth]{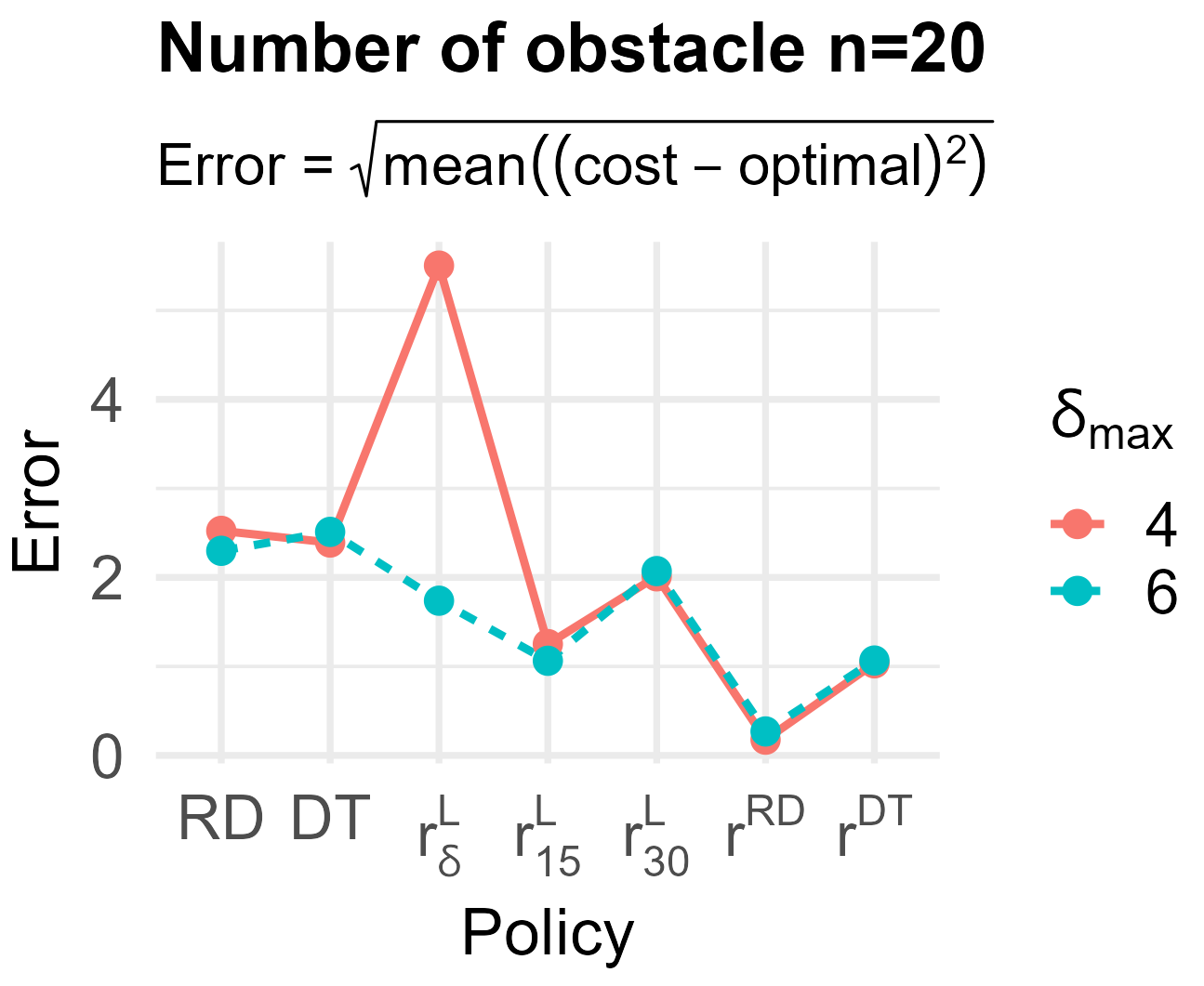}
		\end{subfigure} &
		\begin{subfigure}[b]{0.33\textwidth}
			\centering
			\includegraphics[width=\textwidth]{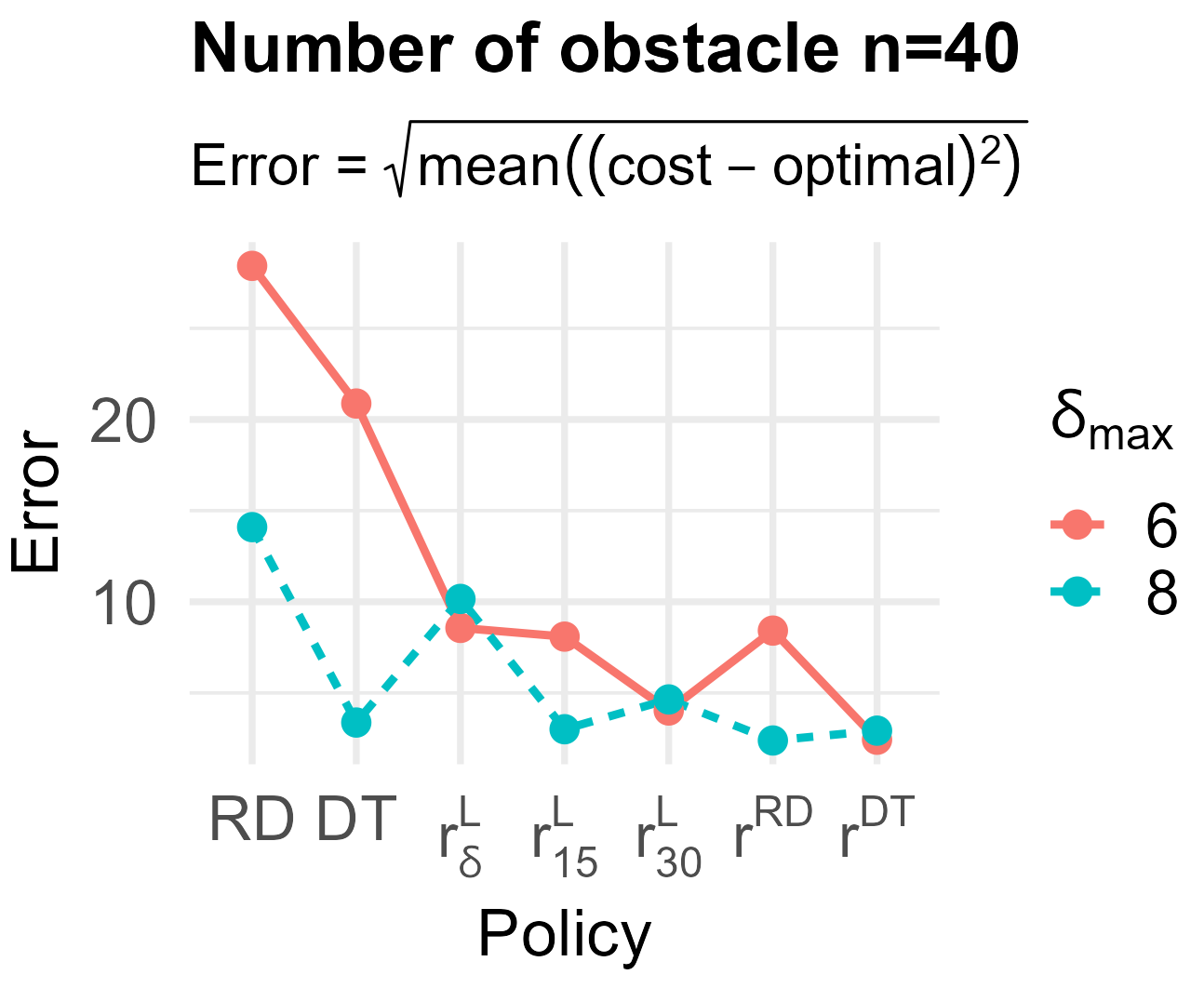}
		\end{subfigure} &
		\begin{subfigure}[b]{0.33\textwidth}
			\centering
			\includegraphics[width=\textwidth]{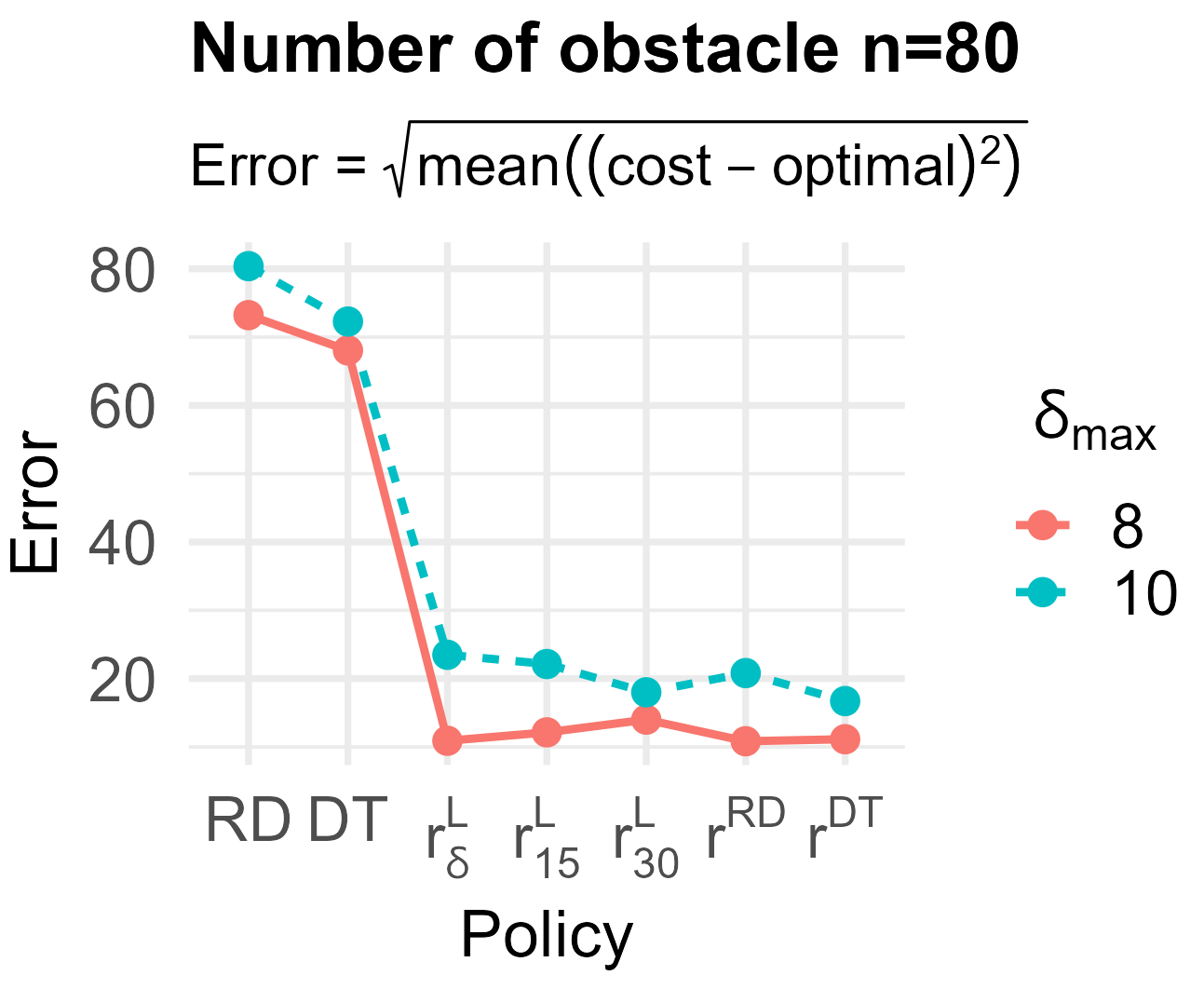}
		\end{subfigure}
	\end{tabular}
	\caption{The deviation of two greedy policies, five RCDP policies with different risk functions 
		from optimal solutions}
	\label{sup-fig:error}
\end{figure}

Figures \ref{sup-fig:quantile25} and \ref{sup-fig:quantile75} show the $25^{th}$ and $75^{th}$ percentile 
traversal cost across environments, respectively,
capturing each policy's performance in relatively favorable and challenging conditions.
Results regarding $25^{th}$ percentile indicate that in less dense environments, policies show 
similar performance,
but distinctions in policies' performance become more obvious as obstacle density increases.
Especially in environments with $n = 80$, greedy policies tend to show much higher variability in 
the $25^{th}$ percentile compared to RCDP policies.
This suggests that although greedy policies may find efficient paths, their lack of strategic 
disambiguation planning results in greater sensitivity to environment conditions.
When we compare policies regarding $75^{th}$ percentile, the benefits of applying RCDP policies 
compared to greedy policies become more evident.
Greedy policies tend to exhibit higher $75^{th}$ percentile traversal costs than RCDP policies, 
particularly in obstacle-dense environments ($n = 80$).
RCDP policies, in contrast, demonstrate better and more robust performance in worst scenarios, 
due to the systematic integration of risk into their entire decision-making process.

\begin{figure}[H]
	\centering
	\begin{tabular}{ccc}
		\begin{subfigure}[b]{0.33\textwidth}
			\centering
			\includegraphics[width=\textwidth]{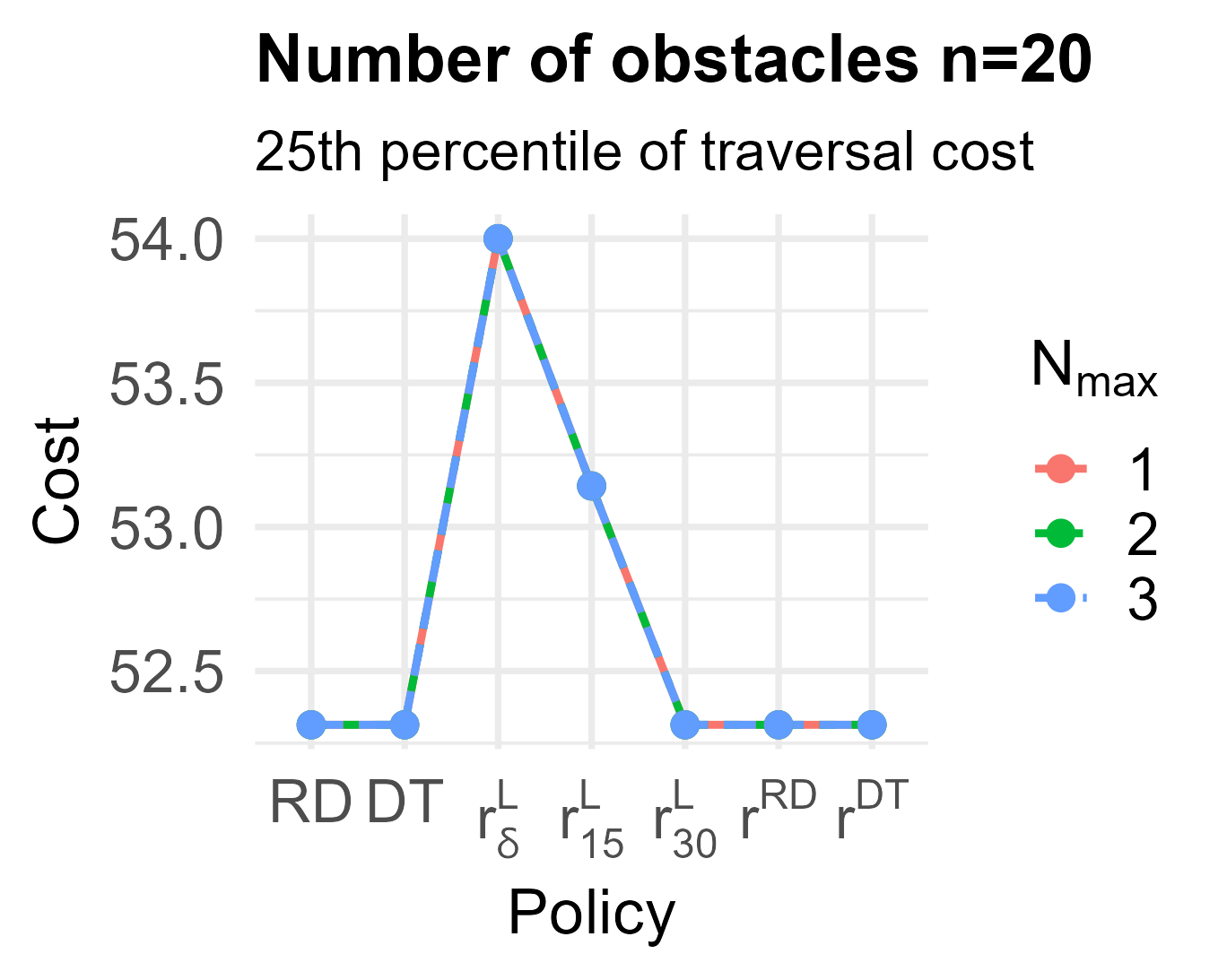}
		\end{subfigure} &
		\begin{subfigure}[b]{0.33\textwidth}
			\centering
			\includegraphics[width=\textwidth]{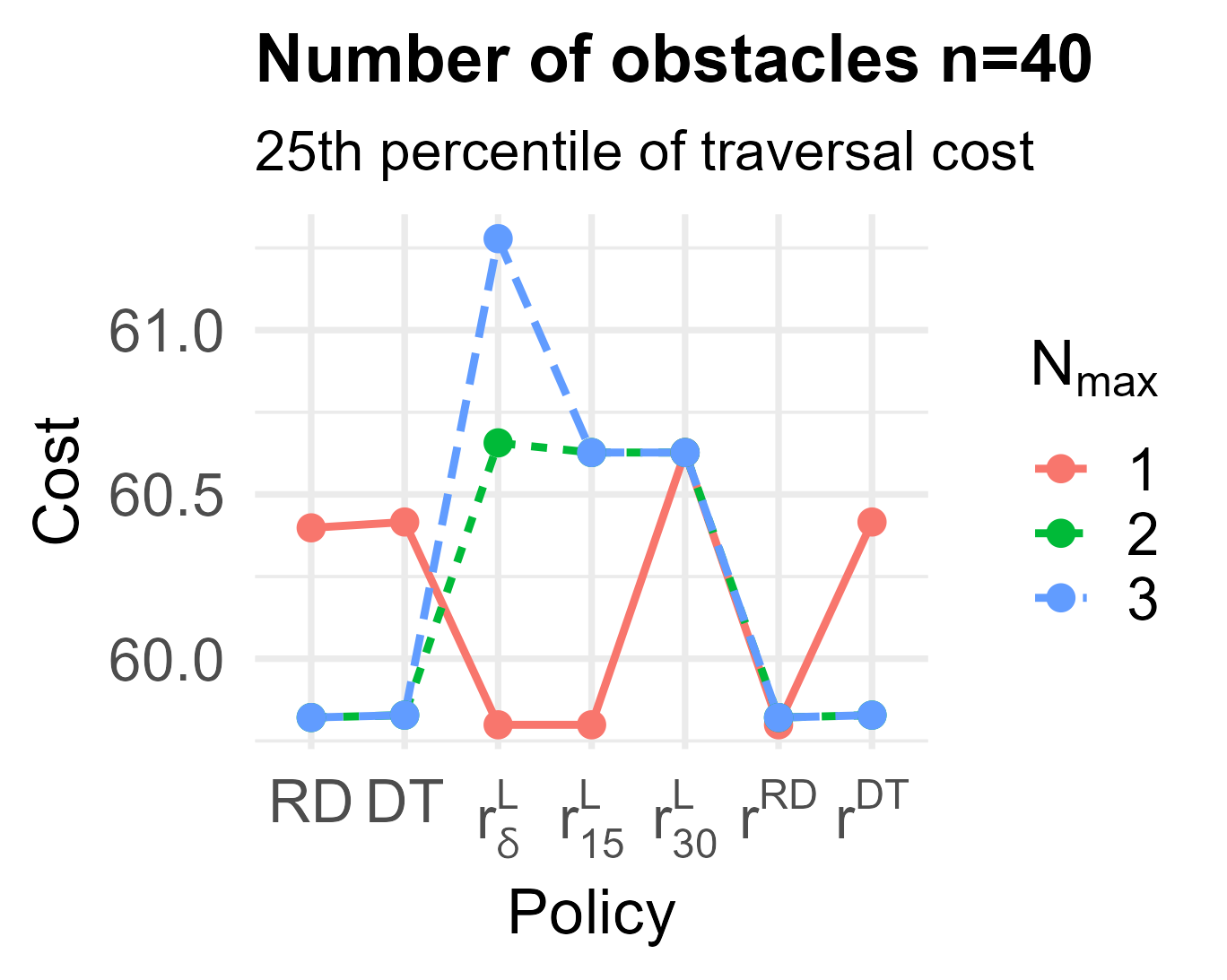}
		\end{subfigure} &
		\begin{subfigure}[b]{0.33\textwidth}
			\centering
			\includegraphics[width=\textwidth]{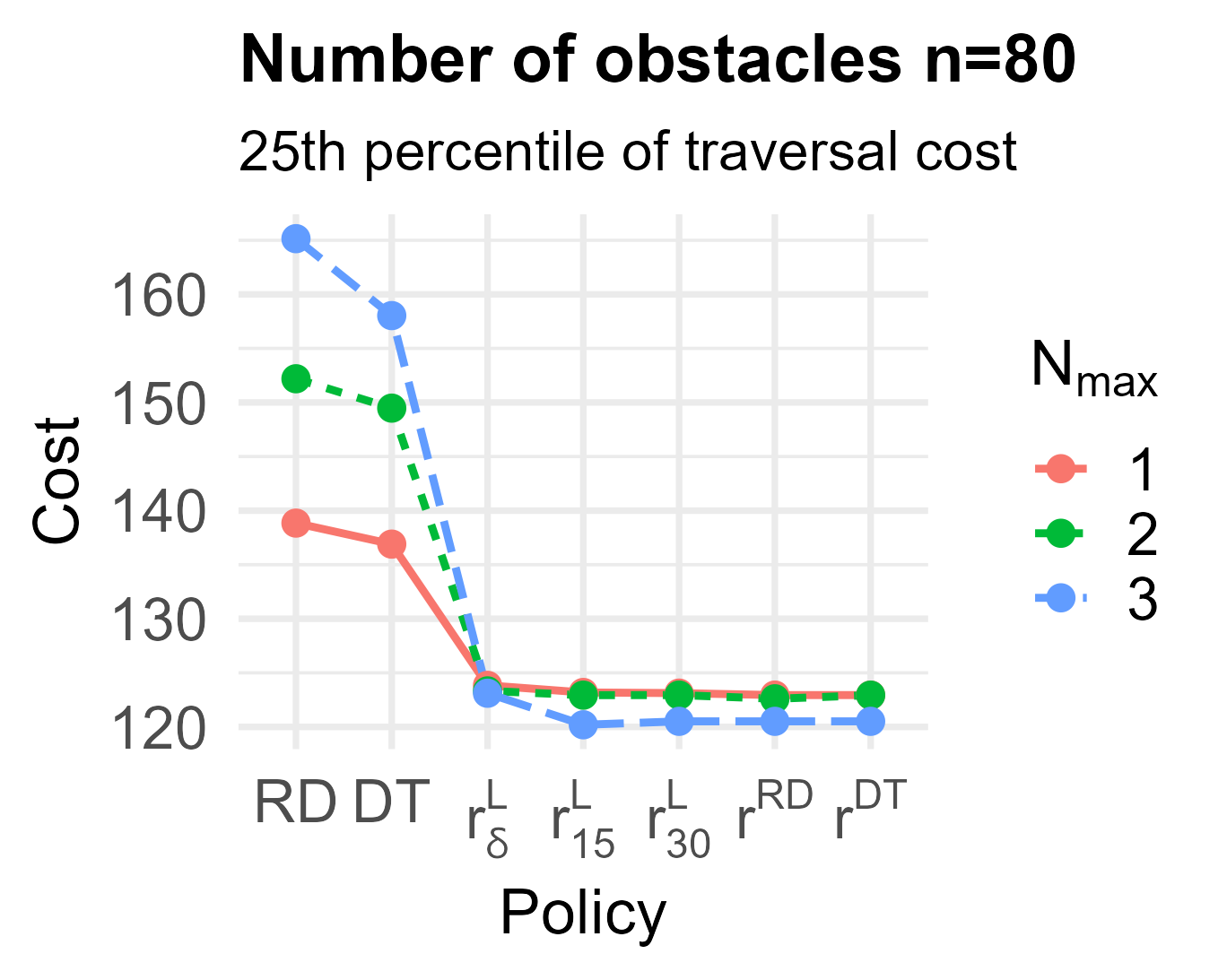}
		\end{subfigure} \\
		\begin{subfigure}[b]{0.33\textwidth}
			\centering
			\includegraphics[width=\textwidth]{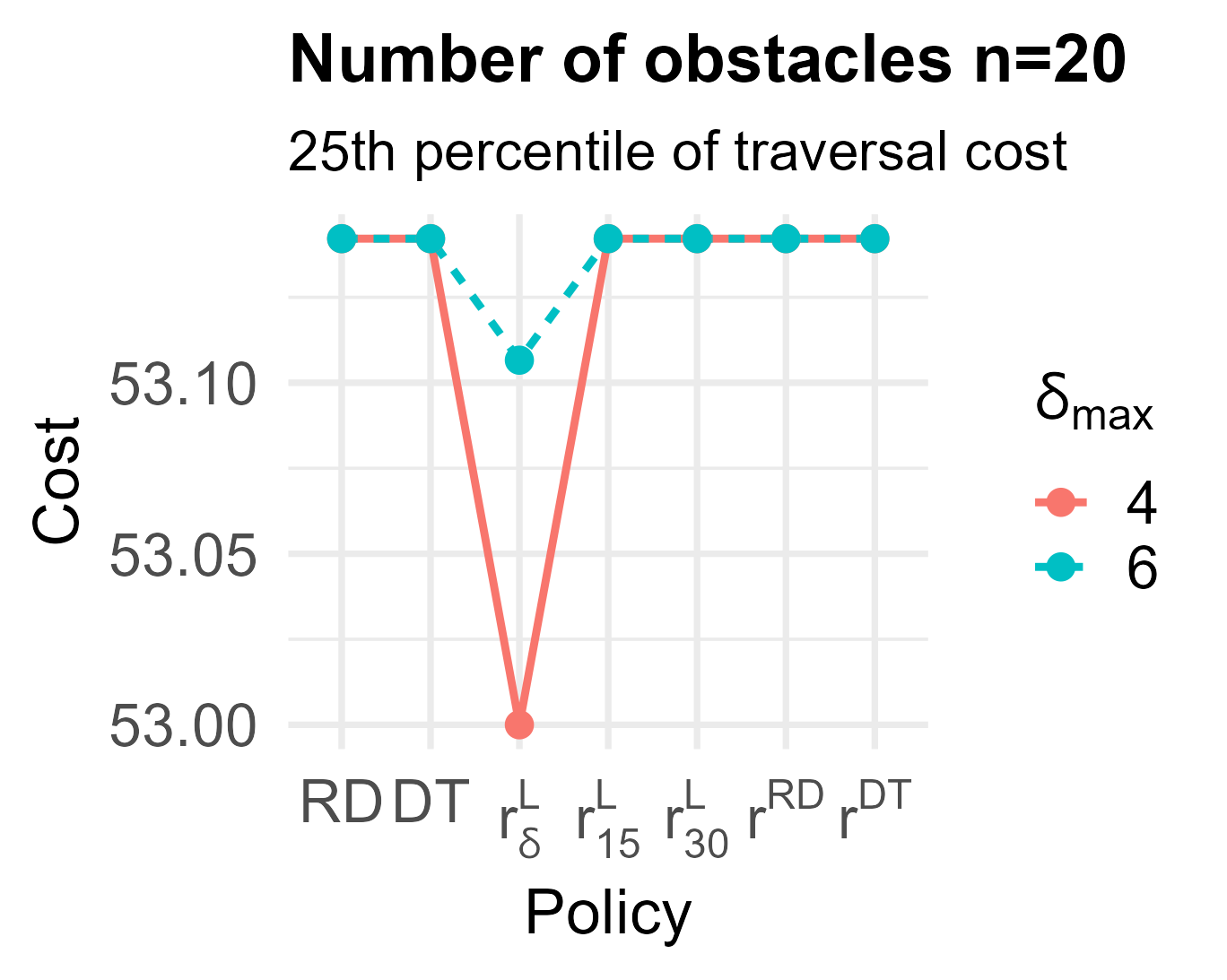}
		\end{subfigure} &
		\begin{subfigure}[b]{0.33\textwidth}
			\centering
			\includegraphics[width=\textwidth]{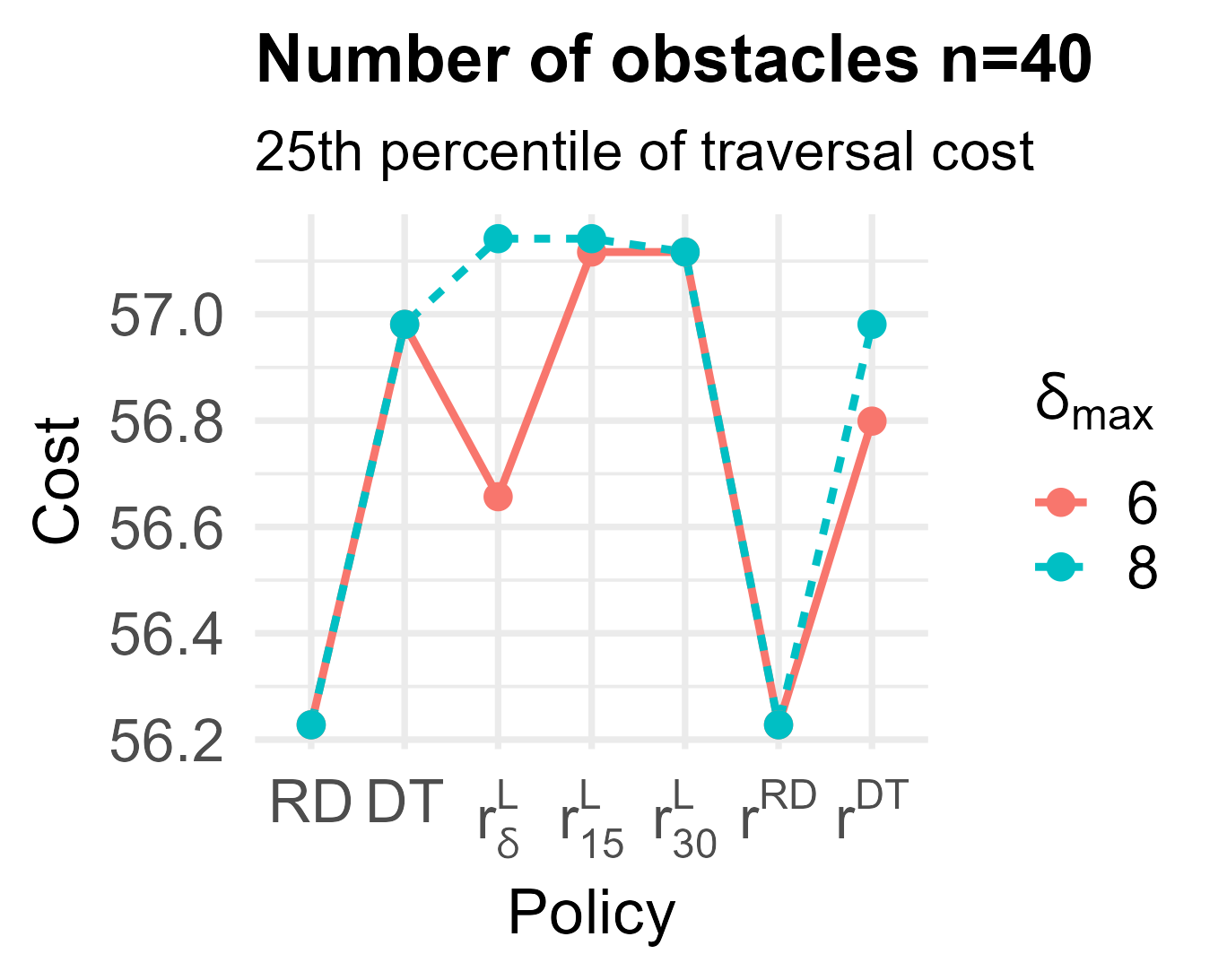}
		\end{subfigure} &
		\begin{subfigure}[b]{0.33\textwidth}
			\centering
			\includegraphics[width=\textwidth]{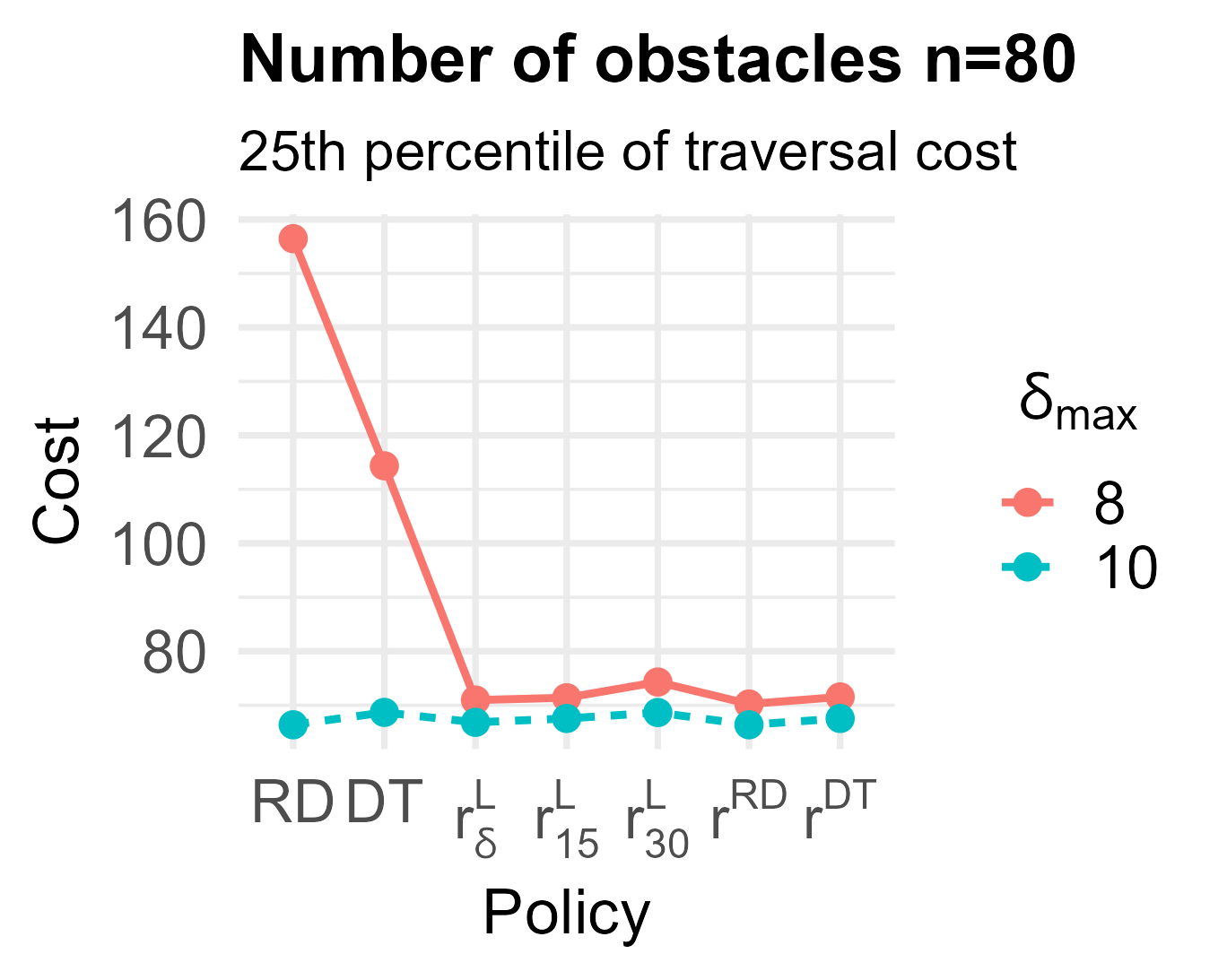}
		\end{subfigure}
	\end{tabular}
	\caption{The $25^{th}$ percentile of traversal costs across different environments using two 
		greedy policies, five RCDP policies with different risk functions}
	\label{sup-fig:quantile25}
\end{figure}

\begin{figure}[H]
	\centering
	\begin{tabular}{ccc}
		\begin{subfigure}[b]{0.33\textwidth}
			\centering
			\includegraphics[width=\textwidth]{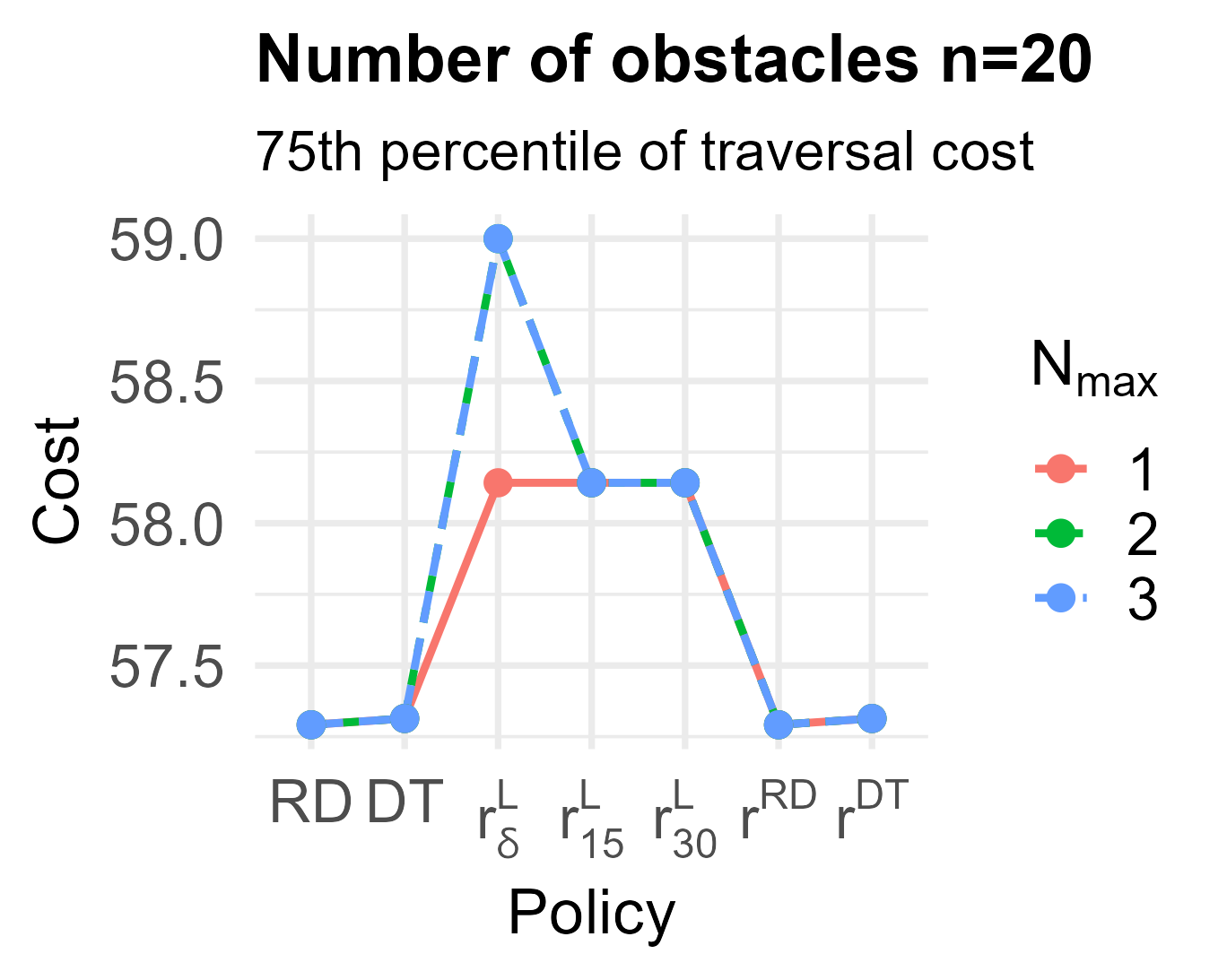}
		\end{subfigure} &
		\begin{subfigure}[b]{0.33\textwidth}
			\centering
			\includegraphics[width=\textwidth]{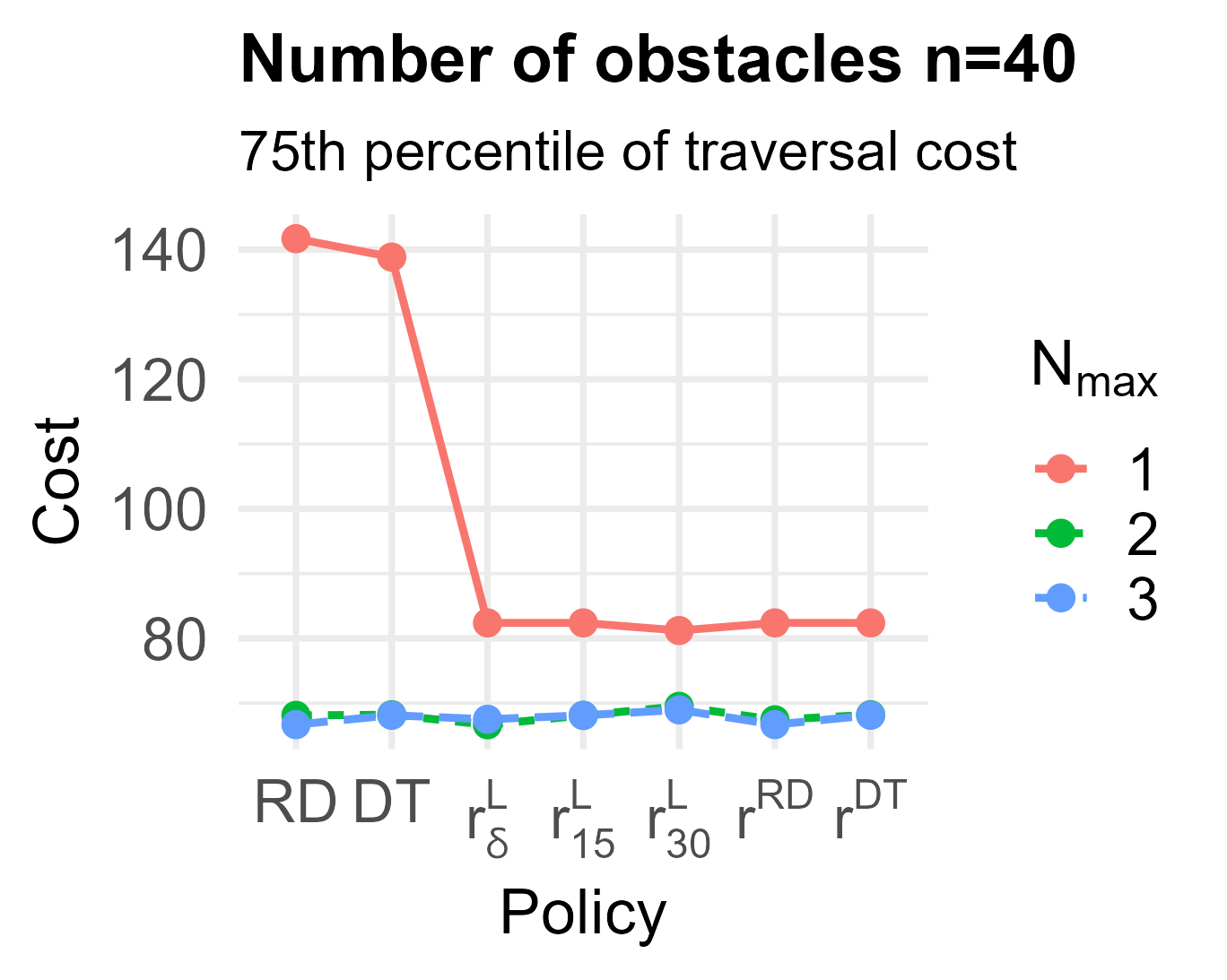}
		\end{subfigure} &
		\begin{subfigure}[b]{0.33\textwidth}
			\centering
			\includegraphics[width=\textwidth]{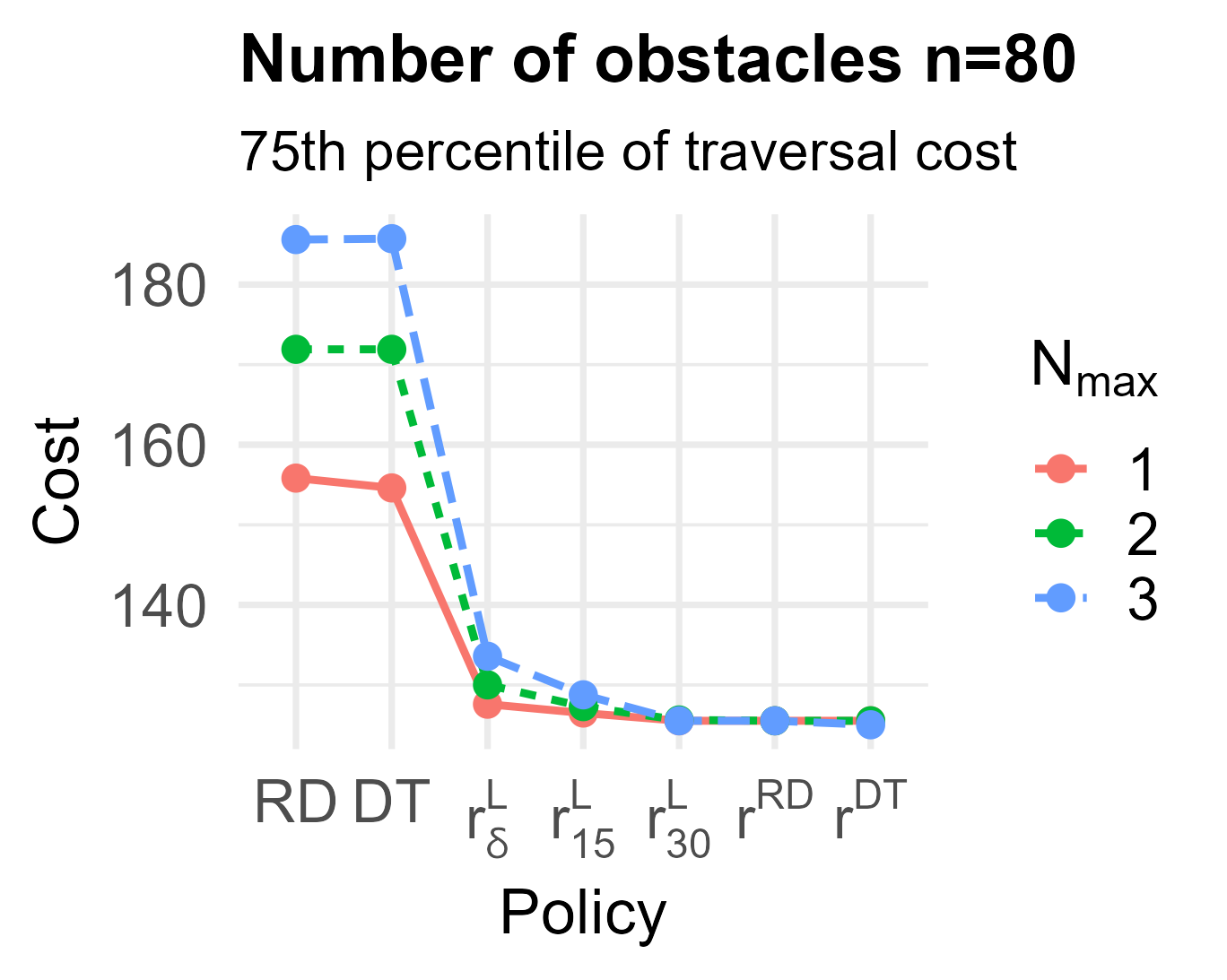}
		\end{subfigure} \\
		\begin{subfigure}[b]{0.33\textwidth}
			\centering
			\includegraphics[width=\textwidth]{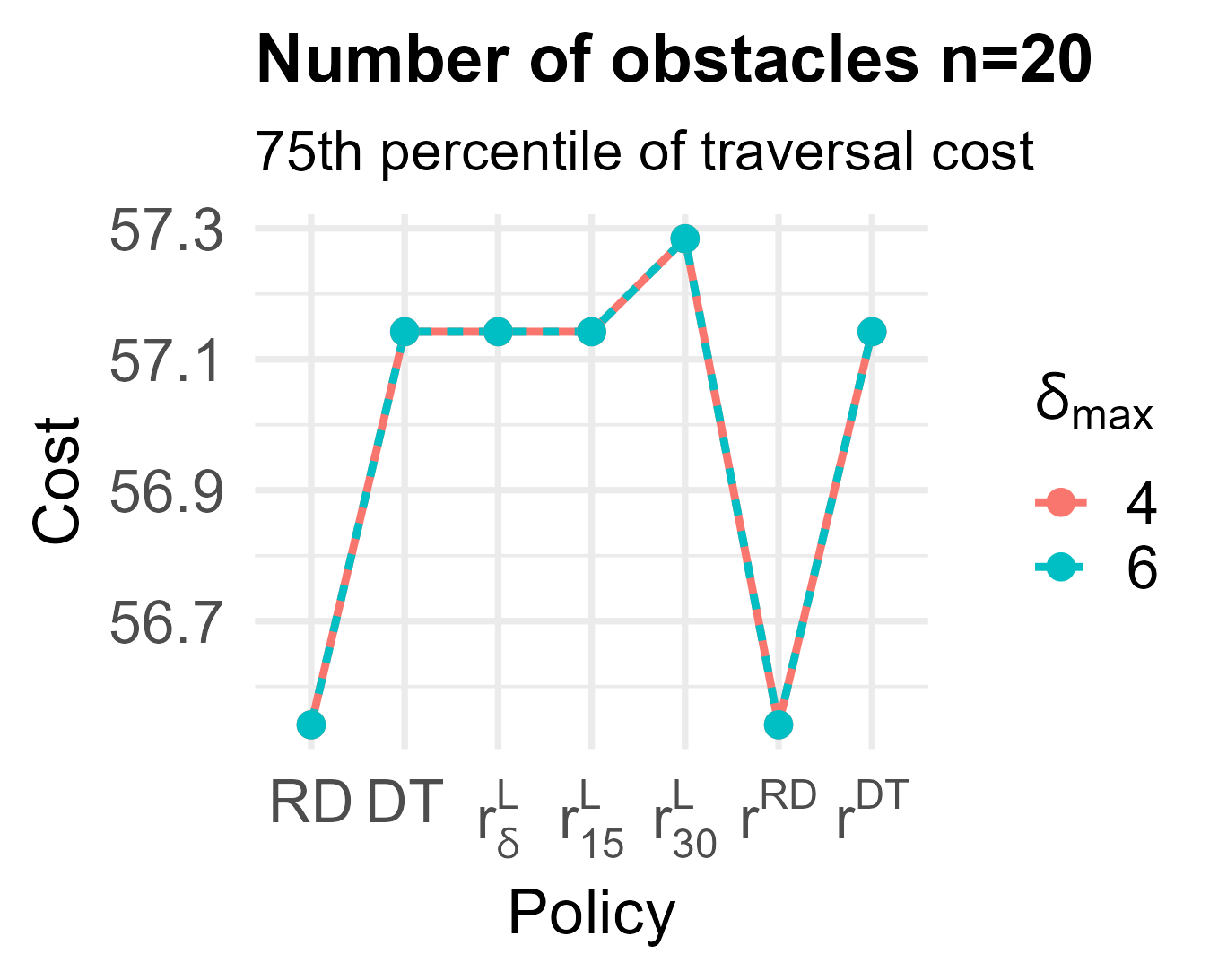}
		\end{subfigure} &
		\begin{subfigure}[b]{0.33\textwidth}
			\centering
			\includegraphics[width=\textwidth]{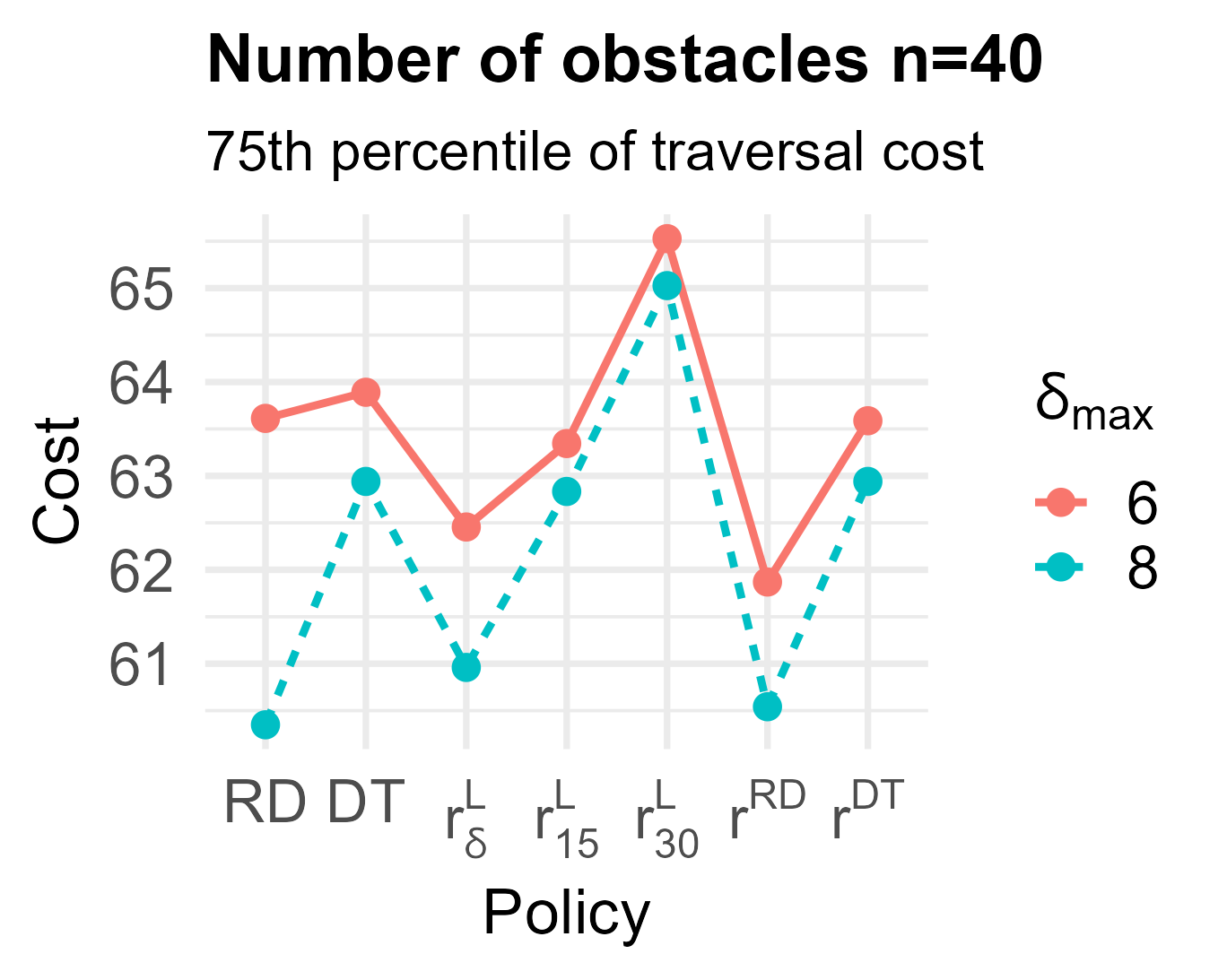}
		\end{subfigure} &
		\begin{subfigure}[b]{0.33\textwidth}
			\centering
			\includegraphics[width=\textwidth]{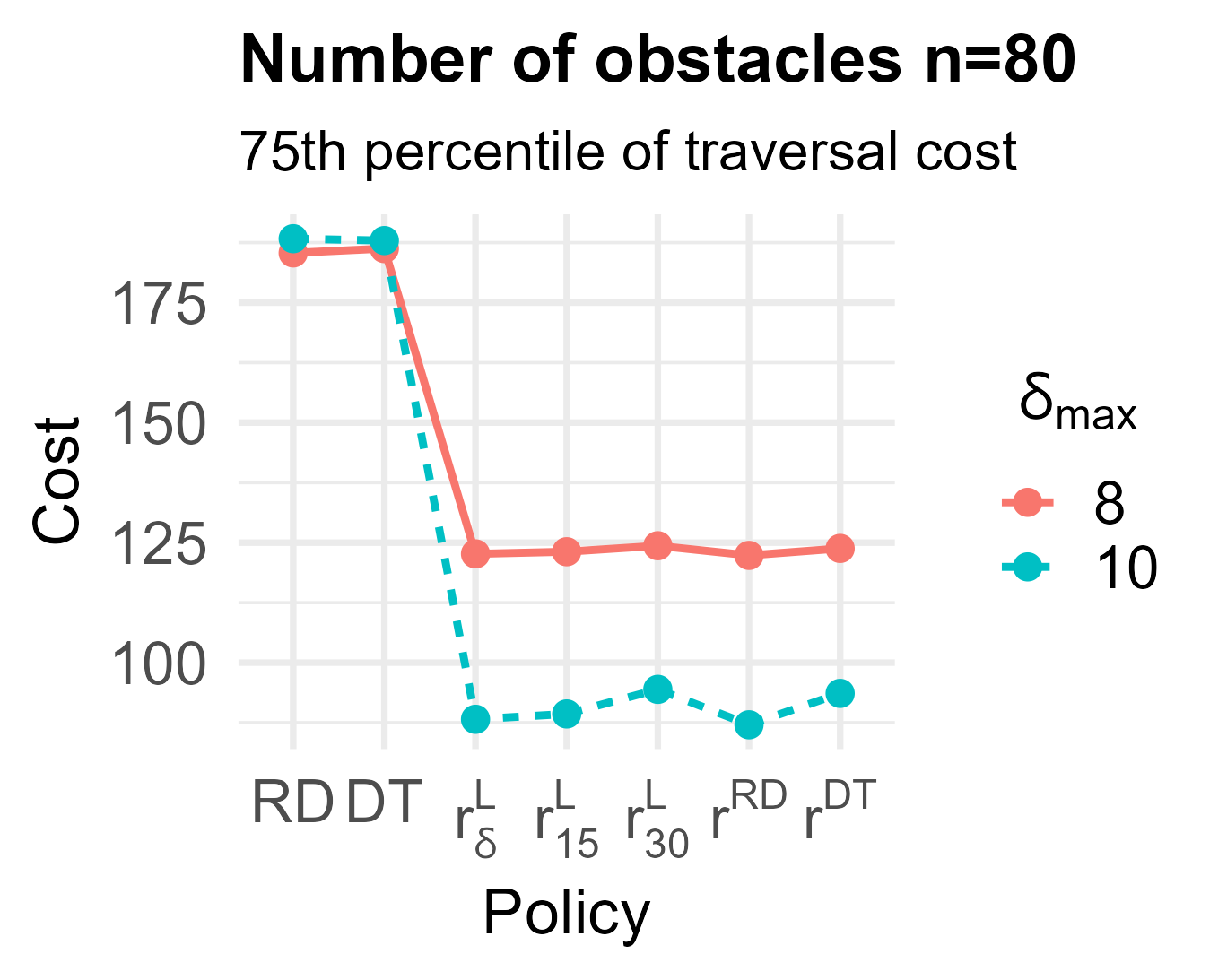}
		\end{subfigure}
	\end{tabular}
	\caption{The $75^{th}$ percentile of traversal costs across different environments using two 
		greedy policies, five RCDP policies with different risk functions}
	\label{sup-fig:quantile75}
\end{figure}

\noindent
In summary, RCDP policies consistently deliver more effective traversal performance than their greedy counterparts, particularly in complex and uncertain environments. Their advantage stems from a principled use of risk-aware planning under budget constraints, resulting in solutions that are not only cost-efficient but also stable across varying obstacle densities and risk profiles.

\begin{sidewaystable}
	\centering
	\resizebox{\textwidth}{!}{%
		\begin{tabular}{c}
			\begin{tabular}{|cc|cccc|ccccccccccc|}
				\hline
				\multicolumn{2}{|c|}{\multirow{2}{*}{Setting}} & \multicolumn{4}{c|}{Greedy Policies} & \multicolumn{10}{c}{RCDP Policies} & \\ \cline{3-17}
				\multicolumn{2}{|c|}{} & \multicolumn{2}{c}{RD} & \multicolumn{2}{c|}{DT} & \multicolumn{2}{c}{$r^L_\delta$} & \multicolumn{2}{c}{$r^L_{15}$} & \multicolumn{2}{c}{$r^L_{30}$} & \multicolumn{2}{c}{$r^{RD}$} & \multicolumn{2}{c|}{$r^{DT}$} & Benchmark (Optimal) \\ \cline{1-2}
				$N_{\max}$ & $n$ & $\bar{\mathcal{C}}_p$ & $SD(\mathcal{C}_p)$ & $\bar{\mathcal{C}}_p$ & $SD(\mathcal{C}_p)$ & $\bar{\mathcal{C}}_p$ & $SD(\mathcal{C}_p)$ & $\bar{\mathcal{C}}_p$ & $SD(\mathcal{C}_p)$ & $\bar{\mathcal{C}}_p$ & $SD(\mathcal{C}_p)$ & $\bar{\mathcal{C}}_p$ & $SD(\mathcal{C}_p)$ & $\bar{\mathcal{C}}_p$ & \multicolumn{1}{c|}{$SD(\mathcal{C}_p)$} & $\bar{\mathcal{C}}_p$ \\ \hline
				\multirow{3}{*}{1} & 20 & 55.47 & 4.45 & 55.44 & 3.89 & 56.10 & 3.53 & 55.81 & 3.85 & 55.73 & 4.10 & \textbf{55.30} & \textbf{3.65} & 55.42 & \multicolumn{1}{c|}{3.82} & 55.14 \\
				& 40 & 99.12 & 44.20 & 95.42 & 43.82 & 73.17 & 19.32 & 73.09 & 19.16 & \textbf{72.95} & \textbf{17.32} & 72.99 & 19.09 & 73.11 & \multicolumn{1}{c|}{18.09} & 71.73 \\
				& 80 & 148.69 & 12.03 & 145.94 & 14.79 & 125.31 & 4.65 & 124.29 & 4.56 & 123.97 & 4.44 & \textbf{123.62} & \textbf{4.22} & 123.68 & \multicolumn{1}{c|}{4.23} & 123.62 \\ \hline
				\multirow{3}{*}{2} & 20 & 55.23 & 3.46 & 55.39 & 3.78 & 56.73 & 3.86 & 55.81 & 3.79 & 55.71 & 4.06 & \textbf{55.23} & \textbf{3.46} & 55.39 & \multicolumn{1}{c|}{3.78} & 55.05 \\
				& 40 & 71.72 & 27.92 & 70.68 & 28.01 & 67.51 & 19.64 & \textbf{64.42} & \textbf{5.97} & 64.96 & 6.54 & 64.61 & 11.77 & 65.16 & \multicolumn{1}{c|}{12.15} & 62.91 \\
				& 80 & 160.58 & 22.31 & 156.91 & 25.18 & 123.91 & 12.79 & 121.86 & 12.20 & 121.67 & 11.55 & \textbf{120.89} & \textbf{11.75} & 121.20 & \multicolumn{1}{c|}{11.26} & 120.74 \\ \hline
				\multirow{3}{*}{3} & 20 & 55.23 & \textbf{3.46} & 55.39 & 3.78 & 57.12 & 4.39 & 55.80 & 3.77 & 55.71 & 4.06 & \textbf{55.23} & \textbf{3.46} & 55.39 & \multicolumn{1}{c|}{3.78} & 55.05 \\
				& 40 & 63.39 & 5.23 & 63.92 & 6.03 & 67.21 & 17.36 & 64.55 & 5.53 & 64.69 & 5.66 & \textbf{63.39} & \textbf{5.23} & 63.74 & \multicolumn{1}{c|}{5.48} & 62.48 \\
				& 80 & 163.15 & 38.08 & 160.67 & 39.29 & 123.62 & 24.91 & 118.83 & 22.12 & 116.96 & 19.11 & 117.15 & 20.15 & \textbf{116.36} & \multicolumn{1}{c|}{\textbf{17.94}} & 114.99 \\ \hline
				\multicolumn{17}{c}{Results for $N_{\max}$ values of 1, 2, and 3}
			\end{tabular} \\
			
			\\
			\\
			
			\begin{tabular}{|cc|cccc|ccccccccccc|}
				\hline
				\multicolumn{2}{|c|}{\multirow{2}{*}{Setting}} & \multicolumn{4}{c|}{Greedy Policies} & \multicolumn{10}{c}{RCDP Policies} & \\ \cline{3-17}
				\multicolumn{2}{|c|}{} & \multicolumn{2}{c}{RD} & \multicolumn{2}{c|}{DT} & \multicolumn{2}{c}{$r^L_\delta$} & \multicolumn{2}{c}{$r^L_{15}$} & \multicolumn{2}{c}{$r^L_{30}$} & \multicolumn{2}{c}{$r^{RD}$} & \multicolumn{2}{c|}{$r^{DT}$} & Benchmark (Optimal) \\ \cline{1-2}
				$\delta_{\max}$ & $n$ & $\bar{\mathcal{C}}_p$ & $SD(\mathcal{C}_p)$ & $\bar{\mathcal{C}}_p$ & $SD(\mathcal{C}_p)$ & $\bar{\mathcal{C}}_p$ & $SD(\mathcal{C}_p)$ & $\bar{\mathcal{C}}_p$ & $SD(\mathcal{C}_p)$ & $\bar{\mathcal{C}}_p$ & $SD(\mathcal{C}_p)$ & $\bar{\mathcal{C}}_p$ & $SD(\mathcal{C}_p)$ & $\bar{\mathcal{C}}_p$ & \multicolumn{1}{c|}{$SD(\mathcal{C}_p)$} & $\bar{\mathcal{C}}_p$ \\ \hline
				\multirow{1}{*}{4} & 20 & 54.87 & 4.27 & 55.06 & 4.11 & 55.62 & 6.81 & 54.94 & 3.10 & 55.27 & 3.35 & \textbf{54.52} & \textbf{2.71} & 54.87 & \multicolumn{1}{c|}{2.93} & 54.47 \\ \hline
				\multirow{2}{*}{6} & 20 & 54.74 & 4.11 & 55.06 & 4.20 & 55.32 & 3.35 & 54.86 & 2.92 & 55.72 & 3.35 & \textbf{54.51} & \textbf{2.69} & 54.87 & \multicolumn{1}{c|}{2.93} & 54.44 \\
				& 40 & 70.54 & 29.93 & 65.44 & 23.30 & 61.42 & 12.31 & 61.66 & 11.54 & 61.95 & 6.60 & 60.99 & 12.26 & \textbf{60.82} & \multicolumn{1}{c|}{\textbf{5.91}} & 59.67 \\ \hline
				\multirow{2}{*}{8} & 40 & 61.25 & 15.45 & 60.24 & 4.33 & 60.47 & 11.21 & 60.33 & 4.56 & 61.37 & 5.49 & \textbf{59.09} & \textbf{4.86} & 60.18 & \multicolumn{1}{c|}{4.58} & 58.49 \\
				& 80 & 156.01 & 42.77 & 148.83 & 48.79 & \textbf{97.21} & \textbf{27.41} & 98.66 & 27.80 & 98.97 & 27.48 & 97.31 & 27.12 & 99.10 & \multicolumn{1}{c|}{26.86} & 95.22 \\ \hline
				\multirow{1}{*}{10} & 80 & 139.81 & 56.98 & 128.40 & 58.11 & 86.44 & 31.86 & 86.31 & 30.53 & 86.31 & 25.55 & 84.72 & 30.32 & \textbf{84.56} & \multicolumn{1}{c|}{\textbf{24.82}} & 80.51 \\ \hline
				\multicolumn{17}{c}{Results for $\delta_{\max}$ values of 4, 6, 8, and 10}
			\end{tabular}
		\end{tabular}
	}
	\caption{The average traversal costs $\bar{\mathcal{C}}_p$ and 
		standard deviation $SD(\mathcal{C}_p)$ of two greedy policies, 
		five RCDP policies with different risk functions, and the benchmark policy.}
	\label{sup-tab:comparison_sevenalgs}
\end{sidewaystable}

\begin{table}[H]
	\centering
	\begin{subtable}{\textwidth}
		\centering
		\resizebox{10cm}{!}{%
			\begin{tabular}{@{}cc|cc|cccccc@{}}
				\toprule
				\multicolumn{2}{c|}{Setting} & \multicolumn{2}{c|}{Greedy Policies} & \multicolumn{6}{c}{RCDP Policies}     \\ \midrule
				$N_{\max}$ & $n$ & RD & DT   & $r^L_\delta$ & $r^L_{15}$ & $r^L_{30}$ & $r^{RD}$ & \multicolumn{1}{c|}{$r^{DT}$} & Benchmark (Optimal) \\ \midrule
				\multirow{3}{*}{1} & 20   & 0.26  & 0.20 & 0.75 & 0.54 & 0.35 & 0.24  & \multicolumn{1}{c|}{0.19}  & 0.32 \\
				& 40   & 1.31  & 1.26 & 0.93 & 0.86 & 0.78 & 0.77  & \multicolumn{1}{c|}{0.72}  & 0.82 \\
				& 80   & 2.00  & 1.87 & 0.73 & 0.38 & 0.18 & 0.07  & \multicolumn{1}{c|}{0.03}  & 0.09 \\ \midrule
				\multirow{3}{*}{2} & 20   & 0.26  & 0.20 & 1.09 & 0.62 & 0.38 & 0.26  & \multicolumn{1}{c|}{0.20}  & 0.34 \\
				& 40   & 1.42  & 1.35 & 1.70 & 1.57 & 1.37 & 1.30  & \multicolumn{1}{c|}{1.26}  & 1.33 \\
				& 80   & 2.96  & 2.75 & 1.60 & 0.78 & 0.37 & 0.24  & \multicolumn{1}{c|}{0.15}  & 0.32 \\ \midrule
				\multirow{3}{*}{3} & 20   & 0.26  & 0.20 & 1.21 & 0.63 & 0.28 & 0.26  & \multicolumn{1}{c|}{0.20}  & 0.34 \\
				& 40   & 1.42  & 1.36 & 2.25 & 1.89 & 1.53 & 1.42  & \multicolumn{1}{c|}{1.34}  & 1.43 \\
				& 80   & 3.82  & 3.54 & 2.52 & 1.47 & 0.80 & 0.74  & \multicolumn{1}{c|}{0.54}  & 0.70 \\ \bottomrule
			\end{tabular}%
		}
		\caption{Results for $N_{\max}$ values of 1, 2 and 3}
	\end{subtable}
	
	\vspace{0.5cm} 
	
	\begin{subtable}{\textwidth}
		\centering
		\resizebox{10cm}{!}{%
			\begin{tabular}{@{}cc|cc|cccccc@{}}
				\toprule
				\multicolumn{2}{c|}{Setting} & \multicolumn{2}{c|}{Greedy Policies} & \multicolumn{6}{c}{RCDP Policies}     \\ \midrule
				$\delta_{\max}$ & $n$ & RD & DT   & $r^L_\delta$ & $r^L_{15}$ & $r^L_{30}$ & $r^{RD}$ & \multicolumn{1}{c|}{$r^{DT}$} & Benchmark (Optimal) \\ \midrule
				\multirow{1}{*}{4} & 20   & 1.14  & 1.16 & 2.92 & 1.66 & 1.04 & 1.10  & \multicolumn{1}{c|}{0.84}  & 1.54 \\ \midrule
				\multirow{2}{*}{6} & 20   & 1.16  & 0.84 & 3.38 & 1.80 & 1.12 & 1.14  & \multicolumn{1}{c|}{0.85}  & 1.66 \\
				& 40   & 4.38  & 3.74 & 5.04 & 3.78 & 3.58 & 3.72  & \multicolumn{1}{c|}{3.30}  & 4.04 \\ \midrule
				\multirow{2}{*}{8} & 40   & 4.60  & 3.62 & 6.46 & 4.54 & 3.78 & 4.50  & \multicolumn{1}{c|}{3.60}  & 5.04 \\
				& 80   & 8.82  & 8.72 & 7.8 & 5.46 & 4.26 & 5.78  & \multicolumn{1}{c|}{4.42}  & 6.50 \\ \midrule
				\multirow{1}{*}{10} & 80   & 9.56  & 9.38 & 9.72 & 8.72 & 7.40 & 8.92  & \multicolumn{1}{c|}{8.78}  & 9.00 \\ \bottomrule
			\end{tabular}%
		}
		\caption{Results for $\delta_{\max}$ values of 4, 6, 8 and 10}
	\end{subtable}
	
	\caption{The average number of intersected obstacles $\bar{N_p}$ and 
		the average resource needed $\bar{\delta_p}$ on the traversal path $p$ generated two greedy policies, 
		five RCDP policies with different risk functions, and the benchmark policy}
	\label{sup-tab:comparison_sevenalgs2}
\end{table}
RCDP policies also intersect fewer true obstacles and require lower disambiguation costs on average. 
Greedy policies often violate budget constraints, particularly in dense environments,
as evident from consistently higher average $\bar{N}_p$ and $\bar{\delta}_p$ values.

\subsection{Efficacy of the Bayesian Linear Undesirability Function} 
\label{sup-sec:Linear-Experiments}

To further assess the efficacy of the Bayesian linear undesirability function, 
we evaluate its performance relative to other risk models 
using simulations across varying obstacle densities, 
proportions of true obstacles, and sensor accuracies. 
We focus on how performance varies with:
(i) the number of obstacles $n$ and true obstacle ratio $\rho_T$,
(ii) the precision of sensor-derived probabilities $\pi_x$, and
(iii) the underlying spatial distribution of obstacles.

\subsubsection{Sensitivity to the Proportion of True Obstacles}
\label{sup-sec:Linear-rho}

Using the simulation setup described in Section~\ref{sec:MC-simulations},
we analyze traversal performance across 
$\rho_T \in \{0, 0.1, 0.2, 0.4, 0.8, 1\}$ and for $n = 20, 40, 80$ obstacles. 
We set $\alpha_{\max}=60$ in the Bayesian risk function.

Figure~\ref{sup-fig_compare_newrisk} shows the average traversal cost under six risk functions. 
The Bayesian undesirability function consistently achieves the lowest mean cost, 
particularly as $\rho_T$ increases—highlighting 
its improved risk discrimination and path planning under elevated uncertainty.

\begin{figure} [H]
	\centering
	\begin{tabular}{c}
		\begin{subfigure}[b]{0.333\textwidth}
			\centering
			\includegraphics[scale=0.333]{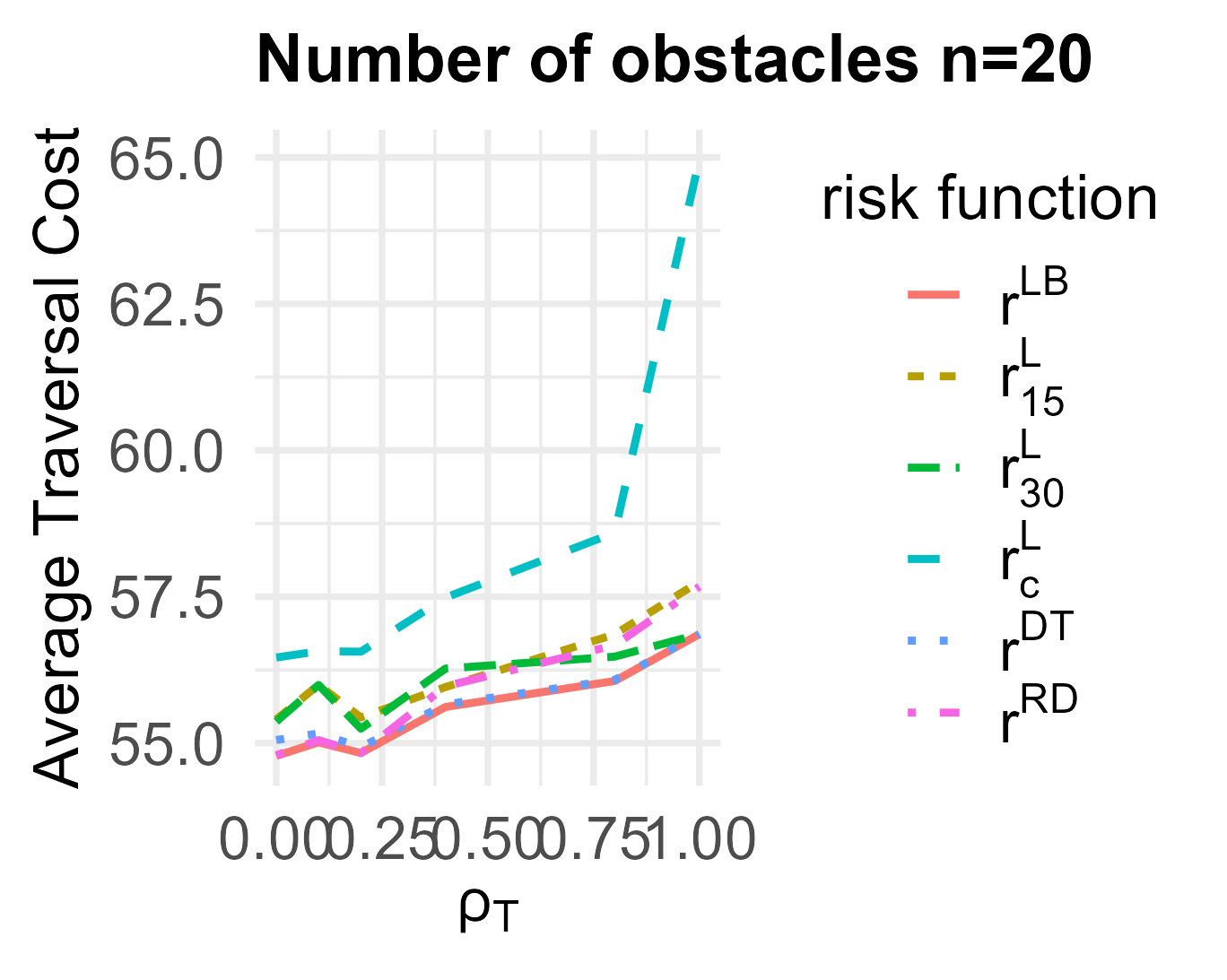}  
			\caption{}
		\end{subfigure}
		\begin{subfigure}[b]{0.333\textwidth}
			\centering
			\includegraphics[scale=0.333]{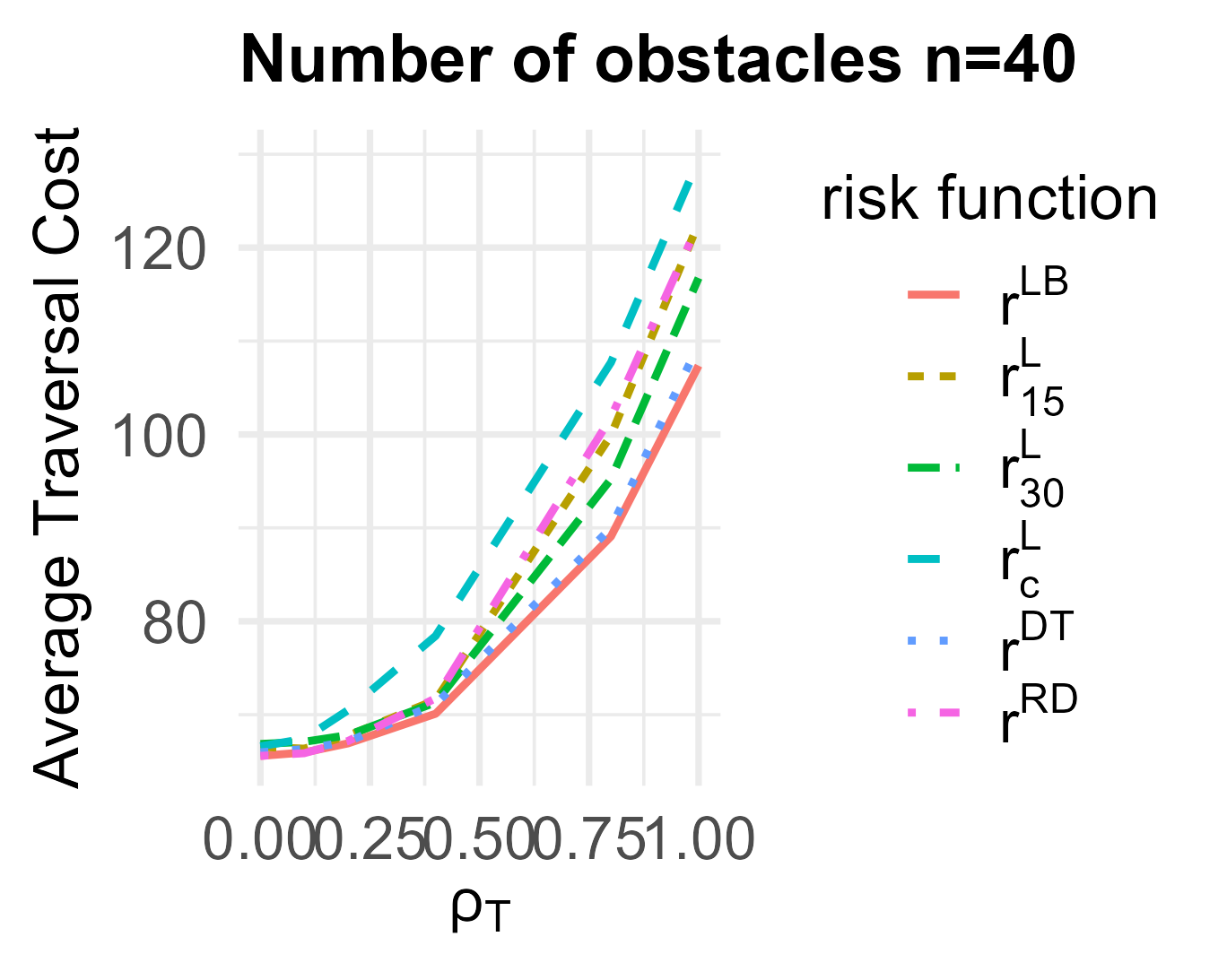}  
			\caption{}
		\end{subfigure}
		\begin{subfigure}[b]{0.333\textwidth}
			\centering
			\includegraphics[scale=0.333]{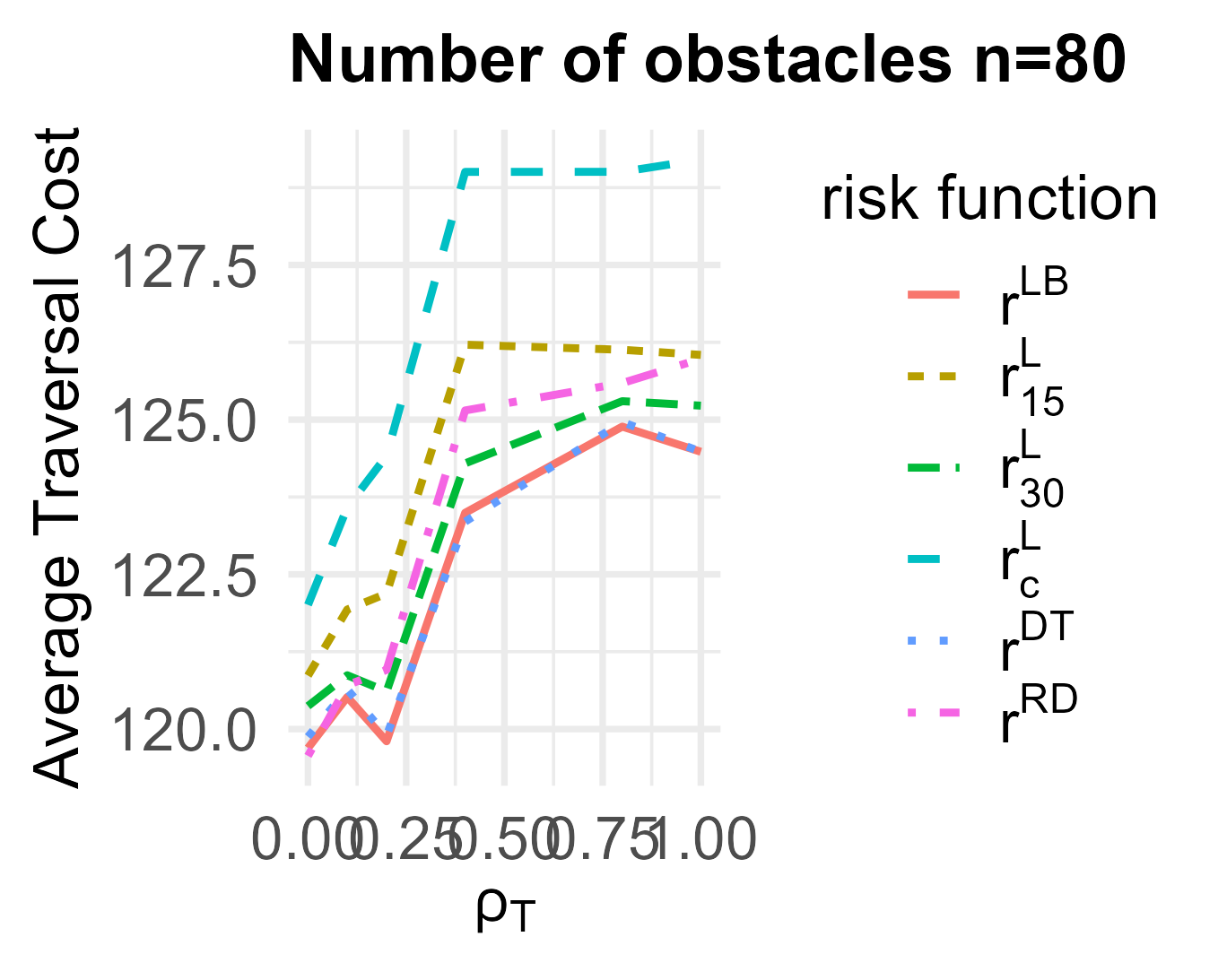}  
			\caption{}
		\end{subfigure}
	\end{tabular}
	\caption{Traversal cost versus true obstacle proportion $\rho_T$ using six different risk functions. 
		Panels: (a) $n=20$, (b) $n=40$, (c) $n=80$.}
	\label{sup-fig_compare_newrisk}
\end{figure}

Beyond mean performance, 
we also assess how often each policy achieves the minimum cost within individual replications. 
Table~\ref{sup-tab:newrisk_minprop} reports the proportion of 100 simulations 
in which the Bayesian risk function ($r^{LB}$) outperforms all others. 
It dominates consistently—especially in high $\rho_T$ settings and
under simplified constraints—confirming its robustness in both average-case and worst-case conditions.

\begin{table}[H]
	\centering
	\resizebox{12cm}{!}{%
		\begin{tabular}{@{}ccccccc@{}}
			\toprule
			\multicolumn{7}{c}{$n=20$} \\
			\multicolumn{1}{c|}{$\rho_T$} & 0 & 0.1 & 0.2 & 0.4 & 0.8 & 1 \\
			\midrule
			\multicolumn{1}{r|}{Simplified Scenario ($N_{\max}$)} & 0.98 & 0.97 & 0.97 & 0.98 & 0.98 & 1.00 \\
			\multicolumn{1}{r|}{General Scenario ($\delta_{\max}$)} & 0.81 & 0.77 & 0.80 & 0.91 & 0.92 & 1.00 \\
			\midrule
			\multicolumn{7}{c}{$n=40$} \\
			\multicolumn{1}{c|}{$\rho_T$} & 0 & 0.1 & 0.2 & 0.4 & 0.8 & 1 \\
			\midrule
			\multicolumn{1}{r|}{Simplified Scenario ($N_{\max}$)} & 0.94 & 0.90 & 0.92 & 0.87 & 0.92 & 0.91 \\
			\multicolumn{1}{r|}{General Scenario ($\delta_{\max}$)} & 0.56 & 0.58 & 0.59 & 0.58 & 0.84 & 0.93 \\
			\midrule
			\multicolumn{7}{c}{$n=80$} \\
			\multicolumn{1}{c|}{$\rho_T$} & 0 & 0.1 & 0.2 & 0.4 & 0.8 & 1 \\
			\midrule
			\multicolumn{1}{r|}{Simplified Scenario ($N_{\max}$)} & 0.96 & 0.97 & 0.98 & 0.98 & 0.99 & 1.00 \\
			\multicolumn{1}{r|}{General Scenario ($\delta_{\max}$)} & 0.62 & 0.55 & 0.57 & 0.56 & 0.84 & 1.00 \\
			\bottomrule
		\end{tabular}
	}
	\caption{Proportion of simulations in which the Bayesian linear undesirability function yields the lowest traversal cost.}
	\label{sup-tab:newrisk_minprop}
\end{table}

\subsubsection{Impact of Sensor Accuracy on Traversal Cost}
\label{sup-sec:Linear-lambda}

To further evaluate the convergence behavior of the RCDP policy
with the Bayesian undesirability function $r^{LB}$, 
we empirically assess how traversal cost evolves with increasing sensor precision. 
As established in Theorem~\ref{thm:convergence}, $r^{LB}$ approximates the benchmark cost under perfect sensing. Here, we validate this convergence across a continuum of sensor accuracies.

Using the same simulation setup as in Sections~\ref{sec:MC-simulations} and \ref{sup-sec:Linear-rho}, we vary a proxy parameter $\lambda$ controlling the sensor's informativeness over the range $[0, 4]$ in increments of 0.5. For each setting, we simulate 100 replications and compute the mean traversal cost under the RCDP policy with $r^{LB}$, denoted $\bar{\mathcal{C}}_{p^{LB}}$, and compare it to the benchmark cost $\bar{\mathcal{C}}_{p^{bm}}$.

Figure~\ref{sup-fig: rLB_convergence} shows that $\bar{\mathcal{C}}_{p^{LB}}$ 
decreases monotonically as $\lambda$ increases, 
aligning with prior findings that higher sensor precision improves policy performance \citep{ye2011sensor}. 
For $\lambda \geq 3$, the cost closely matches the benchmark, 
indicating practical convergence even under moderately accurate sensors.

This monotonicity not only reinforces theoretical guarantees 
but also demonstrates that $r^{LB}$ provides an effective, 
adaptive framework in settings with evolving or imperfect sensor inputs.

\begin{figure}[H]
	\centering
	\begin{tabular}{c}
		\begin{subfigure}[b]{0.333\textwidth}
			\centering
			\includegraphics[width=\textwidth]{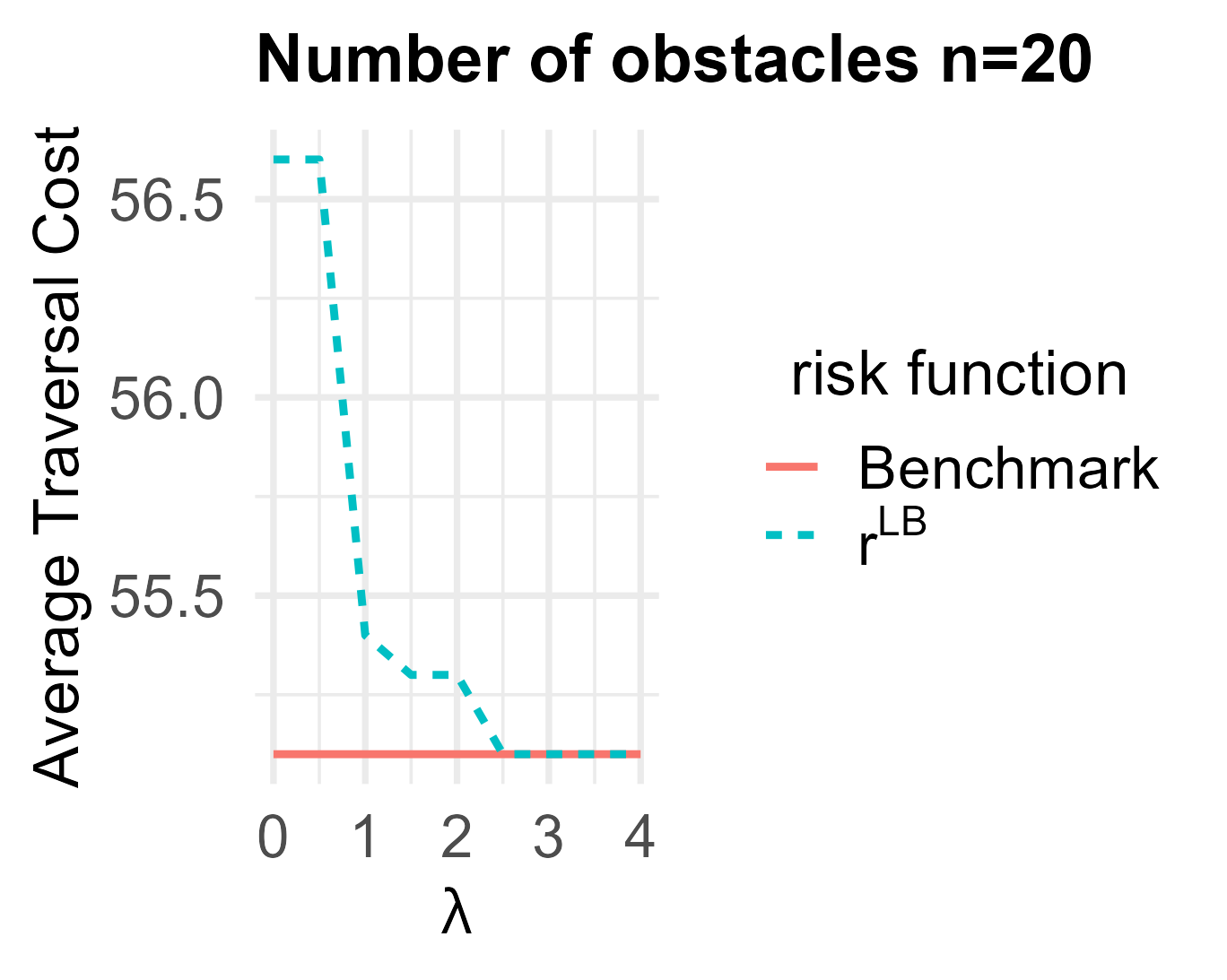}
			\caption{}
		\end{subfigure}
		\begin{subfigure}[b]{0.333\textwidth}
			\centering
			\includegraphics[width=\textwidth]{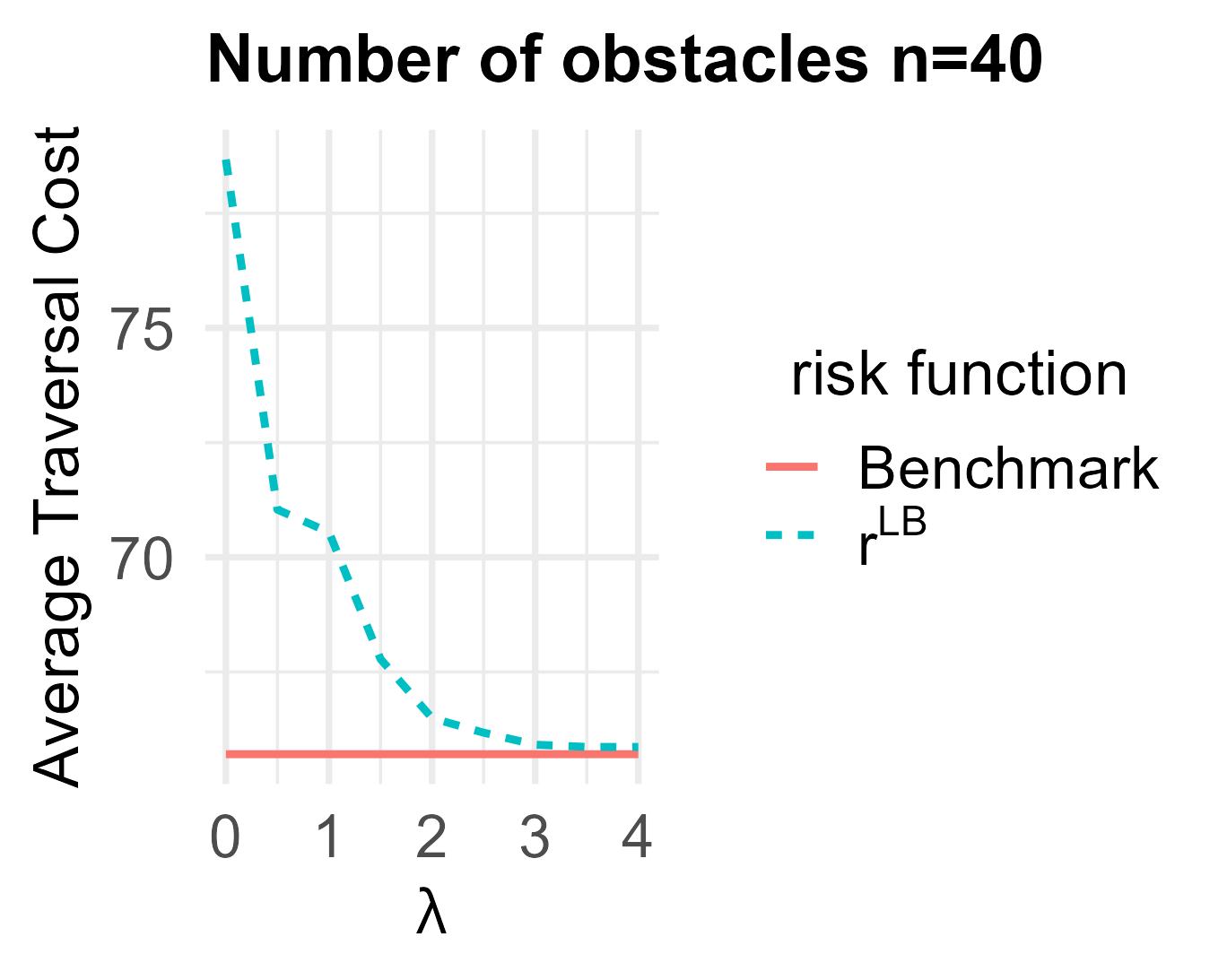}
			\caption{}
		\end{subfigure}
		\begin{subfigure}[b]{0.333\textwidth}
			\centering
			\includegraphics[width=\textwidth]{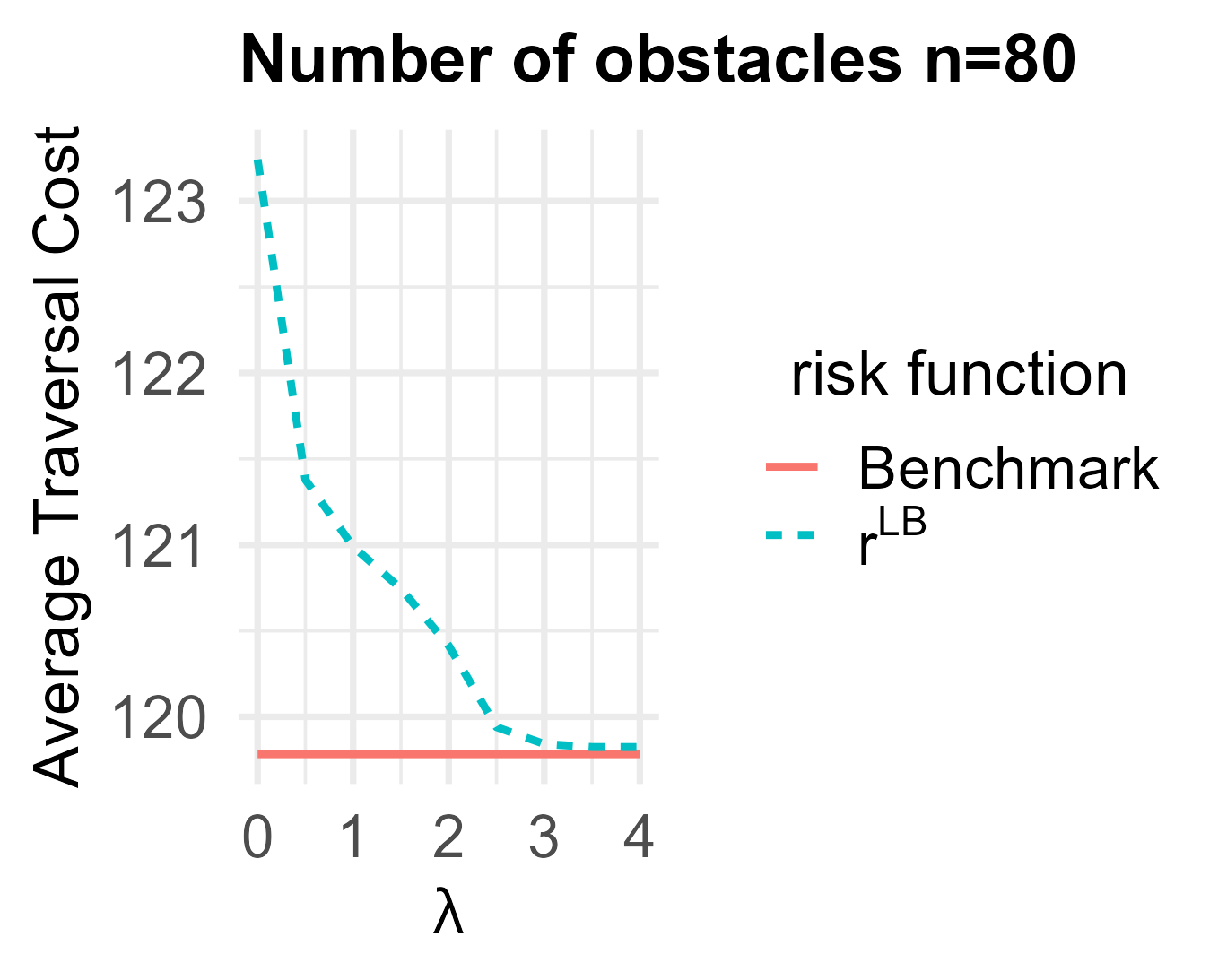}
			\caption{}
		\end{subfigure}
	\end{tabular}
	\caption{Traversal cost under the RCDP policy with $r^{LB}$ versus benchmark cost, as a function of sensor precision $\lambda$. Panels: (a) $n=20$, (b) $n=40$, (c) $n=80$.}
	\label{sup-fig: rLB_convergence}
\end{figure}

\subsubsection{Effect of Obstacle Spatial Pattern on Pathfinding Performance}
\label{sup-sec:Linear-spatial pattern}

To test the robustness of our findings under alternative spatial assumptions, 
we replicate the experiments from 
Sections~\ref{sup-sec:Linear-rho} and~\ref{sup-sec:Linear-lambda} 
using two different obstacle placement processes:  
(i) a uniform point process, and  
(ii) a Matérn clustering point process.

All other parameters are held constant to isolate the effect of spatial patterning. 
Figures~\ref{sup-fig:rLB_uniform} and~\ref{sup-fig:rLB_cluster} 
show that the RCDP policy with $r^{LB}$ retains its performance advantage across both spatial patterns. 
Its superior adaptivity and convergence to benchmark performance are consistent 
with results under Strauss-based configurations.

\begin{figure}[H]
	\centering
	\begin{tabular}{c}
		\begin{subfigure}[b]{0.35\textwidth}
			\centering
			\includegraphics[width=\textwidth]{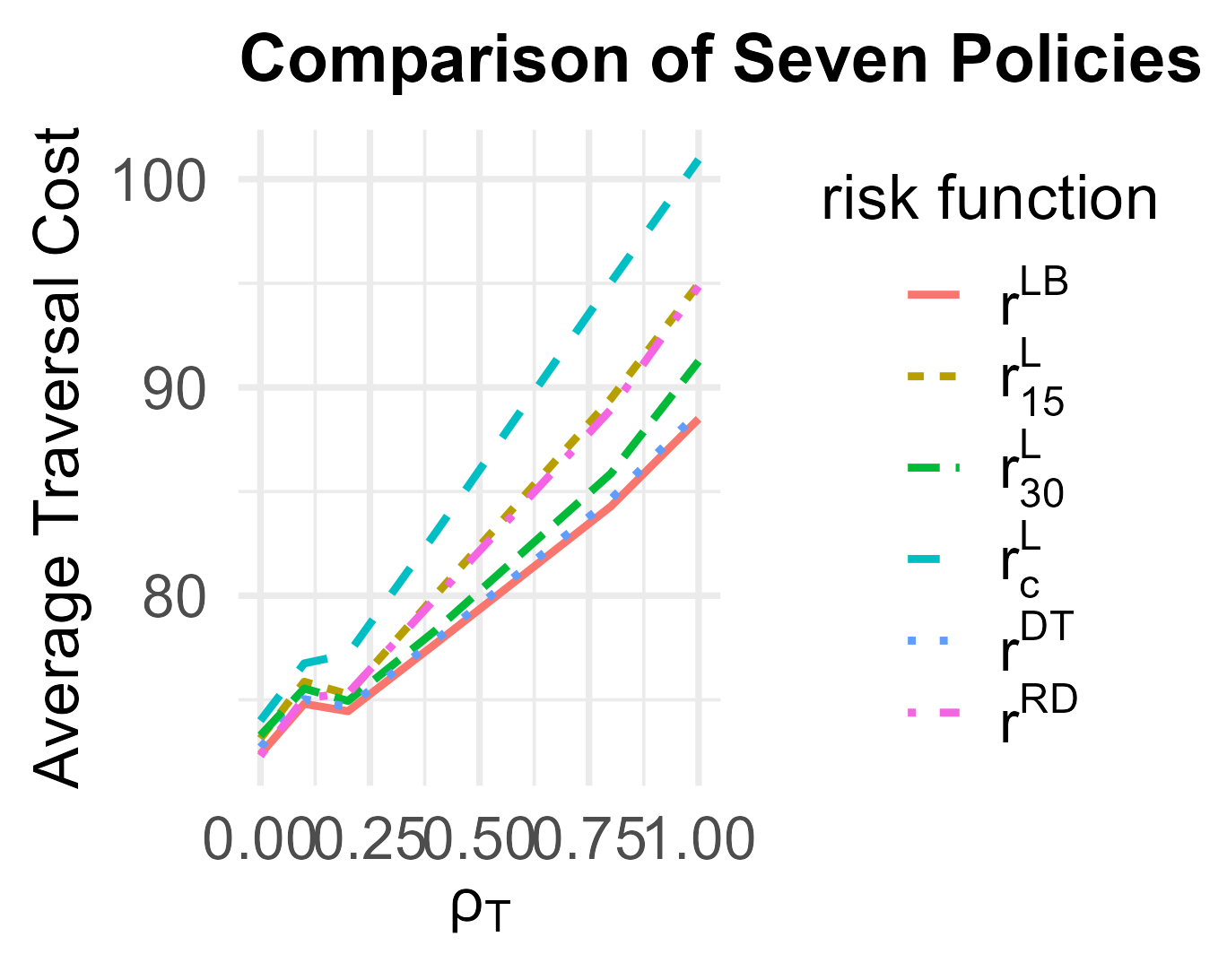}
			\caption{}
		\end{subfigure}
		\begin{subfigure}[b]{0.35\textwidth}
			\centering
			\includegraphics[width=\textwidth]{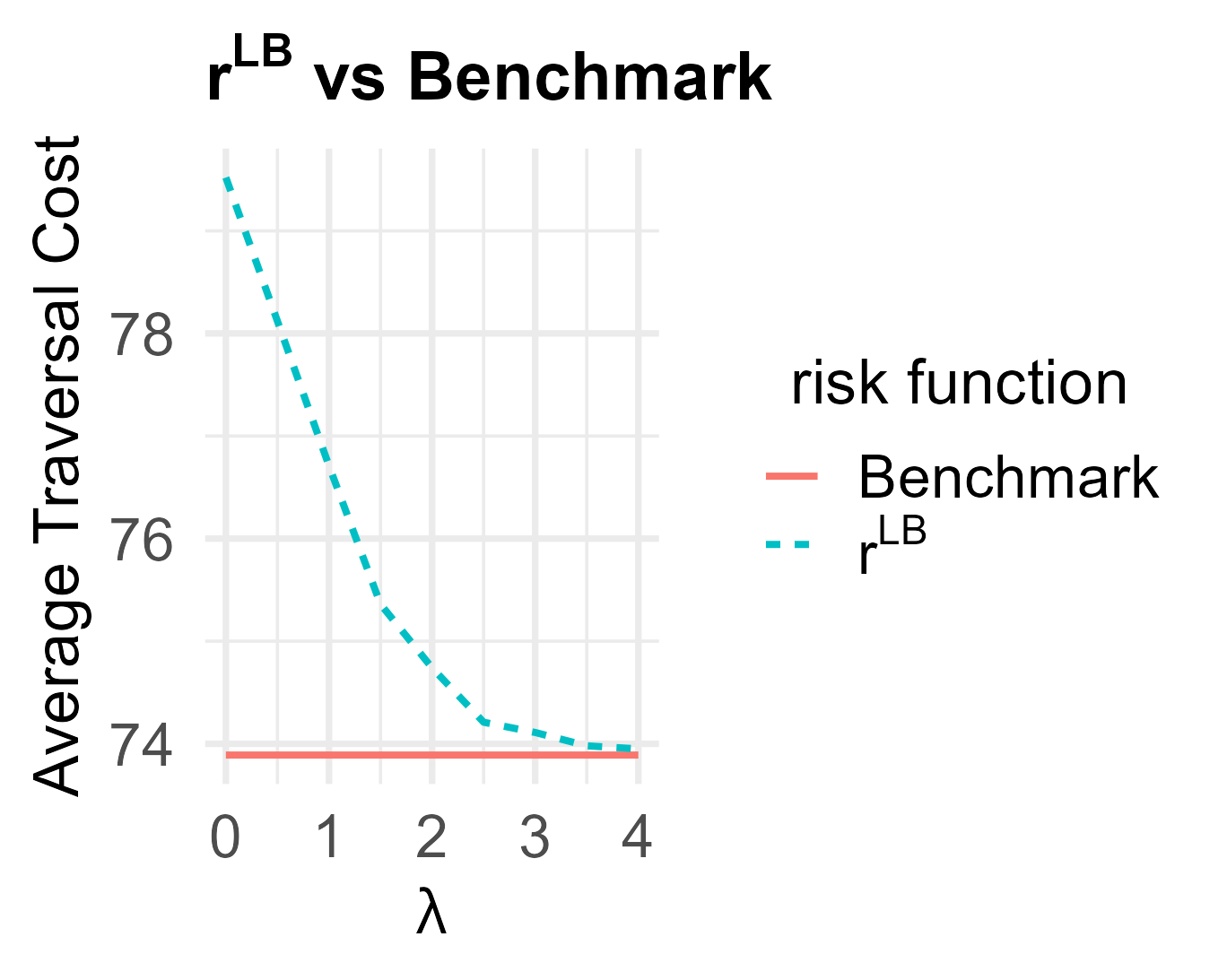}
			\caption{}
		\end{subfigure}
	\end{tabular}
	\caption{RCDP policy with $r^{LB}$ under uniform obstacle placement:  
		(a) traversal cost versus proportion of true obstacles;  
		(b) cost comparison with benchmark versus sensor accuracy $\lambda$.}
	\label{sup-fig:rLB_uniform}
\end{figure}

\begin{figure}[H]
	\centering
	\begin{tabular}{c}
		\begin{subfigure}[b]{0.35\textwidth}
			\centering
			\includegraphics[width=\textwidth]{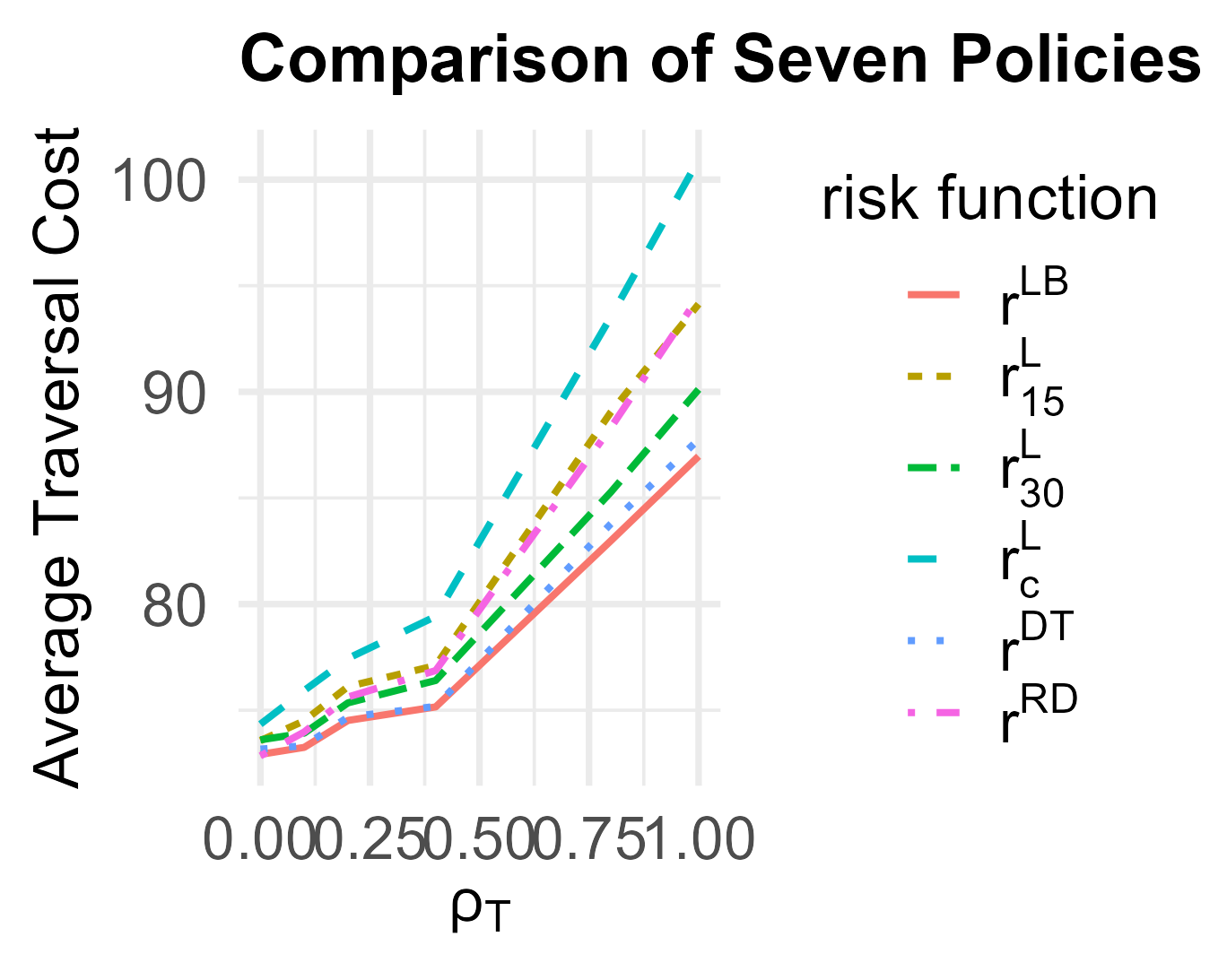}
			\caption{}
		\end{subfigure}
		\begin{subfigure}[b]{0.35\textwidth}
			\centering
			\includegraphics[width=\textwidth]{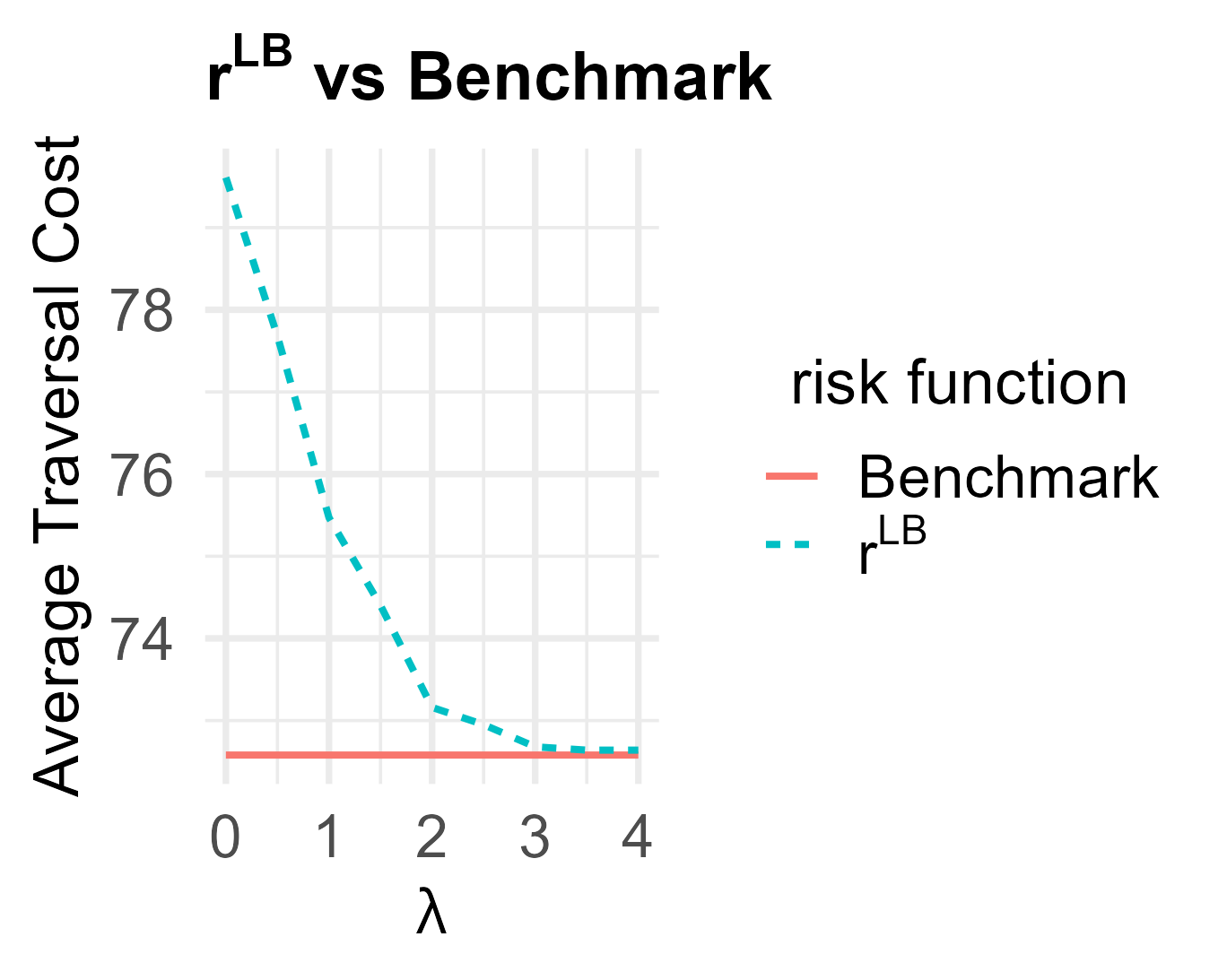}
			\caption{}
		\end{subfigure}
	\end{tabular}
	\caption{RCDP policy with $r^{LB}$ under clustered (Matérn) obstacle placement:  
		(a) traversal cost versus proportion of true obstacles;  
		(b) cost comparison with benchmark versus sensor accuracy $\lambda$.}
	\label{sup-fig:rLB_cluster}
\end{figure}

\subsection{Visualization of Relative Efficiency}

Figure~\ref{fig:efficiency-boxplot} presents the relative efficiency box plots 
for each policy across 15 simulation regimes. LU-based policies consistently 
achieve high median efficiency with tight variance, indicating robust performance 
under diverse environment and sensing settings.

\begin{figure}[H]
	\centering
	\includegraphics[width=0.32\textwidth]{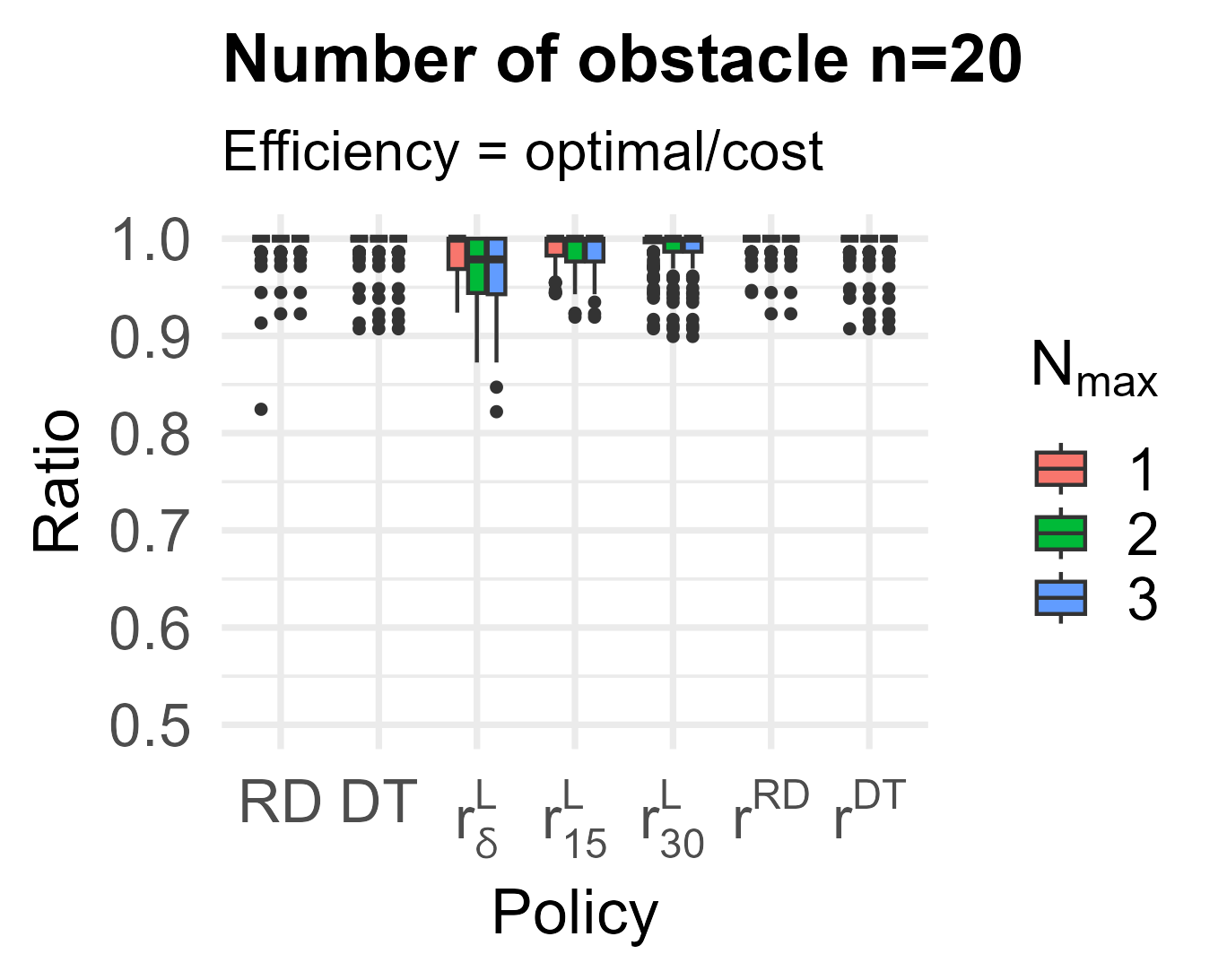}
	\includegraphics[width=0.32\textwidth]{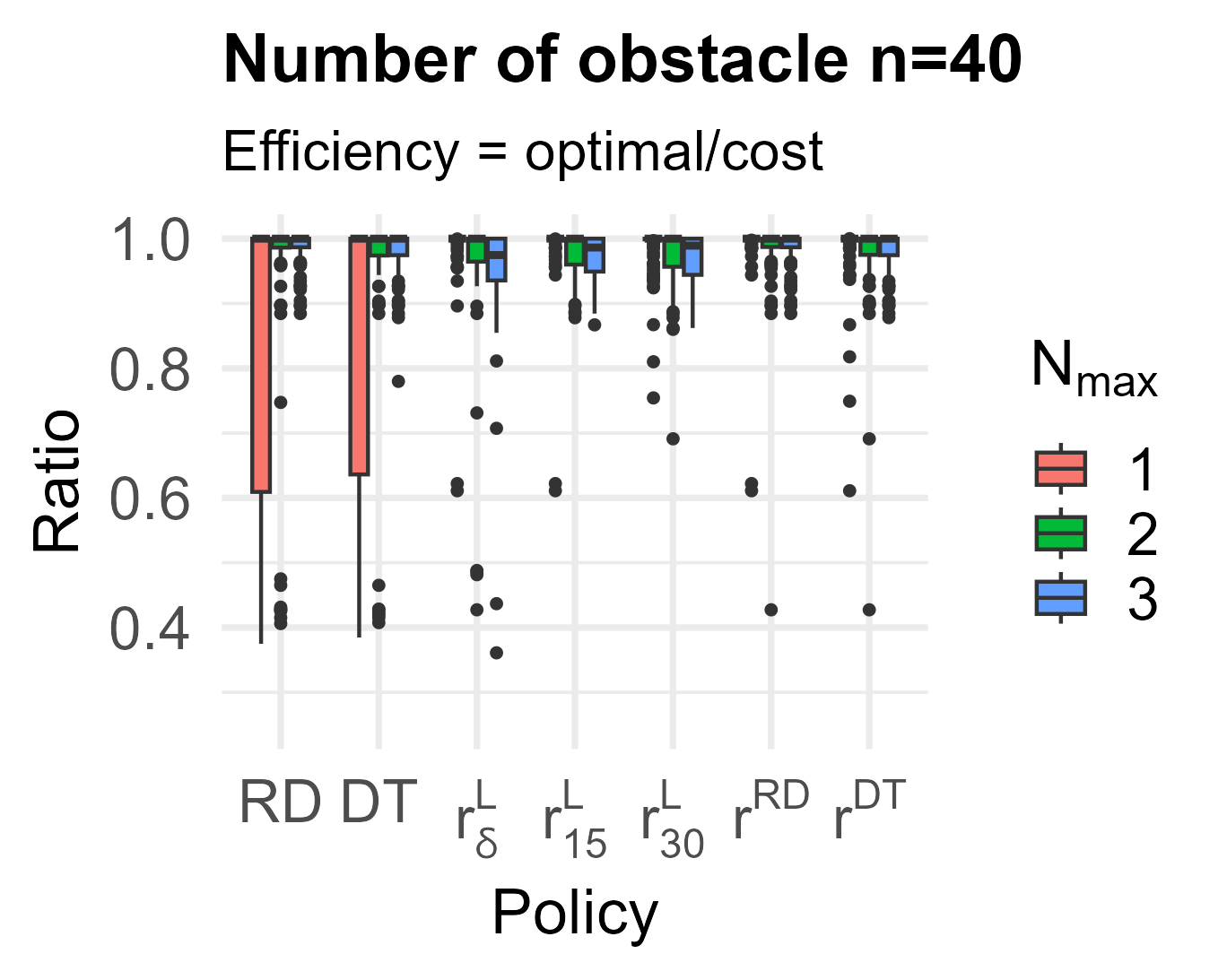}
	\includegraphics[width=0.32\textwidth]{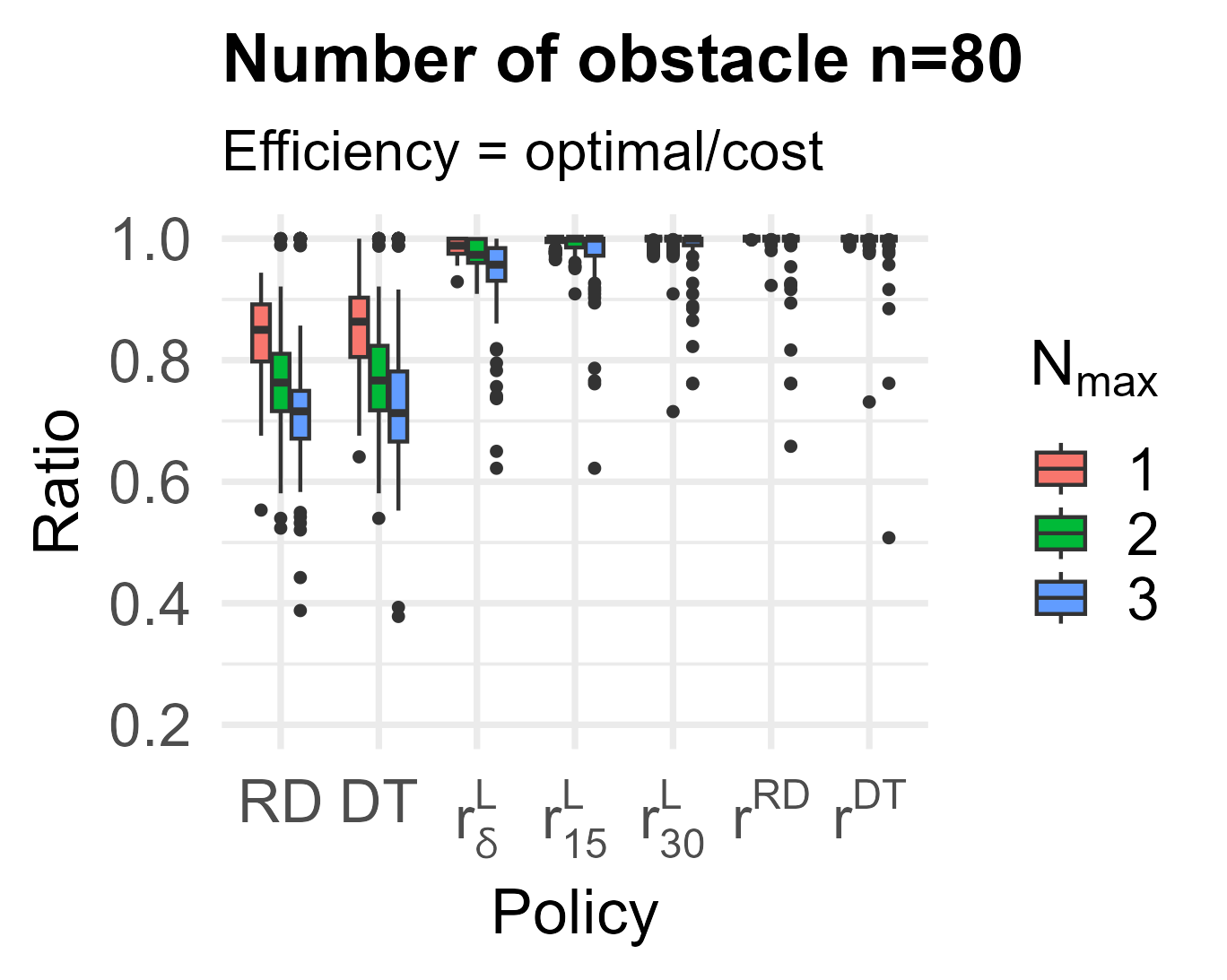}
	\includegraphics[width=0.32\textwidth]{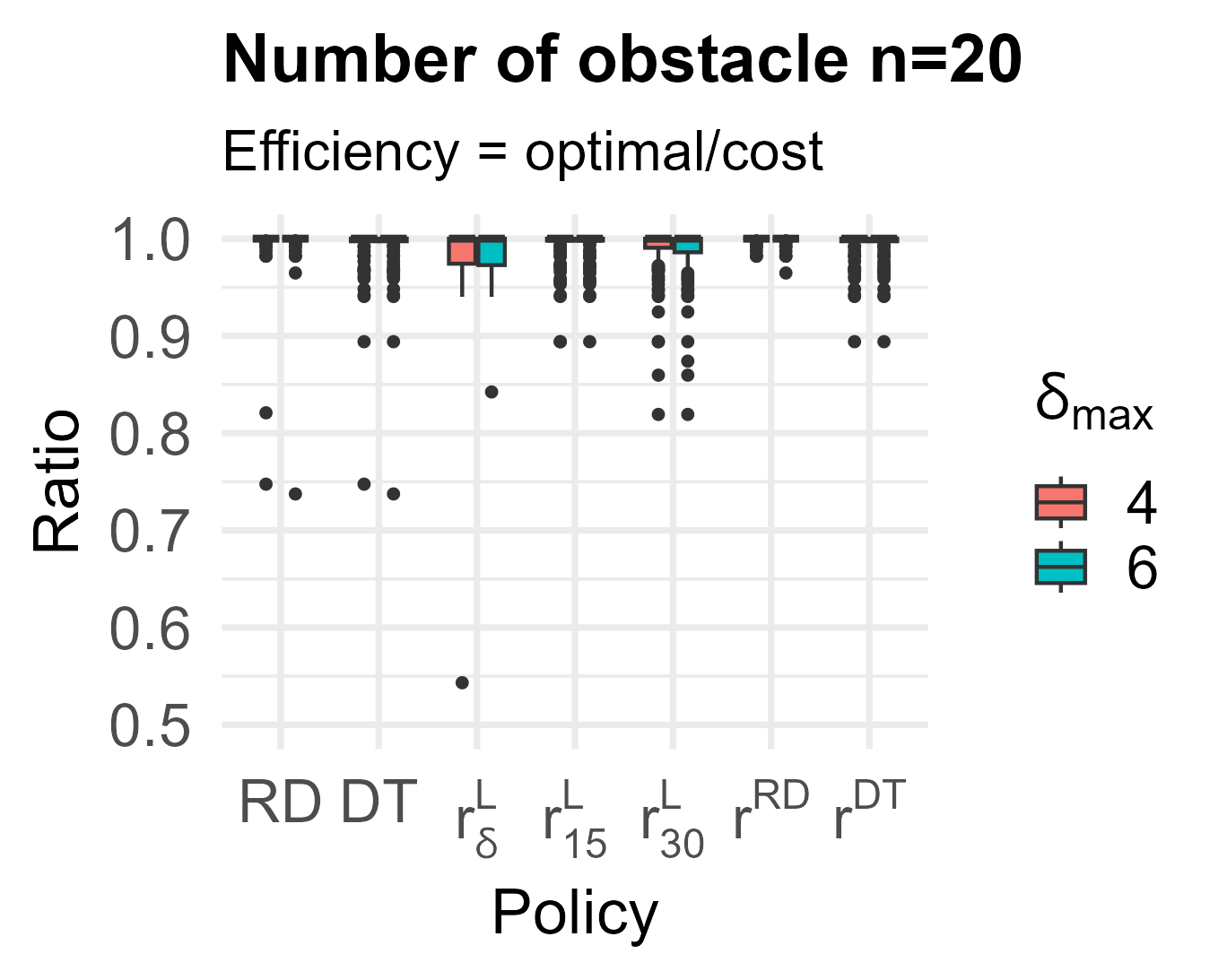}
	\includegraphics[width=0.32\textwidth]{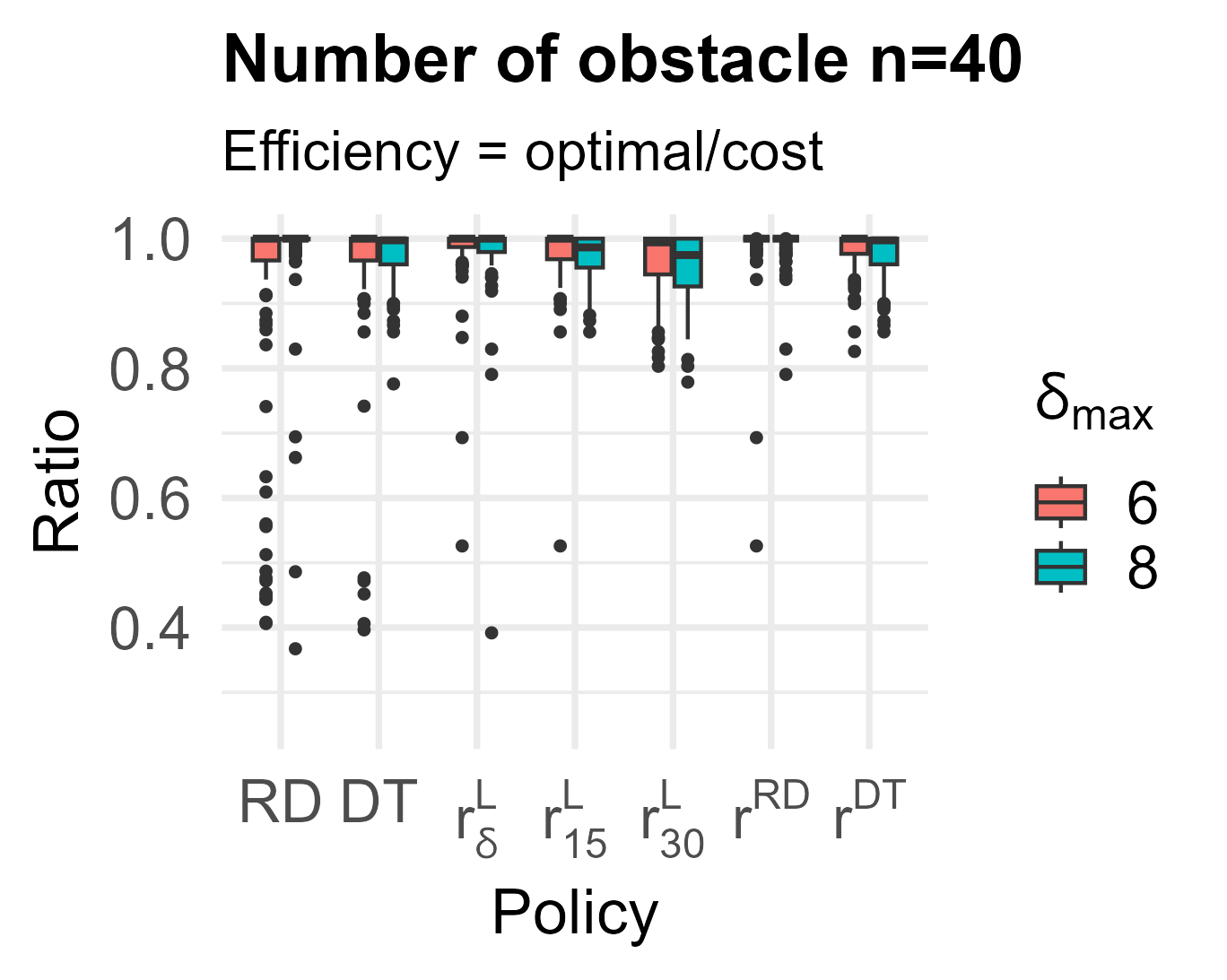}
	\includegraphics[width=0.32\textwidth]{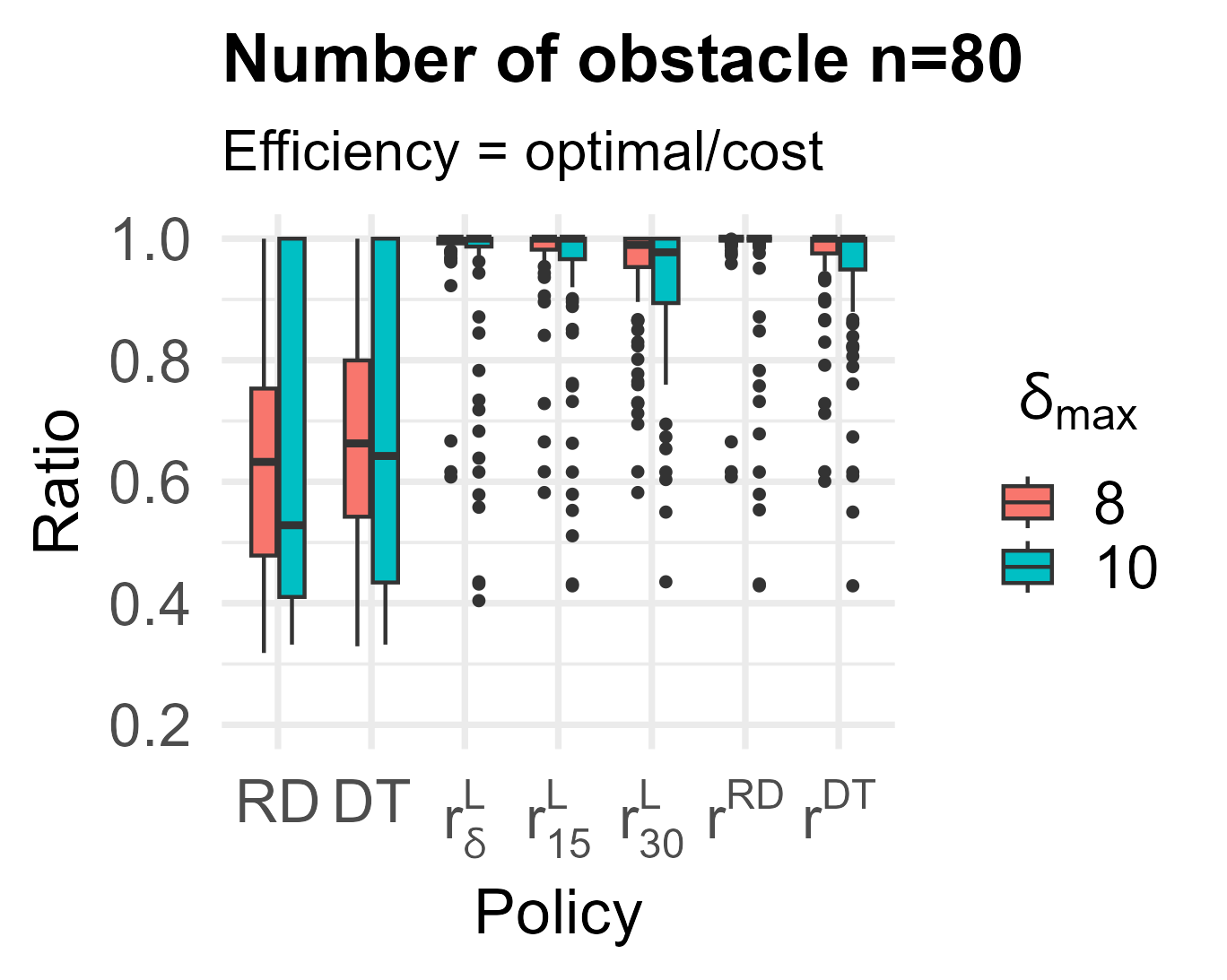}
	\caption{Relative efficiency of risk-based RCDP policies aggregated over all 
		15 simulation regimes. LU policies achieve high median efficiency with tight variance.}
	\label{fig:efficiency-boxplot}
\end{figure}

\makeatletter
\@input{ResConstRDP_LV.bbl}
\makeatother
\end{document}